\pdfoutput=1
\documentclass[10pt,twocolumn]{article}
\usepackage[letterpaper,textwidth=6.75in,textheight=9.25in,columnsep=0.25in,centering]{geometry}
\usepackage[round]{natbib}

\PassOptionsToPackage{hyphens}{url}
\usepackage{amsmath,amssymb,amsthm,mathtools,graphicx,booktabs,microtype,tikz}
\usepackage{hyperref,url}
\hypersetup{hidelinks,pdftitle={Verification with Transfer: Exact Information Frontiers and Their Price in Calls},pdfauthor={Hazar Yueksel}}
\usepackage[capitalise,noabbrev]{cleveref}
\usepackage{alphalph}
\newtheorem{theorem}{Theorem}
\newtheorem{proposition}[theorem]{Proposition}
\newtheorem{lemma}[theorem]{Lemma}
\newtheorem{corollary}[theorem]{Corollary}
\theoremstyle{definition}
\theoremstyle{remark}
\makeatletter
\@ifl@t@r\fmtversion{2021-06-01}{%
  \AddToHook{env/proposition/begin}{\crefalias{theorem}{proposition}}%
  \AddToHook{env/lemma/begin}{\crefalias{theorem}{lemma}}%
  \AddToHook{env/corollary/begin}{\crefalias{theorem}{corollary}}%
  \AddToHook{env/definition/begin}{\crefalias{theorem}{definition}}%
  \AddToHook{env/remark/begin}{\crefalias{theorem}{remark}}%
  \AddToHook{cmd/appendix/after}{\crefalias{section}{appendix}\crefalias{subsection}{subappendix}\crefalias{subsubsection}{subsubappendix}}%
}{}
\makeatother
\newcommand{\E}{\mathbb E}
\newcommand{\Pp}{\mathbb P}
\newcommand{\KL}{D_{\mathrm{KL}}}
\newcommand{\aux}{B_{\mathrm{aux}}}
\newcommand{\Ls}{\mathcal L_n}
\newcommand{\cc}{c_{\theta L}}
\newcommand{\R}{\mathcal R_n}
\newcommand{\ind}{\mathbf 1}
\newcommand{\given}{\mid}
\newcommand{\acc}{\mathrm{acc}}
\newcommand{\F}{\mathbb F_2}
\newcommand{\NeuralParams}{58{,}816}
\newcommand{\NoiseUp}{0.9900}
\newcommand{\NoiseAfter}{0.0793}

\newcommand{\QSPLargeTime}{5.15}

\newcommand{\VFiveIIDExchange}{0.734}
\newcommand{\VFiveIIDRegret}{2.83}
\newcommand{\VFiveIIDPortfolio}{0.750}
\newcommand{\VFiveIIDPortfolioRegret}{2.28}
\newcommand{\VFiveIIDRegressions}{1}
\newcommand{\VFiveClusterExchange}{0.812}
\newcommand{\VFiveClusterRegret}{2.36}
\newcommand{\VFiveClusterPortfolio}{0.875}
\newcommand{\VFiveClusterPortfolioRegret}{1.10}
\newcommand{\VFiveClusterRegressions}{2}

\newcommand{\VFiveNoiseSingle}{0.402}
\newcommand{\VFiveNoiseTriple}{0.9979}
\newcommand{\VFiveNoiseConditional}{0.877}
\newcommand{\VFiveCleanCap}{9.25}
\newcommand{\VFiveTripleCap}{25.25}
\newcommand{\VFiveBiasedUniform}{0.642}
\newcommand{\VFiveBiasedMatched}{0.876}
\newcommand{\VFiveTestCount}{199}
\newcommand{\VFivePlannedCells}{559}
\newcommand{\VFiveTotalCells}{720}
\newcommand{\VFiveDecodedContexts}{286,208}
\newcommand{\BayesLossDEight}{0.48}
\newcommand{\BayesLossDEleven}{0.50}
\newcommand{\BayesLossDFive}{0.44}
\newcommand{\BayesLossDSeven}{0.48}
\newcommand{\BayesLossDSix}{0.48}

\newcommand{\CurEightBayesQsp}{1.8}

\newcommand{\CurEightElimNone}{34.8}
\newcommand{\CurEightElimQsp}{20.9}

\newcommand{\CurEightSpreadLearner}{1.4}
\newcommand{\CurEightSpreadRef}{24.0}

\newcommand{\CurElevenRhoGreedy}{6.58}

\newcommand{\CurElevenRhoQsp}{7.00}

\newcommand{\CurFiveBlindNone}{16.0}
\newcommand{\CurFiveBlindQsp}{2.7}
\newcommand{\CurFiveBlindRandom}{7.8}

\newcommand{\CurFiveBlindRefQsp}{2.7}
\newcommand{\CurFiveBlindRefRandom}{5.3}

\newcommand{\ICexactDFive}{12}

\newcommand{\ICfracDEight}{14\%}
\newcommand{\ICfracDEightMax}{20\%}
\newcommand{\ICfracDEightMin}{0\%}
\newcommand{\ICfracDEleven}{0\%}

\newcommand{\ICfracDSeven}{49\%}
\newcommand{\ICfracDSevenMax}{78\%}
\newcommand{\ICfracDSevenMin}{0\%}
\newcommand{\ICfracDSix}{84\%}
\newcommand{\ICfracDSixMax}{100\%}
\newcommand{\ICfracDSixMin}{49\%}

\newcommand{\ICrunsDFive}{12}

\newcommand{\IWgrpoRuns}{27}

\newcommand{\IWgrpoStuckCallsPct}{51\%}
\newcommand{\IWgrpoStuckRuns}{16}

\newcommand{\IWklStuckCallsPct}{47\%}
\newcommand{\IWklStuckRuns}{13}

\newcommand{\SPbigSixDEightfracQsp}{91\%}

\newcommand{\SeedMaxEight}{12\%}
\newcommand{\SeedMaxSeven}{26\%}
\newcommand{\SeedMaxSix}{66\%}

\newcommand{\SeedMinEight}{0\%}
\newcommand{\SeedMinSeven}{1\%}
\newcommand{\SeedMinSix}{49\%}

\newcommand{\CanonicalSentence}{The plain decoder, trained on that same bank, uses 40\% with the planner's sources (37\% over the three curricula).}

\newcommand{\CtrlResultsParagraph}{On the low banks the decoder used 66\%--69\% of the delivered bits, and on the high banks 42\%--67\%.}
\newcommand{\DegHighMax}{78\%}

\newcommand{\DegLowMin}{99\%}

\newcommand{\DegreeConclusion}{P1 failed: every low-degree run lost between 0.92 and 1.03 of its 3 delivered bits, so low degree is not sufficient. In ten of the twelve runs, low- and high-degree alike, the decoder used between 1.97 and 2.08 of the 3 bits. P2 held on the controlled banks, but its bound fails on two of the 15 banks above: bank 0 at $d=6$ converted 1.98 and 1.90 bits under two further training seeds with only one bit of degree at most two, and bank 2 at $d=8$ converted 0.77 bits with no bit of degree at most two. Degree therefore neither guarantees nor bounds use.}
\newcommand{\DegreeTestSentence}{To separate degree from dimension, we fixed $(k,h,J)=(4,2,12)$, chose three new banks of each kind from the bank alone, and recorded two predictions before training: (P1) the decoder uses at least $95\%$ of the delivered bits on every low-degree bank and less on every high-degree bank; (P2) on every high-degree bank, the bits it uses are at most the number of determined bits of degree at most two, plus $0.5$. Over two training seeds each, it used 66\%--69\% of the delivered bits on the low-degree banks and 42\%--67\% on the high-degree ones.}
\newcommand{\FrontierGapInter}{0.28}
\newcommand{\FrontierGapSF}{0.35}
\newcommand{\LongSentence}{Training four times longer (24,000 steps, same seeds): on bank 202609500 the final loss moved from 0.579 to 0.541 nats (Bayes-optimal 0.471) and the bits used from 2.17 to 3.00 of 3; on bank 202609503 the final loss moved from 0.490 to 0.482 nats (Bayes-optimal 0.445) and the bits used from 2.00 to 3.00 of 3, at the frontier for every source budget; on bank 202609512 the final loss moved from 0.549 to 0.519 nats (Bayes-optimal 0.463) and the bits used from 2.00 to 2.00 of 3; on bank 202609506 the final loss moved from 0.621 to 0.527 nats (Bayes-optimal 0.480) and the bits used from 1.16 to 2.76 of 3; on bank 202609518 the final loss moved from 0.568 to 0.530 nats (Bayes-optimal 0.466) and the bits used from 2.00 to 2.01 of 3; on bank 202609534 the final loss moved from 0.564 to 0.540 nats (Bayes-optimal 0.477) and the bits used from 2.00 to 2.00 of 3.}

\newcommand{\LongTestSentence}{Scored on all 64 secrets with one training seed (the $6{,}000$-step ranges of \cref{sec:neuralexp} use four secrets and two seeds), the low-degree banks used 67\%--100\% of the delivered bits and the high-degree banks 67\%--92\%; prediction (L1) failed and (L2) failed.}
\newcommand{\PerBitN}{six}
\newcommand{\PerBitUsedMin}{0.99}
\newcommand{\PerBitLostMax}{0.14}
\newcommand{\PerBitLostList}{$a_0$ on two banks (degrees 2, 2); $a_1$ on four banks (degrees 2, 3, 4, 4); $a_2$ on one bank (degree 4)}

\newcommand{\PerBitSecretsGap}{0.21}
\newcommand{\PerBitFullByPos}{$a_0$ on no bank, $a_1$ on two, $a_2$ on four and $a_3$ on all six banks}
\begin{document}
\twocolumn[{%
\centering
{\LARGE\bfseries Verification with Transfer:\\[2pt] Exact Information Frontiers and Their Price in Calls\par}
\vspace{1.2em}
{\large Hazar Yueksel\par}
\vspace{0.3em}
\texttt{hazar@hazaryueksel.com}\par
\vspace{1.8em}
}]
{\renewcommand{\thefootnote}{}\footnotetext{The Lean~4 formalization referred to in the text is available as ancillary
files on arXiv. The code archive referred to in the text is not included in this version.}}

\begin{abstract}
A verifier that accepts or rejects whole answers reveals little: under a flat prior over $k$-bit answers, zero error needs $2^k-1$ verifications. The usual remedy is to solve related source tasks, either all first, as a curriculum does, or interleaved with verification. We price this remedy in information and in calls. With an exact verifier, the least causal information that any interleaving of source calls and $n$ verifications needs to succeed with probability $s$ is a list rate--distortion function, attained by one observation before any verification. It lower-bounds the expected number of binary source calls, which designed sources meet within $1+\log_25$ calls for unique answers and within a logarithmic term in general, where no additive constant suffices. With an exact verifier and fixed sources, moving every call before the first verification preserves all hard caps on calls, although interleaving can save unboundedly many expected calls; under a noisy verifier, source-first protocols can lose unbounded factors in information and in error. For linear banks over $\F$, optimal accuracy has a closed form, and after a polynomial-time reduction the budget profile is computable in time $2^{O(h^2)}\operatorname{poly}(J,k+h)$ for $J$ sources and nuisance dimension $h$. In these banks, for zero error under a hard cap, the calls beyond the rounded-up information price are exactly those spent on nuisance. Every numbered result apart from two clauses about the planner is machine-checked in Lean~4, assuming two published results. Used as a ruler, the frontier shows a small transformer using all delivered bits at latent dimension $5$ and none at $11$ within fixed training budgets; in a test with predictions recorded before training, low XOR degree of the target bits did not suffice for their use.
\end{abstract}
\section{Introduction}
A verifier that accepts or rejects a complete answer reveals little per call. A learner that has made $n$ verifications succeeds only if it has tried the answer or can output it as a final guess, so its success probability is at most the prior mass of the $n+1$ answers it could have covered (Proposition~\ref{prop:coverage}). For a $k$-bit answer with a flat prior, zero error takes $2^k-1$ verifications. The same limit appears in exact learning with equivalence queries \citep{angluin1988queries}, in guessing \citep{massey1994,arikan1996}, and in reinforcement learning from verifiable rewards, where a policy that has never been rewarded receives no gradient \citep{shao2024}.

The usual remedy is transfer: related, easier source tasks whose answers exclude many target answers at once. Curricula and autocurricula exploit exactly this structure \citep{bengio2009curriculum,rajaraman2026}. They also raise a scheduling question. Should a learner work through its source tasks first and only then turn to the target, as a curriculum does, or interleave, choosing each source in light of what the verifier has already said? We answer this and the questions around it, exactly for information and, with fixed sources, for hard caps. A learner has a latent world, a target verifier, and source tasks that reveal information about the world, and it may interleave source calls and verifications in any order. How much source information does it need to succeed with probability $s$ within $n$ verifications? How many source calls? When does putting all source calls first lose nothing? For a concrete family of banks, what does each call buy, and how is the best set of calls found?

The answers separate three regimes for \emph{source-first} protocols, which make every source call before the first verification, as a curriculum does. Under exact verification they lose nothing in information when the sources can be designed, nothing in hard caps on calls when the sources are fixed, and, for unique answers, only a constant number of expected calls with designed sources. Against a given set of sources they can lose an unbounded number of expected calls, and more information. Under a noisy verifier they can lose unbounded factors in information and in error.

\paragraph{Contributions.}
\begin{itemize}
\item \emph{Information frontier} (\cref{sec:model}). The least causal auxiliary information over all interleaved protocols equals a list rate--distortion function, and one observation before any verification attains it (\cref{thm:exact}). The converse is a change of measure that keeps exact verifications in the reference law, so they cancel and are never charged; the dual is the classical rate--distortion dual in list-coverage coordinates.
\item \emph{Call frontier} (\cref{sec:model}). Binary source calls cost at least this information in expectation. For unique answers the bound is tight to an additive constant at every success level ($1+\log_25$, and $2$ at zero error), whereas for general acceptance sets the gap between chromatic and graph entropy makes the excess unbounded, though at most logarithmic (\cref{thm:calls,prop:general}, \cref{fig:frontiers}). Against a given set of sources, interleaving can save an unbounded number of expected calls, but for unique answers it saves only a constant over designed source-first protocols (\cref{prop:interleave}). With fixed sources, moving every source call before the first verification preserves hard caps on calls (\cref{thm:fixed}). Under a noisy verifier the converse survives and is exact on our instance, but source-first protocols can lose unbounded factors in information and in error (\cref{prop:noisy}).
\item \emph{Linear banks} (\cref{sec:planner}). When the target and the sources are linear over $\F$ and the prior is uniform, the optimal accuracy under any interleaving is $\min\{1,(n+1)2^{-k+\rho_b}\}$, where $\rho_b$ is the number of target dimensions the best $b$ sources resolve (\cref{thm:rank}). The profile $b\mapsto\rho_b$ inverts a relative generalized Hamming weight hierarchy, and after a polynomial-time reduction a planner computes it for $J$ sources in time $2^{O(h^2)}\operatorname{poly}(J,k+h)$ in the nuisance dimension $h$ (\cref{thm:qsp}). For zero error under a hard cap, the calls a bank needs beyond the rounded-up information price are exactly the calls its optimal subset spends on nuisance (\cref{cor:waste}).
\item \emph{A ruler for learned decoders} (\cref{sec:neuralexp}). On such banks we measure how much of the delivered information a small transformer turns into target probability, and what the rest costs in verifications.
\end{itemize}
Several ingredients are classical and credited where used (list-Fano and rate--distortion duals, guessing with side information, interactive change of measure, relative weight hierarchies; \cref{sec:related}). What is new is the accounting when verification and transfer interleave: exact verifications are never charged, so the interleaved frontier equals the static one (\cref{thm:exact}); reordering is free under hard caps with fixed sources and an exact verifier, but not in expected calls or under noise (\cref{thm:fixed,prop:interleave,prop:noisy}); and in a linear bank at zero error the calls beyond the rounded-up information price are exactly the nuisance an optimal subset pays for (\cref{cor:waste}). Every numbered result apart from two clauses about the planner is machine-checked in Lean~4; the scope is stated at the start of the appendix.
\section{Verification with side information}\label{sec:model}
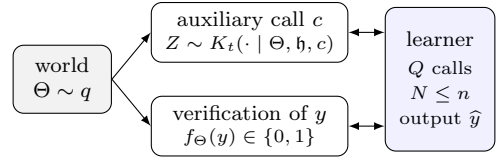
\begin{figure}[t]
\centering
\begin{tikzpicture}[x=0.84cm,y=1cm,font=\footnotesize,>=latex]
 \node[draw,rounded corners,fill=gray!10,minimum width=1.3cm,minimum height=.9cm,align=center] (world) at (0,0) {world\\$\Theta\sim q$};
 \node[draw,rounded corners,minimum width=2.6cm,minimum height=.75cm,align=center] (src) at (2.95,0.62) {auxiliary call $c$\\[-1pt]\scriptsize $Z\sim K_t(\cdot\mid\Theta,\mathfrak h,c)$};
 \node[draw,rounded corners,minimum width=2.6cm,minimum height=.75cm,align=center] (tgt) at (2.95,-0.62) {verification of $y$\\[-1pt]\scriptsize $f_\Theta(y)\in\{0,1\}$};
 \node[draw,rounded corners,fill=blue!6,minimum width=1.45cm,minimum height=1.9cm,align=center] (lrn) at (5.95,0) {learner\\[1pt]\scriptsize $Q$ calls\\\scriptsize $N\le n$\\\scriptsize output $\widehat y$};
 \draw[->] (world.east) -- (src.west); \draw[->] (world.east) -- (tgt.west);
 \draw[<->] (src.east) -- (lrn.west|-src.east); \draw[<->] (tgt.east) -- (lrn.west|-tgt.east);
\end{tikzpicture}
\caption{\textbf{Verification with transfer.} Auxiliary calls reveal information about the world; a verification accepts or rejects a whole response, and an accepted response may be output. We price the auxiliary information $\aux$ (defined in~\eqref{eq:budget}) and the calls $Q$ needed to succeed with probability $s$ within $n$ verifications.}\label{fig:protocol}
\end{figure}
\paragraph{Model.}
A latent world $\Theta\in[M]$ has a full-support prior $q$. A finite response set $\mathcal Y$ and known maps $f_\theta:\mathcal Y\to\{0,1\}$ define the target task. Every world has an acceptable response, and acceptance sets may overlap; \emph{unique answers} means $\mathcal Y=[M]$ and $f_\theta(y)=\ind\{\theta=y\}$. At each step the learner does one of three things. It \emph{verifies} a response $y$ and observes $f_\Theta(y)$. It makes an \emph{auxiliary call} $c$ and observes $Z\sim K_t(\cdot\given\Theta,\mathfrak h,c)$, a finite-alphabet kernel that may depend on the whole history $\mathfrak h$. Or it stops and outputs $\widehat y$ (\cref{fig:protocol}). A request $c$ is any label from a finite set and may encode the learner's coins; the history records every request and reply. Actions depend on the history and a private seed $\xi$ independent of $\Theta$; $C_t$ is the $t$th action and $H_{t-1}$ the history before it. Every input correlated with $\Theta$, including an informed source's choice of what to reveal, enters through a charged auxiliary call. We write $N$ for the number of verifications, $Q$ for the number of auxiliary calls, and $\acc=f_\Theta(\widehat y)$. The horizon is finite. Sources are \emph{designed} when the learner may choose every kernel $K_t$, and \emph{fixed} when the kernels are given in advance and do not depend on verifications (made precise before \cref{thm:fixed}).

Verification differs from twenty questions because an accepted response is itself a valid output, so a learner without auxiliary calls is effectively a list. Let $\Ls$ be the lists of at most $n+1$ responses, and let
\begin{equation}\label{eq:coverage}
\begin{aligned}
 \cc&=\ind\{\exists y\in L:f_\theta(y)=1\},\\
 V_n(r)&=\max_{L\in\Ls}\textstyle\sum_\theta r_\theta\cc .
\end{aligned}
\end{equation}
\begin{proposition}[Local benchmark]\label{prop:coverage}
Without auxiliary calls, the best accuracy under $N\le n$ with prior $r$ is $V_n(r)$, and every stopping rule has
\begin{equation}\label{eq:stopped}
 \Pp_r(\acc=1,N\le n)\le V_n(r).
\end{equation}
\end{proposition}
Fix the seed and follow the trajectory on which every verification is rejected. Succeeding within $n$ verifications requires an accepted response on that trajectory or a correct final output, and these form a list of at most $n+1$ responses (\cref{app:coverage}). For unique answers, $V_n$ is the mass of the $n+1$ most probable answers. In general it is weighted maximum coverage \citep{nemhauser1978}.

\paragraph{Causal information.}
Let $\overline K_t(z\given\mathfrak h,c)=\sum_\theta P(\theta\given\mathfrak h,c)K_t(z\given\theta,\mathfrak h,c)$ be the posterior-predictive kernel, and charge
\begin{equation}\label{eq:budget}
 \aux=\textstyle\sum_t\E\big[\ind\{C_t\text{ auxiliary}\}\,\mathcal D_t\big],
\end{equation}
\[
 \mathcal D_t=\KL\big(K_t(\cdot\given\Theta,H_{t-1},C_t)\,\|\,\overline K_t(\cdot\given H_{t-1},C_t)\big),
\]
bits. This is the directed information from the world to the auxiliary observations, conditioned causally on everything else the learner has seen \citep{massey1990causality}. It discounts whatever verifications and earlier calls already revealed.

\subsection{The information frontier}\label{sec:main}
With reproduction alphabet $\Ls$ and distortion $1-\cc$, define the list-coverage rate--distortion function
\begin{equation}\label{eq:RD}
 \R(q,s)=\min\big\{I(\Theta;L):\ L\sim W(\cdot\given\Theta),\ \E\,c_{\Theta L}\ge s\big\}.
\end{equation}
Write $h_2$ for the binary entropy function, $\mathrm{kl}(a\|b)$ for the binary relative entropy, and set $\mathrm{kl}_+(a\|b)=\mathrm{kl}(a\|b)\ind\{a>b\}$.
\begin{theorem}[Information frontier]\label{thm:exact}
The infimum of $\aux$ over learners and auxiliary mechanisms with $N\le n$ almost surely and $\Pp(\acc=1)\ge s$ equals $\R(q,s)$. It is attained by a single auxiliary observation before any verification, and
\begin{equation}\label{eq:dual}
 \R(q,s)=\sup_{r:\,q\ll r}\big\{\mathrm{kl}_+\big(s\,\|\,V_n(r)\big)-\KL(q\|r)\big\}.
\end{equation}
Under an arbitrary stopping rule, $\aux\ge\R\big(q,\Pp(\acc=1,N\le n)\big)$.
\end{theorem}
\emph{Proof idea.} Build a reference law $\widetilde P$ with prior $r$, the same learner and the same exact verifications, but with every auxiliary kernel replaced by its posterior predictive $\overline K_t$. Verification transitions are identical under the two laws, so they cancel from the likelihood ratio, and over complete histories
\begin{equation}\label{eq:pathKL}
 \KL(P\,\|\,\widetilde P)=\KL(q\|r)+\aux .
\end{equation}
Under $\widetilde P$ every auxiliary observation can be simulated from the history, so \cref{prop:coverage} gives $\widetilde P(\acc=1,N\le n)\le V_n(r)$, and data processing to that event, as in the fundamental inequality of \citet{garivier2019explore}, yields the following (\cref{app:selective}).
\begin{lemma}[Selective-channel converse]\label{lem:selective}
For every protocol and every reference prior $r$ with $q\ll r$,
\begin{equation}\label{eq:selective}
 \mathrm{kl}_+\big(\Pp(\acc=1,N\le n)\,\|\,V_n(r)\big)\le\KL(q\|r)+\aux .
\end{equation}
\end{lemma} Optimizing $r$ makes the bound exact: \eqref{eq:dual} is the rate--distortion dual of \citet{csiszar1974extremum} written in the coordinates of a reference prior (\cref{app:duality}). For achievability, send an optimal list, verify all but its last entry, and output the last entry if none is accepted.

For unique answers, uniform $q$ and $n+1<M$, the frontier is the list-Fano value $\mathrm{kl}_+(s\,\|\,(n+1)/M)$ \citep{sakai2020}, so each optimally placed bit halves the verification budget. For $q=(0.8,0.1,0.1)$ and $n=1$, zero error costs $\R(q,1)=0.2$ bits: list answer~1 together with answer~2 or~3, each with probability $1/2$ (\cref{cor:zero} in \cref{sec:zero}); the value at $r=q$ is only $\log_2(1/0.9)=0.152$. With overlapping acceptance sets, taking $r=q$ in~\eqref{eq:dual} gives only the lower bound $\mathrm{kl}_+(s\|V_n(q))$, and the optimal $r$ need not be uniform even when $q$ is (\cref{sec:zero}). The theorem prices an optimally \emph{designed} source; a fixed source with the same mutual information can be far less useful (\cref{sec:channels}): as in the comparison of experiments \citep{blackwell1953equivalent}, no single number such as mutual information ranks sources for every task.\looseness=-1

\subsection{From information to calls}\label{sec:fixed}
\begin{figure}[t]
\centering\includegraphics[width=\columnwidth]{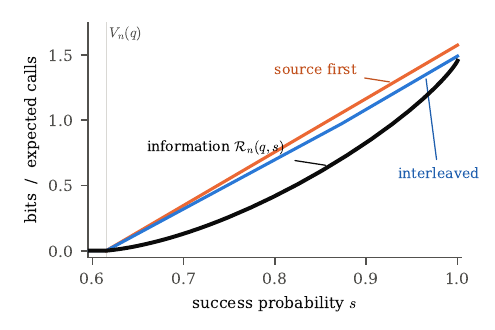}
\caption{\textbf{Information and call frontiers} for a Zipf prior $q_i\propto i^{-1.2}$ on $M=8$ unique answers with $n=1$ verification. Black: $\R(q,s)$. Blue and orange: the exact minimum expected number of binary calls over all interleaved and all source-first protocols with designed sources, computed by dynamic programming over candidate sets. They exceed the information price by at most \FrontierGapInter{} and \FrontierGapSF{} calls; \cref{thm:calls}(ii) guarantees less than $1+\log_25$.}\label{fig:frontiers}
\end{figure}
A learner pays for sources in calls. The next results convert the information price into calls.
\begin{theorem}[Calls and information]\label{thm:calls}
Call sources \emph{binary} when each auxiliary observation takes two values.
\begin{enumerate}
\item[(i)] Every protocol has $\aux\le\E\sum_t\log_2|\mathcal Z_t|$ over its auxiliary observations. With exact verification and binary sources, any protocol with $N\le n$ almost surely and $\Pp(\acc=1)\ge s$ therefore makes $\E Q\ge\R(q,s)$ calls, whatever its source mechanisms or order.
\item[(ii)] For unique answers the bound is tight to an additive constant. For every $s$, designed binary sources that ignore target history, with no shared randomness, support a source-first protocol with $N\le n$, $\Pp(\acc=1)\ge s$, a bounded number of calls and
\begin{align*}
 \E Q&<\R(q,s)+1+h_2(s)+2s\\
 &\le\R(q,s)+1+\log_2 5 .
\end{align*}
At zero error, sources that are deterministic functions of the world give $\E Q<\R(q,1)+2$, and no constant below $1$ is possible.
\end{enumerate}
\end{theorem}
\emph{Why the bounds hold.} Part (i) holds because a binary observation carries at most one bit. For (ii) at zero error, the optimal list channel includes world $\theta$ with probability $a_\theta=\min\{1,q_\theta/\tau\}$, where $\sum_\theta a_\theta=n+1$ (\cref{cor:zero}). Code lengths $\lceil\log_2(1/a_\theta)\rceil+1$ satisfy a list version of the Kraft inequality once worlds are grouped $n+1$ at a time in order of length. The sources can therefore reveal a prefix-free code for $\theta$'s group, and the learner then verifies the group. For $s<1$, conditioning any list channel with coverage $s$ on coverage shows $I(\Theta;L)\ge s\,\mathcal R_n(q_1,1)-h_2(s)$, where $q_{1,\theta}\propto q_\theta c_\theta$ and $c_\theta$ is the probability that $\theta$ is covered. One stochastic call that reports coverage, followed by the zero-error protocol for $q_1$, attains the bound (\cref{app:calls}).

For unique answers the list structure keeps the excess constant; for general acceptance sets it does not.
\begin{proposition}[General acceptance sets]\label{prop:general}
No constant suffices for general acceptance sets, for any $n$. Let $n=0$, let the responses be the independent sets of a graph $G$ on $[M]$, and let $f_\theta(I)=\ind\{\theta\in I\}$. Every zero-error protocol with binary sources has $\E Q\ge H_\chi(G,q)$, the least entropy of a colouring of $G$, while $\R(q,1)$ is the graph entropy $H_{\mathrm K}(G,q)$ of \citet{korner1973coding}; their difference is unbounded. Pairing each world with an index that an accepted response must match extends the gap to every $n$. The excess is at most logarithmic: sources that are deterministic functions of the world and the request achieve $\E Q\le\R(q,s)+\log_2(\R(q,s)+1)+5$ for every $s$.
\end{proposition}
For $n=0$ and unique answers, $\mathcal R_0(q,s)$ is Erokhin's $\epsilon$-entropy \citep{erokhin1958} ($\epsilon=1-s$), which governs variable-length compression with errors \citep{kostina2015variable}. At zero error, $\R(q,1)$ is a hypergraph entropy \citep{csiszar1990}; for graphs the chromatic entropy is at least the graph entropy \citep{alon1996source} and can exceed it without bound \citep[Prop.~2]{cardinal2008minimum}, which is \cref{prop:general}. For general acceptance sets, the strong functional representation lemma \citep{li2018strong} writes an optimal list as a function of the world and an independent variable $V$ with $H(L\given V)\le\R+\log_2(\R+1)+4$, as in one-shot channel simulation \citep{harsha2010communication}. Because $V$ is independent of $\Theta$, each value $v$ gives a deterministic list map $g_v$. At $s=1$ the learner takes the $v$ of positive probability that minimizes $H(g_v(\Theta))$; for $s<1$, a mixture of two values, chosen by a coin the learner sends in its request, meets the coverage constraint. The sources then send a Huffman codeword of $g_v(\Theta)$. At $n=0$ and $s=1$ the proof of \cref{prop:general} gives $H_\chi\le H_{\mathrm K}+\log_2(H_{\mathrm K}+1)+4$ for every graph, an existential analogue of the bound proved for greedy colourings of perfect graphs \citep{cardinal2010efficient,cardinal2012minimum}. The logarithmic order is attained: on the comparability graphs of the interval orders $I_k$ of \citet{cardinal2010efficient}, with $M$ vertices and a uniform prior, $H_\chi-H_{\mathrm K}=\Omega(\log\log M)$, and $H_{\mathrm K}\le\log_2M$, so the excess in \cref{prop:general} can be $\Omega(\log H_{\mathrm K})$.
\begin{proposition}[What interleaving saves]\label{prop:interleave}
Against a given set of sources, interleaving can save an unbounded number of expected calls, even with exact verification. For unique answers, $q=(1-\epsilon,\epsilon/K,\dots,\epsilon/K)$, $n=1$, $s=1$ and sources that return the $\log_2K$ bits ($K\ge4$ a power of two) of an index in which world~1 shares the index of world~2, an interleaved protocol makes $\epsilon\log_2K$ expected calls and every source-first protocol at least $(1-\epsilon)\log_2K$. Interleaving can also divide $\E Q$ by an unbounded factor: for $q=(1-2p,p,p)$ with $0<p\le1/4$, $n=1$ and $s=1$, an interleaved protocol makes $\E Q=2p=\R(q,1)$ calls, while every source-first protocol makes $\E Q\ge1$. Against designed sources and unique answers, however, \cref{thm:calls} bounds the saving by $1+\log_2 5<3.33$ calls, and by $2$ at zero error.
\end{proposition}
\Cref{prop:interleave} shows why the next theorem is stated for hard caps. Precisely, sources are fixed when each kernel may depend on the world and on earlier source requests and replies, but not on whether or how the target has been verified; the learner may still choose sources adaptively. Let $Z_\sigma$ be the transcript of a source-only policy $\sigma$.
\begin{theorem}[Source-first suffices for fixed sources]\label{thm:fixed}
With exact verification and fixed sources, any interleaved protocol with $Q\le b$ and $N\le n$ can be replaced, with the same source mechanisms, by a source-first protocol: all its source calls, then all its verifications. The replacement has at least the same success probability and obeys the same caps, including a total cap $Q+N\le C$. Hence
\begin{equation}\label{eq:fixedvalue}
 \mathsf A_{b,n}:=\sup\Pp(\acc=1)=\sup_{\sigma:\,Q\le b}\E\,V_n\big(q(\cdot\given Z_\sigma)\big).
\end{equation}
\end{theorem}
The replacement simulates the learner with every verification rejected, makes its source calls, records its candidate responses and final output, and verifies that list afterwards. Coupling through the first true acceptance shows that nothing is lost (\cref{app:fixed}). Equation~\eqref{eq:fixedvalue} is sequential value of information with terminal utility $V_n$ \citep{krause2009,golovin2011adaptive}. Exactness of verification cannot be dropped:
\begin{proposition}[What survives a noisy verifier]\label{prop:noisy}
Let a verification of $y$ return a bit drawn from a binary-input kernel $G(\cdot\given f_\Theta(y))$, independent of the history, and let $V^G_n(r)$ be the best value of $\Pp_r(\acc=1,N\le n)$ without auxiliary information.
\begin{enumerate}
\item[(i)] The converse survives in two forms: $\mathrm{kl}_+\big(\Pp(\acc=1,N\le n)\,\|\,V^G_n(r)\big)\le\KL(q\|r)+\aux$ for every $r$ with $q\ll r$, and, with $N\le n$ almost surely and $\Pp(\acc=1)\ge s$, $\aux\ge\mathcal R_0(q,s)-n\,\operatorname{cap}(G)$, where $\operatorname{cap}(G)$ is the capacity of $G$.
\item[(ii)] Neither reduction survives. Source-first protocols can need an unbounded factor more information than interleaved ones: for one flipped query at zero error, $1$ bit against $h_2(\epsilon)$ (\cref{thm:noise}). At equal call caps, take a uniform binary answer $A$, let the verifier flip its reply with probability $\epsilon$, and let two fixed sources report $\ind\{A\ne r\}$ for $r=0,1$, each missing a true $1$ with probability $\eta$. With one call and one verification, interleaving reaches error $\epsilon\eta$, while every source-first protocol has error at least $\min\{\eta/2,\epsilon(1+\eta)/2\}$, which the best one attains. At $\epsilon=\eta$ the ratio is $1/(2\eta)$, so \cref{thm:fixed} fails.
\end{enumerate}
\end{proposition}
The first form of (i) repeats the proof of \eqref{eq:selective} with the noisy kernel inside the reference law. Since a noisy verifier is a garbling of the exact one, $V^G_n\le V_n$, and \cref{thm:calls}(i) extends to it. The second form is the chain rule: the transcript carries at most $\aux+n\,\operatorname{cap}(G)$ bits about $\Theta$, and the output alone must carry $\mathcal R_0(q,s)$. On the instance of \cref{thm:noise} the first form is loose, while the second gives $h_2(\epsilon)-h_2(\delta)$ at error $\delta$, which is exactly the interleaved optimum (\cref{sec:noise}). With exact verification the second form gives only $\mathcal R_0(q,s)-n$, which is weaker than \cref{thm:exact}: an exact verification is worth more than its one bit of capacity, because an accepted response is itself the output. In part (ii) the interleaved protocol uses the noisy reply to choose which source to call, and one extra call closes the gap: calling both sources first also reaches $\epsilon\eta$.
\section{Linear task banks}\label{sec:planner}
\paragraph{Rank law.}
Let $W$ be uniform on $\F^d$, let the target be $A=TW$ for a known $k\times d$ matrix $T$ of rank $k$, and let a bank of $J$ binary source tasks return $Z_j=v_jW$. A source task is itself verifiable: one source call, verifying the candidate $0$, reveals $Z_j$ whether it is accepted or rejected. A target verification tests an entire $k$-bit vector. For a selected set $S$ of sources with rows $B_S$, the number of target dimensions it resolves is
\begin{equation}\label{eq:rank}
 \rho(S)=k+\operatorname{rank}(B_S)-\operatorname{rank}\!\begin{bmatrix}T\\B_S\end{bmatrix}=I(A;Z_S),
\end{equation}
and $\rho_b=\max_{|S|\le b}\rho(S)$.
This is the equivocation identity of linear secret sharing \citep{kurihara2012}: rows in $B_S$ that do not reduce the target's uncertainty only reveal nuisance directions of $W$.
\begin{theorem}[Exact accuracy]\label{thm:rank}
For arbitrary adaptive interleaving,
\begin{align}
 \mathsf A_{b,n}&=\min\{1,(n+1)2^{-k+\rho_b}\},\label{eq:rankrect}\\
 \mathsf A_C&=\max_{0\le b\le\min(C,J)}\min\{1,(C-b+1)2^{-k+\rho_b}\}.\label{eq:ranktotal}
\end{align}
A fixed maximizing subset followed by enumeration of the compatible targets attains both. Neither outcome-adaptive source selection nor adaptive budget allocation helps.
\end{theorem}
On every source-only leaf the posterior of $W$ is uniform on an affine fiber, so the target is uniform on $2^{k-\rho(S)}$ vectors. \cref{thm:fixed} and \cref{prop:coverage} finish the proof (\cref{app:rank}). We write $C^\star$ for the least total cap with $\mathsf A_C=1$ and $(b^\star,n^\star)$ for an optimal split of it. Target-only search needs $2^k-1$ verifications for zero error; a bank with $\rho_b=b$ for $b\le k$ needs only $k$ source calls.

\paragraph{Computing the profile.}
Choosing a maximizing subset contains syndrome decoding, which is NP-hard \citep{berlekamp1978}, so with unbounded nuisance no algorithm polynomial in $J$ is expected. The tractable parameter is the \emph{nuisance dimension}. Let $\mathcal T=\operatorname{row}(T)$, $\mathcal V=\operatorname{span}(T,v_1,\dots,v_J)$ and $h=\dim\mathcal V-k$. Let $\bar v_j$ be the image of $v_j$ in the $h$-dimensional quotient $\mathcal V/\mathcal T$. For a subspace $U$ of the quotient spanned by some of the $\bar v_j$, define
\begin{equation}\label{eq:eligible}
\begin{aligned}
 \mathcal E(U)&=\{j:\bar v_j\in U\},\\
 t(U)&=\operatorname{rank}\{v_j:j\in\mathcal E(U)\}-\dim U .
\end{aligned}
\end{equation}
These are the sources whose nuisance content lies in $U$, and the target rank they carry once $U$ is paid for.
\begin{theorem}[Quotient-subspace planner, QSP]\label{thm:qsp}
For every source budget $b$, with $(x)_+=\max\{x,0\}$,
\begin{equation}\label{eq:qsp}
 \rho_b=\max_{U}\ \min\{t(U),\,(b-\dim U)_+\}.
\end{equation}
After a polynomial-time reduction, all budgets and maximizing subsets are computed in $\mathcal N_h\operatorname{poly}(J,k+h)$ field operations, where $\mathcal N_h=\sum_{u=0}^h{h\brack u}_2=2^{O(h^2)}$ counts the subspaces of $\F^h$.
\end{theorem}
A subset whose images span $U$ gains at most $t(U)$ and at most $|S|-\dim U$. Conversely, a basis of $U$ extended to a basis of the eligible rows realizes every gain up to $t(U)$ with $\dim U$ plus that many sources (\cref{app:planner}). The static profile is classical. Let $K_0\subseteq K_1\subseteq\F^J$ be the kernels of the linear maps that send the $j$th unit vector to $v_j$ and to $\bar v_j$; then
\begin{equation}\label{eq:relative}
 \rho(S)=\dim(K_1\cap\F^S)-\dim(K_0\cap\F^S),
\end{equation}
so $\min\{b:\rho_b\ge t\}$ is the $t$th relative generalized Hamming weight of the pair \citep{luo2005,kurihara2012,zhuang2014relative}. QSP enumerates the $2^{O(h^2)}$ subspaces of the nuisance quotient instead of the $2^J$ subsets of sources, which is enumeration on the small side of Forney's duality \citep{forney1994dimension} (\cref{app:coding}). General algorithms exist \citep{sanjose2025}. For $k=1$, a breadth-first search over $\F^{h+1}$ decides $\rho_b$ in $2^{h+1}J$ steps, so we do not claim that the exponent is optimal.

\paragraph{The call price of nuisance.}
\cref{thm:qsp} also says what nuisance costs in calls under a hard cap.
\begin{corollary}[In a linear bank the call excess is nuisance]\label{cor:waste}
In a uniform linear bank, take zero error with $n$ verifications, $n+1<2^k$, and let $\kappa=\lceil k-\log_2(n+1)\rceil=\lceil\R(q,1)\rceil$. A hard source cap $b$ suffices if and only if $b\ge b_n$, where
\[
 b_n=\kappa+u(\kappa),\qquad u(\kappa)=\min\{\dim U:\ t(U)\ge\kappa\},
\]
over the subspaces $U$ of \cref{thm:qsp} ($b_n=\infty$ if $\rho_J<\kappa$). The same $b_n$ is the least expected number of calls of a source-first protocol. The excess over the rounded-up information price, $b_n-\kappa=u(\kappa)\le h$, is the number of calls an optimal subset spends on nuisance directions, and it vanishes exactly when an optimal subset resolves one new target dimension per call.
\end{corollary}
The proof combines \eqref{eq:rankrect}, \eqref{eq:qsp} and the list-Fano value (\cref{app:calls}). The subspace $U$ is what the learner must pay for, one call per dimension, before the remaining calls resolve one target dimension each. Interleaved protocols can do better in expectation. With $k=2$, two sources that reveal one target bit each, and $n=2$, verifying two candidates first and calling a source only if both are rejected uses $\E Q=1/2$ calls, against $b_n=1$.

\paragraph{Learning the bank, and checks.}
A bank of rank $\ell\le\ell_0$ is recovered from $\ell_0+\lceil\log_2(1/\delta)\rceil$ fully labelled contexts with probability $1-\delta$, after which QSP plans optimally (\cref{thm:learnlinear}; \citealp{helmbold1992learning}). Exhaustive checks on small banks agree with QSP and with \cref{cor:waste} (\cref{app:vfourexp,app:calls}). We also ran the unmodified generalized-Hamming-weight package GHWs of \citet{sanjose2025} on 27 banks under a 15-second cap: profiles agree wherever both finish, QSP finishes all 27 and GHWs 21, and GHWs is faster on high-nuisance banks and on the Reed--Muller code RM(2,4) (\cref{sec:fullexternal}). Heuristics do fall short of the optimum. On the 40 label-fitted banks of \cref{sec:fullexternal}, the best of all our source-aware baselines reaches accuracy $\VFiveIIDPortfolio$ on the 16 independent and $\VFiveClusterPortfolio$ on the 16 nuisance-clustered banks at the exact zero-error cap $C^\star$, where QSP reaches one by definition, and matches QSP there on the eight coded variants; averaged over all budgets, its gaps on the first two groups are only $\VFiveIIDPortfolioRegret$ and $\VFiveClusterPortfolioRegret$ percentage points. On the $d=11$ banks of \cref{sec:neuralexp}, one-step greedy delivers \CurElevenRhoGreedy{} bits where the optimum is \CurElevenRhoQsp.
\section{The frontier as a ruler}\label{sec:neuralexp}
\begin{figure*}[t]
\centering\includegraphics[trim=0 0 67 0,clip,width=0.8671\textwidth]{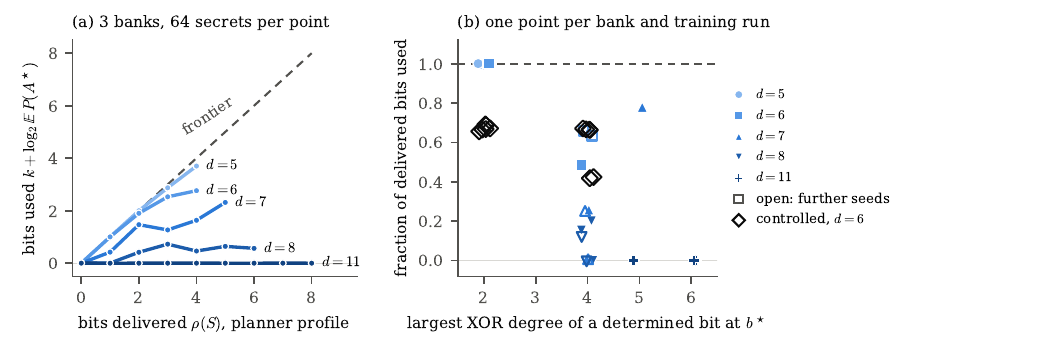}
\caption{\textbf{The frontier as a ruler.} (a) Bits a two-layer transformer uses against bits the planner's sources deliver, pooled over three banks per latent dimension $d$ and 64 secrets per point (the zero-error-split fractions in the text use four secrets per bank); the dashed line is the exact posterior. (b) Fraction of delivered bits used at the zero-error split, one point per bank and training run, against the largest XOR degree among the determined target bits (small horizontal jitter). Filled markers are the 15 banks of (a) with their first training seed, open markers the two further seeds of bank~0 at $d\in\{6,7,8\}$, and diamonds the degree-controlled banks at $(k,h,J)=(4,2,12)$, selected and predicted before training.}\label{fig:ruler}
\end{figure*}
On a uniform linear bank the exact posterior after acquiring a set $S$ with replies $z$ is uniform on the $2^{k-\rho(S)}$ compatible targets. A trained network that receives the same labels induces its own distribution over the $2^k$ targets, which we compute exactly by teacher forcing, and for a secret $W^\star$ with target $A^\star=TW^\star$ we measure
\begin{align}
 \text{bits used}&=k+\log_2\E\,P(A^\star\given\text{acquired})\nonumber\\
 &\le k+\log_2\E\,2^{\rho(S)-k}=\text{bits delivered},\label{eq:bitsused}
\end{align}
with the expectation over the secret $W^\star$ and, when banks are pooled, over banks. The inequality holds because, given the labels, the true target's expected probability under any predictor is at most $2^{\rho-k}$; with four secrets per bank an estimate can exceed it by sampling noise (2.01 against 2.00 on one bank). Unused bits cost verifications: sampling the policy until a verification accepts takes $\E[1/P(A^\star)]\ge2^{k-\text{bits used}}$ samples, against $\E\,2^{k-\rho}$ for the exact posterior.

\paragraph{How much is used.}
The learner is a two-layer transformer ($0.4$M parameters) trained from scratch on sequences that list acquired (source, label) pairs followed by the $k$ target bits; at test time the planner chooses which labels appear, one call each. We use latent dimensions $d=k+h\in\{5,6,7,8,11\}$ with three random banks each and evaluate at the zero-error split $(b^\star,n^\star)$ of \cref{sec:planner} (setup in \cref{app:controls27}). At $d=5$ the decoder is on the frontier: all \ICexactDFive{} of \ICrunsDFive{} contexts are within $0.2$ bit of their delivered $\rho(S)$ (\cref{fig:ruler}). With the planner's sources, the fraction used (pooled over banks) then falls to \ICfracDSix{} at $d=6$, \ICfracDSeven{} at $d=7$, \ICfracDEight{} at $d=8$ and \ICfracDEleven{} at $d=11$, although the signal is there to learn: the Bayes-optimal pretraining loss at $d=11$ is \BayesLossDEleven{} nats against $\ln2=0.69$ for a model that ignores the labels. Training seeds move single banks substantially (seed ranges on bank~0 in \cref{app:controls27}), and none of 8 configurations at $d=8$ (14 runs, widths 64--512, up to $20{,}000$ steps) reached the frontier (\cref{tab:capacity}). The unused bits have a price: at $d=8$ the planner's sources leave the learner \CurEightElimQsp{} guesses against the exact posterior's \CurEightBayesQsp{}, so a better choice of sources helps only as far as the decoder can use it (\cref{app:controls27}).

\paragraph{Which bits are used.}
Each determined target bit is the XOR of some acquired labels and earlier, teacher-forced target bits; the fewest such terms is its \emph{degree}. Across these 15 banks the decoder is on the frontier on exactly the five banks where every determined bit has degree at most two (\cref{fig:ruler}(b)), which suggests that degree decides use. A controlled test says otherwise. \DegreeTestSentence{} Prediction P1 failed; P2 held on the controlled banks, but its bound fails on two of the 15 banks above, so at this budget degree neither guarantees nor bounds use (\cref{app:controls27}). At four times the training budget three of the six controlled banks recover further bits. What the ruler adds to known parity failures \citep{barak2022,bhattamishra2024understanding} is a count of the unused delivered bits against an exact reference, and their price in verifications.
\section{Related work}\label{sec:related}
\paragraph{Verification, guessing and exact learning.}
A verifier that accepts whole answers is an equivalence-query oracle, and the target-only lower bound is the all-``no'' adversary of exact learning \citep{angluin1988queries,hegedus1995generalized}. For unique answers, the target-only problem is guessing \citep{massey1994,arikan1996}. How much side information reduces guessing has been studied for a fixed helper, a compressed description, or a distortion criterion \citep{arikan1998guessing,bunte2014encoding,weinberger2020guessing,graczyk2021}. The static list rate--distortion function and its list-Fano special case are classical \citep{csiszar1974extremum,sakai2020}. We add the interleaved, causal version: adaptivity buys nothing in information under exact verification (\cref{thm:exact}), and \cref{thm:calls,prop:general,prop:interleave,thm:fixed} say what it buys in calls.

\paragraph{Lower bounds with interaction and value of information.}
The converse is a change of measure \citep{kaufmann2016,simchowitz2017,garivier2019explore} of the interactive-Fano type \citep{chen2024assouad}. The difference is that exact verifications stay in the reference law, where they cancel and are never charged. Charging interaction by information rather than by messages follows information complexity \citep{braverman2014information}; the gap between the information price and the expected number of calls (\cref{thm:calls,prop:general}) rests on the gap between graph and chromatic entropy in source coding with side information \citep{alon1996source}, whose interactive version is studied by \citet{orlitsky1992average}. Noisy verification is a search game with lies \citep{pelc2002searching}. In noisy twenty questions a non-adaptive policy is optimal for entropy loss \citep{jedynak2012twenty}, much as one upfront observation attains \cref{thm:exact}, whereas for error probability adaptive query selection matters \citep{naghshvar2013active}, as interleaving does in \cref{prop:noisy}. \Cref{eq:fixedvalue} is sequential value of information \citep{krause2005,krause2009,golovin2011adaptive} with list coverage as the terminal utility.

\paragraph{Reinforcement learning from verifiable rewards.}
Recent theory studies coverage and sharpening \citep{chen2026coverage,huang2025self}, autocurricula \citep{rajaraman2026,rajaraman2026learning}, compositional structure \citep{barzilai2026}, budgets \citep{wachi2026}, the limits of outcome feedback \citep{chen2025outcome}, and the query complexity of generation with a verifier \citep{botta2025query}. Without sources, $V_n(q)$ is the best pass@$(n+1)$ \citep{chen2021evaluating,brown2024large} of any list of $n+1$ answers. Those works analyse learning or generation across prompts; we price what source tasks deliver about a single instance.

\paragraph{Codes and parities.}
The profile $\rho_b$ is a relative dimension/length profile \citep{wei1991generalized,forney1994dimension,luo2005,kurihara2012,zhuang2014relative}, with a general solver \citep{sanjose2025}. That delivered information need not be usable information is the premise of $\mathcal V$-information \citep{xu2020theory}. The ruler of \cref{sec:neuralexp} measures that gap against an exact reference and prices it in verifications.

\section{Limitations and conclusion}\label{sec:conclusion}
The frontiers assume finite alphabets. \Cref{thm:exact} prices designed sources and \cref{thm:fixed} fixed ones. Reordering to source-first preserves hard caps, not expected calls, and it needs exact verification (\cref{prop:interleave,prop:noisy}); under noise we have converses, one of them exact on our instance, but no general frontier. Under hard caps the call frontier is exact for fixed sources, as the recursion~\eqref{eq:fixedvalue}, and in closed form for linear banks. In expectation it is tight to an additive constant for unique answers and to a logarithmic term in general, where the excess can be unbounded. The rank law, planner and learner assume uniform, noiseless linear banks, and the planner is efficient only for small nuisance dimension. The neural measurements use small transformers trained from scratch on synthetic banks, with three banks per dimension, few training seeds and fixed budgets of $6{,}000$ pretraining steps ($8{,}000$ at $d=11$); at $24{,}000$ steps three of the six controlled banks recover further bits, so the fractions used depend on the budget. The degree analysis rests on few banks, and none of these measurements is a claim about pretrained language models.

\paragraph{What this means for curricula.} With an exact verifier and sources whose replies do not depend on verification outcomes, all source calls can be scheduled first at no cost under any hard cap on calls (\cref{thm:fixed}); interleaving can still save expected calls (\cref{prop:interleave}) and information (\cref{app:fixed}). With a noisy verifier, interleaving can reduce the error at equal call caps by an unbounded factor (\cref{prop:noisy}), and the information needed for zero error by an unbounded factor (\cref{thm:noise}). The information price $\R(q,s)$ lower-bounds the expected number of binary calls, and designed sources meet it within $1+\log_25$ calls for unique answers (\cref{thm:calls}); in a linear bank at zero error under a hard cap, the calls beyond the rounded-up price go to nuisance directions that the planner identifies (\cref{cor:waste}). Whether a learned decoder uses what the sources deliver is a separate question, which the ruler measures.

\section*{AI Use Statement}
In this work, we used generative AI tools for helping to develop the theoretical framework and to formulate mathematical claims, for providing ingredients for their proofs and drafting them, for proposing and refining hypotheses, designing experiments and giving feedback on the methodology, and for implementing the numerical methods and experiments and interpreting their results. We have not used generative AI tools to generate data sets or results: the synthetic task banks are generated by released code, and every reported number is computed by that code from logged runs. Translation, data-set cleaning and qualitative or thematic data analysis are not applicable to this work. Additionally, we used generative AI tools for literature search and checking the bibliography, for writing code, figures and the Lean formalization, and for drafting and editing the text. We have reviewed all AI-assisted work. We checked all AI-assisted proofs, code and descriptions of cited work, and the numbered results are also machine-checked in Lean~4, with the exceptions and the two assumed published results stated at the start of the appendix. We take responsibility for the final content of this work, including text, claims, or artifacts produced with the aid of generative AI.

\nocite{sagemath2025,paszke2019pytorch}
\bibliographystyle{plainnat}
\bibliography{references}

\clearpage
\appendix
\onecolumn
\renewcommand{\thesection}{\AlphAlph{\value{section}}}
\paragraph{Contents.} \Cref{app:coverage,app:selective,app:duality,app:zero,sec:zero,app:illustrations,app:fixed,app:rank,app:calls,sec:noise,app:noise,sec:channels} hold the proofs for \cref{sec:model} and for the rank law, in this order: local benchmark, selective-channel converse, dual, zero error, frontier illustrations, fixed sources, rank law, calls and \cref{cor:waste}, noisy verification, and what information quantity leaves out. \Cref{app:planner,app:coding,app:learnlinear} give the planner, its coding-theory form and learning the bank; \cref{app:vfourexp,sec:fullexternal,app:vfive} the planner experiments; \cref{app:vfiveexp,app:controls27} the neural experiments with the degree analysis and controls; \cref{app:compute} numerical methods; and \cref{app:relatedfull} extended related work.

\paragraph{Where each result is proved.} \Cref{prop:coverage}: \cref{app:coverage}. \Cref{thm:exact}: \cref{app:selective,app:duality}. \Cref{lem:selective}: \cref{app:selective}. \Cref{thm:calls,prop:general,prop:interleave,prop:noisy,cor:waste}: \cref{app:calls}. \Cref{thm:fixed}: \cref{app:fixed}. \Cref{thm:rank}: \cref{app:rank}. \Cref{thm:qsp}: \cref{app:planner}. \Cref{cor:zero}: \cref{app:zero} (statement in \cref{sec:zero}). \Cref{thm:noise}: \cref{app:noise} (statement in \cref{sec:noise}). \Cref{thm:learnlinear}: \cref{app:learnlinear}.

\paragraph{Machine-checked results.} All numbered results are formalized in Lean~4 \citep{demoura2021lean4} with Mathlib \citep{mathlib2020} and checked by the Lean kernel (ancillary files, \texttt{lean/}; \texttt{Paper.lean} maps each statement to its Lean counterpart and records the modelling choices), except the algorithmic claim of \cref{thm:qsp} (correctness and running time of the enumeration) and the last clause of \cref{thm:learnlinear}. Two published results \citep{cardinal2008minimum,li2018strong} are assumed as axioms, both used only for \cref{prop:general}. Some formal proofs take a different route from the written ones. The numerical and empirical claims, and a few unnumbered remarks listed in the ancillary files, are not formalized.

\section{Complete proof of the local benchmark}\label{app:coverage}
All alphabets and the protocol horizon are finite. A private seed may have any distribution independent of $\Theta$; the arguments below first condition on that seed. Any history-only stochastic kernel can be realized from an independent uniform random variable. Include all such uniforms in a complete seed $u$. Thus, when a protocol has no informative auxiliary channel, fixing $u$ makes every action deterministic in the past target replies.

Run this deterministic protocol with every target reply forced to zero, including at histories that are impossible for a particular world. The policy has an arbitrary fixed extension to such histories. Form $L_n(u)$ from the first $n$ target queries on this silent trajectory, or all queries if it stops sooner. Include its final output only if it stops after at most $n$ queries. If it has not stopped by its $(n+1)$st query, no fallback is needed. Repeated responses are kept only once, so $|L_n(u)|\le n+1$.

For a real trajectory on which $\acc=1,N\le n$, there are two possibilities. If a positive target reply occurs, the first such reply follows the same actions as the silent trajectory up to that time. Its acceptable queried response belongs to $L_n(u)$. If no positive reply occurs, the entire real trajectory agrees with the silent trajectory through stopping, and its acceptable output is the included fallback. Consequently,
\[
 \{\theta:\acc(\theta,u)=1,N(\theta,u)\le n\}
 \subseteq\bigcup_{y\in L_n(u)}\{\theta:f_\theta(y)=1\}.
\]
The prior mass of this union is at most $V_n(r)$. Average over the independent seed to prove~\eqref{eq:stopped}. If the protocol obeys a hard cap, the same event is simply $\acc=1$.

For attainability, take a maximizing list $L=\{y_1,\ldots,y_k\}$ with $k\le n+1$. Query $y_1,\ldots,y_{k-1}$ until one receives reward one and then output that response. If none succeeds, output $y_k$. This succeeds precisely in worlds covered by $L$. Multiple acceptable responses and overlap between their acceptance sets cause no difficulty. Computationally, this is maximum coverage rather than arbitrary twenty-questions search: after a positive reply, the goal has already been achieved.

\section{Complete selective-channel construction}\label{app:selective}
Write the history as $(\xi,H_{t-1})$, including all randomness belonging to the learner, and treat the learner's action $C_t$ as a deterministic function of it. A target-query transition returns $f_\Theta(y_t)$; a source transition has kernel $K_t(\cdot\mid\Theta,H_{t-1},C_t)$. Stop and padding transitions are deterministic. Any randomized final output is included in $\xi$.

Under $P$, the initial law is $q\otimes P_\xi$. Define $\overline K_t$ by the $P$ posterior predictive at every positive-$P$ auxiliary history. At other histories choose any fixed probability distribution on the observation alphabet. Define $\widetilde P$ using initial law $r\otimes P_\xi$, the same action functions and target replies, and auxiliary kernels $\overline K_t$.

\paragraph{Absolute continuity.}
At any positive-$P$ atom $(\theta,u,\mathfrak h)$, $q_\theta>0$ implies $r_\theta>0$. At every auxiliary node on its path, the conditional posterior gives positive weight to $\theta$. Thus an observation with $K_t(z\mid\theta,\mathfrak h,c)>0$ also has $\overline K_t(z\mid \mathfrak h,c)>0$. Every positive-$P$ terminal path has positive $\widetilde P$ probability. For general private seeds this statement holds for almost every $u$, and conditional integration gives the same result. The extension at null histories does not enter any $P$-expected divergence.

\paragraph{Likelihood ratio.}
Along such a path, private-seed probabilities and learner-action transitions cancel; target, stop, and padding transitions agree. Only the prior and source factors remain:
\begin{equation}\label{eq:ratioapp}
 \log_2\frac{P(\theta,u,\mathfrak h_T)}{\widetilde P(\theta,u,\mathfrak h_T)}
 =\log_2\frac{q_\theta}{r_\theta}
 +\sum_{t:C_t\ \mathrm{auxiliary}}
 \log_2\frac{K_t(z_t\mid\theta,\mathfrak h_{t-1},c_t)}
                  {\overline K_t(z_t\mid \mathfrak h_{t-1},c_t)}.
\end{equation}
Taking expectation gives~\eqref{eq:pathKL}. Target feedback may be informative under both laws, but it is not charged a second time: its conditional experiment is identical under the two laws.

\paragraph{Why the reference has no auxiliary information.}
At a fixed history, $\overline K_t$ does not depend on the realized $\theta$. It may be a table constructed from the true model $P$ and the entire learner, but it can be simulated using only its history and independent random bits. It is therefore part of a target-only randomized protocol under initial prior $r$. Such a protocol need not have the same unconditional distribution of transcripts as the original one. Nor must it obey a cap that held only $P$-almost surely; this is precisely why Proposition~\ref{prop:coverage} was stated for joint success--budget events.

Set $E_n=\{\acc=1,N\le n\}$, $a=P(E_n)$, $b=\widetilde P(E_n)$, and $c=V_n(r)$. Data processing gives $\mathrm{kl}(a\|b)\le\KL(q\|r)+\aux$, and Proposition~\ref{prop:coverage} gives $b\le c$. For $a>c$, binary KL decreases in its second argument below $a$, so $\mathrm{kl}(a\|c)\le \mathrm{kl}(a\|b)$. For $a\le c$, the one-sided divergence is zero. Endpoint cases follow by absolute continuity and the usual extended-value conventions. This proves Lemma~\ref{lem:selective}.

A useful consequence for unconstrained stopping is, for every fixed integer $n$,
\[
 P(\acc=1)\le P(N>n)+
 \sup\{s\in[0,1]:\R(q,s)\le\aux\}.
\]
This is a tail decomposition, not an expectation of a fixed-budget frontier evaluated at $N$. No independence of $N$ from $\Theta$ is assumed.

\section{Rate--distortion dual and exact operational reduction}\label{app:duality}
This appendix proves~\eqref{eq:dual} through the Gibbs variational form
\begin{equation}\label{eq:gibbs}
 (\ln2)\R(q,s)=\sup_{\beta\ge0}\Big\{\beta s-\max_{\mu\in\Delta(\Ls)}\sum_\theta q_\theta\ln\sum_L\mu_Le^{\beta c_{\theta L}}\Big\},
\end{equation}
whose inner value is $\inf_{r>0}\{D(q\|r)+\ln(1+(e^\beta-1)V_n(r))\}$; the binary log-moment dual then yields~\eqref{eq:dual}. The result is the classical rate--distortion dual \citep{csiszar1974extremum} in the coordinates of a reference prior; we include a self-contained finite-alphabet derivation because the list-coverage distortion and the reference-prior form are what the protocol theorem uses. Only in this appendix use natural logarithms, with $D$ and $I$ in nats; write $R^{\mathrm{nat}}(s)=(\ln2)\R(q,s)$. The source prior has full support, and the union of covered worlds is all of $[M]$. The finite channel space is compact, mutual information is continuous there, and the success constraint is closed. Thus the minimum is attained for every $s\in[0,1]$.

\subsection{The finite-alphabet Gibbs formula}
For a channel $W$ and any list distribution $\mu$,
\begin{equation}\label{eq:decompose}
 \sum_\theta q_\theta D(W(\cdot\mid\theta)\|\mu)
 = I_q(\Theta;L)+D(qW\|\mu).
\end{equation}
It follows that $I_q(\Theta;L)$ is the minimum of the left side over $\mu$, attained at $\mu=qW$. The Lagrange multiplier for $\E c_{\Theta L}\ge s$ is $\beta\ge0$. For any fixed $\mu$,
\[
 \min_{W(\cdot\mid\theta)}\left[
 D(W(\cdot\mid\theta)\|\mu)-\beta\sum_LW(L\mid\theta)c_{\theta L}
 \right]
 =-\ln\sum_L\mu_Le^{\beta c_{\theta L}}.
\]
Indeed, the bracket equals divergence to the normalized Gibbs distribution
$\mu_Le^{\beta c_{\theta L}}/Z_\theta$ minus $\ln Z_\theta$.
The minimizations over $W$ and $\mu$ are over a product space and may be exchanged. Convex duality for the finite channel problem therefore gives~\eqref{eq:gibbs} for $s<1$, where a strictly better-than-$s$ success channel exists. For $s\le V_n(q)$ a constant maximizing list has zero information, and the same formula has value zero.

For $s=1$, choose $s_j\uparrow1$ and corresponding minimizing channels. A convergent subsequence has success one and, by continuity, information equal to the limit of their costs. Monotonicity then gives $R^{\mathrm{nat}}(s_j)\uparrow R^{\mathrm{nat}}(1)$. Taking increasing limits in the dual establishes the endpoint formula without assuming a finite optimal multiplier.

\subsection{Optimizing the reference prior}
Fix $\beta\ge0$ and put $E_{\theta L}=e^{\beta c_{\theta L}}$, $Z_\theta=\sum_L\mu_LE_{\theta L}$, and $F_\beta(\mu)=\sum_\theta q_\theta\ln Z_\theta$. For any strictly positive prior $r$, Jensen gives
\begin{align}
 F_\beta(\mu)
 &=D(q\|r)+\sum_\theta q_\theta\ln\frac{r_\theta Z_\theta}{q_\theta}\\
 &\le D(q\|r)+\ln\sum_\theta r_\theta Z_\theta\\
 &\le D(q\|r)+\ln\max_L\sum_\theta r_\theta E_{\theta L}\\
 &=D(q\|r)+\ln\big(1+(e^\beta-1)V_n(r)\big).
\end{align}
We next show that this upper bound is attained, establishing the identity
\begin{equation}\label{eq:bridge}
 \max_\mu F_\beta(\mu)=\inf_{r>0}\{D(q\|r)+\ln(1+(e^\beta-1)V_n(r))\}.
\end{equation}
Let $\mu^*$ maximize the smooth concave function $F_\beta$ over the simplex. Its gradient is
\[
 g_L=\sum_\theta\frac{q_\theta E_{\theta L}}{Z_\theta}.
\]
The simplex first-order condition gives $g_L\le\lambda$, with equality on the support of $\mu^*$. Since $\sum_L\mu^*_Lg_L=\sum_\theta q_\theta=1$, $\lambda=1$ and $\max_Lg_L=1$.
Define
\[
 C_0=\sum_\theta q_\theta/Z_\theta,\qquad
 r^*_\theta=\frac{q_\theta}{C_0Z_\theta}.
\]
Then
\[
 D(q\|r^*)=F_\beta(\mu^*)+\ln C_0,
 \qquad \max_L\sum_\theta r^*_\theta E_{\theta L}=1/C_0.
\]
Their sum is $F_\beta(\mu^*)$, proving equality in~\eqref{eq:bridge}. Positivity of $E$ makes $Z$ and $r^*$ strictly positive. No minimax interchange between nonconvex feasible sets is used in this step.

Substitute~\eqref{eq:bridge} into~\eqref{eq:gibbs}. The two suprema may be exchanged:
\begin{align}
 R^{\mathrm{nat}}(s)
 &=\sup_{r>0}\left[-D(q\|r)+
   \sup_{\beta\ge0}\{\beta s-\ln(1+(e^\beta-1)V_n(r))\}\right]\\
 &=\sup_{r>0}\{\mathrm{kl}_{\mathrm{nat},+}(s\|V_n(r))-D(q\|r)\}.
\end{align}
For $0<c<1$, the inner supremum is attained at $\beta=0$ if $s\le c$, and otherwise at
$\beta=\ln[s(1-c)/(c(1-s))]$, with its limiting interpretation at $s=1$. Its value is the one-sided binary KL. This proves~\eqref{eq:dual} after division by $\ln2$. Values $V_n(r)=1$ yield zero; impossible-support endpoints follow by limits.

\subsection{From the dual to the protocol theorem}
For any protocol, Lemma~\ref{lem:selective} holds for every $r$. Taking the supremum over $r$ gives
$\aux\ge\R(q,P(\acc=1,N\le n))$. Under a hard cap, this is at least $\R(q,s)$ whenever accuracy is at least $s$, since $\R$ is nondecreasing.

Conversely, take a minimizing $W(L\mid\Theta)$ in~\eqref{eq:RD}. The auxiliary observation is the list $L$ itself, at the initial history. Its information charge is $I(\Theta;L)$. The target-only list strategy of Appendix~\ref{app:coverage} succeeds exactly when the list covers $\Theta$, uses at most $n$ queries, and has success at least $s$. This constructs a finite upfront protocol attaining the converse for every finite prior and acceptance geometry. Initial randomization to choose list order carries no target information.

\section{Zero error, list allocation, and uniform sharpness}\label{app:zero}
This appendix proves \cref{cor:zero} and~\eqref{eq:uniform}, both stated in \cref{sec:zero}.
\subsection{Coverage-polytope entropy}
At zero error, any feasible list channel satisfies $W(L\mid\theta)=0$ whenever $c_{\theta L}=0$. For a proposed list distribution $\mu$, let $a_\theta=\sum_L\mu_Lc_{\theta L}$. On the support of a zero-error $W$,
\[
 D(W(\cdot\mid\theta)\|\mu)
 =D\!\left(W(\cdot\mid\theta)\middle\|\frac{\mu_Lc_{\theta L}}{a_\theta}\right)
   +\log_2(1/a_\theta).
\]
Use~\eqref{eq:decompose} at $\mu=qW$ to obtain a lower bound by the right-hand side of~\eqref{eq:zero}. Every feasible $a$ is a convex combination of coverage vectors.

For equality, minimize $-\sum q_\theta\ln a_\theta$ over $a=C\mu$. A finite optimizer exists with every $a_\theta>0$, since a uniform mixture of all lists covers every world with positive probability. For an optimal $\mu$, the first-order conditions for maximizing $\sum q_\theta\ln a_\theta$ give
\[
 g_L=\sum_\theta q_\theta c_{\theta L}/a_\theta\le1,
 \qquad g_L=1\quad\text{if }\mu_L>0.
\]
Set $W(L\mid\theta)=\mu_Lc_{\theta L}/a_\theta$. Its rows sum to one; its marginal is $\mu_Lg_L=\mu_L$; and its information is exactly $\sum q_\theta\log_2(1/a_\theta)$. This proves~\eqref{eq:zero} and provides the attaining channel. The static convex-corner formula is classical; its interpretation as the exact causal price follows from Theorem~\ref{thm:exact}.

\subsection{Unique answers}
For lists of size $\ell<M$, the coverage polytope of maximal-size lists is
$\{a\in[0,1]^M:\sum_i a_i=\ell\}$. To see this, every indicator list lies in that polytope; conversely any nonintegral point has at least two fractional coordinates and can be perturbed in opposite directions while preserving their sum, so cannot be an extreme point. The extreme points are precisely the indicator lists. Smaller lists are dominated by extensions and cannot improve the decreasing objective.

The optimizer has $a_i>0$ because $q_i>0$. The KKT equations for minimizing $-\sum_iq_i\ln a_i$ with the sum and upper-bound constraints give $a_i=q_i/\tau$ at unsaturated coordinates and $a_i=1$ at saturated ones. Hence~\eqref{eq:clip}. The function $\sum_i\min(1,q_i/\tau)$ decreases continuously from $M$ to zero and identifies the needed $\tau$ when $\ell<M$.

For $q=(0.8,0.1,0.1)$, $\ell=2$, take the two lists $\{1,2\}$ and $\{1,3\}$ with equal marginal probability. Conditional on $A=1$ choose either equally; conditional on $A=2$ or $3$, choose its corresponding list. The cost is $0.8\cdot0+0.2\cdot1=0.2$ bits. An alternative interleaved protocol queries answer one first, and only after failure receives the bit distinguishing answers two and three. It also costs $0.2$ causal bits. The upfront message still has one raw bit of entropy: mutual information and message length are distinct resources.

\subsection{Uniform finite-error frontier}
For $q$ uniform, $\ell=n+1<M$, and $s\ge\ell/M$, choose a size-$\ell$ subset $L$ uniformly, and define the posterior of $A$ to be $s/\ell$ on $L$ and $(1-s)/(M-\ell)$ off $L$. Symmetry makes the marginal of $A$ uniform. The equivalent conditional channel chooses $L$ given the already sampled $A$. Then
\[
 I(A;L)=s\log_2\frac{s}{\ell/M}+(1-s)\log_2\frac{1-s}{1-\ell/M}
       =\mathrm{kl}(s\|\ell/M),
\]
and list coverage is $s$. Lemma~\ref{lem:selective} with $r=q$ proves the matching converse. For $s\le\ell/M$, no source is necessary. This reproduces the static list-Fano extremizer \citep{sakai2020} and establishes its optimality over the larger interleaved protocol class under the causal information budget.

\section{Zero-error price and the remaining exponential cost}\label{sec:zero}
Let $\mathcal A_n=\operatorname{conv}\{c_{\cdot L}:L\in\Ls\}$ be the coverage polytope.
\begin{corollary}[Exact zero-error transfer price]\label{cor:zero}
The minimum auxiliary information for zero error and $n$ target queries is
\begin{equation}\label{eq:zero}
 \R(q,1)=\min_{a\in\mathcal A_n}\sum_\theta q_\theta\log_2(1/a_\theta).
\end{equation}
For unique answers, let $\ell=\min(n+1,M)$. If $\ell<M$, choose $\tau>0$ so that
\begin{equation}\label{eq:clip}
 a_i^*=\min\{1,q_i/\tau\},\qquad \sum_i a_i^*=\ell.
 \quad\text{Then}\quad
 \R(q,1)=\sum_{q_i<\tau}q_i\log_2(\tau/q_i).
\end{equation}
If $\ell=M$, the price is zero.
\end{corollary}
The static expression~\eqref{eq:zero} is a standard convex-corner/hypergraph-entropy representation \citep{csiszar1990}; the clipped formula is its elementary specialization. Theorem~\ref{thm:exact} makes it the \emph{exact interleaved verification price}. A direct proof and channel construction appear in Appendix~\ref{app:zero}.

For example, take $q=(0.8,0.1,0.1)$ and one query. Fixing $r=q$ in~\eqref{eq:selective} gives only $-\log_2(0.9)=0.1520$ bits for zero error. The exact price is $0.2$ bits: the optimum has inclusion probabilities $(1,1/2,1/2)$, and $r=(2/3,1/6,1/6)$ attains~\eqref{eq:dual}. Thus optimizing the comparison prior is substantively stronger than relabeling the prior-mode bound (Figure~\ref{fig:nonuniform}).

For unique answers ($\mathcal Y=[M]$, $f_\theta(y)=\ind\{\theta=y\}$), uniform $q$ and $n<M-1$, the complete frontier simplifies to the classical list-Fano value
\begin{equation}\label{eq:uniform}
 \R(q,s)=\mathrm{kl}_+\left(s\middle\|\frac{n+1}{M}\right),\qquad
 \R(q,1)=\log_2\frac{M}{n+1}.
\end{equation}
An upfront size-$(n+1)$ list with posterior uniform inside and outside the list attains every point \citep{sakai2020}. For unique answers and uniform $q$, zero error with $B$ optimally placed auxiliary bits is therefore possible exactly when $n+1\ge M2^{-B}$: each such bit halves the verifications needed. This does not count the calls needed to \emph{obtain} those bits, which \cref{thm:calls} prices. With overlapping acceptance sets, $r=q$ in~\eqref{eq:dual} gives only the lower bound $\mathrm{kl}_+(s\|V_n(q))$, and it can be strict even for uniform $q$. With $M=4$, $n=0$ and two responses accepted on $\{1,2,3\}$ and $\{4\}$, zero error requires learning $\ind\{\Theta=4\}$, so $\R(q,1)=h_2(1/4)=0.811$ bits, while $\mathrm{kl}(1\|3/4)=0.415$ bits. The bound is attained when the automorphism group of the acceptance structure acts transitively on worlds: averaging a maximum-coverage list over its orbit and taking the posterior uniform inside and outside its coverage set gives a channel with uniform marginal and information $\mathrm{kl}(s\|V_n(q))$.

\paragraph{Why a random-budget shortcut is invalid.}
Let $A$ be uniform on three answers, draw an independent uniform permutation, and stop querying that permutation at the first success. The count $N$ is uniform conditional on every $A$, hence $N\perp A$, yet accuracy is one and $\E V_N(q)=8/9$. Marginal independence is not an exogenous budget. The stopped-event version of Theorem~\ref{thm:exact} avoids this conditioning error.

\section{Information-frontier illustrations}\label{app:illustrations}
\begin{figure}[!h]
\centering\includegraphics[width=.72\linewidth]{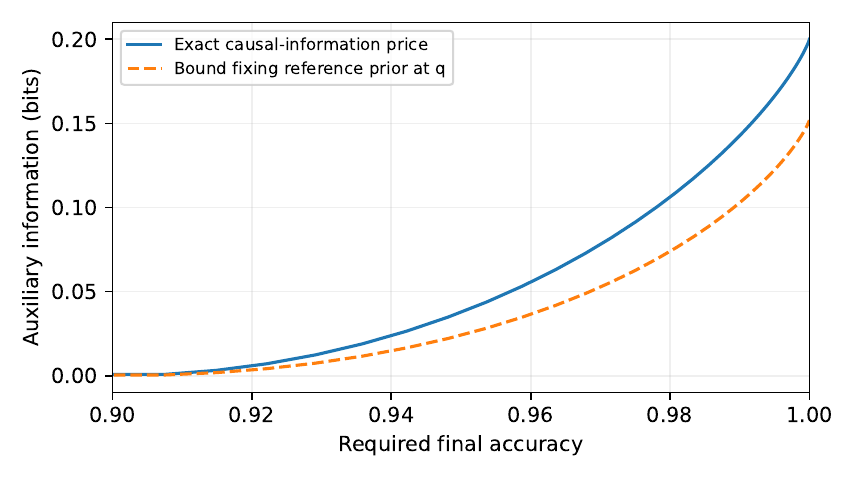}
\caption{\textbf{The optimized bound is exact beyond uniform priors.} Unique answers, $q=(0.8,0.1,0.1)$, one target query. Curves are certified finite-alphabet optimization. At zero error, the exact price is $0.2$ bits, versus $0.1520$ for the unoptimized reference.}\label{fig:nonuniform}
\end{figure}

\begin{figure}[!h]
\centering\includegraphics[width=.72\linewidth]{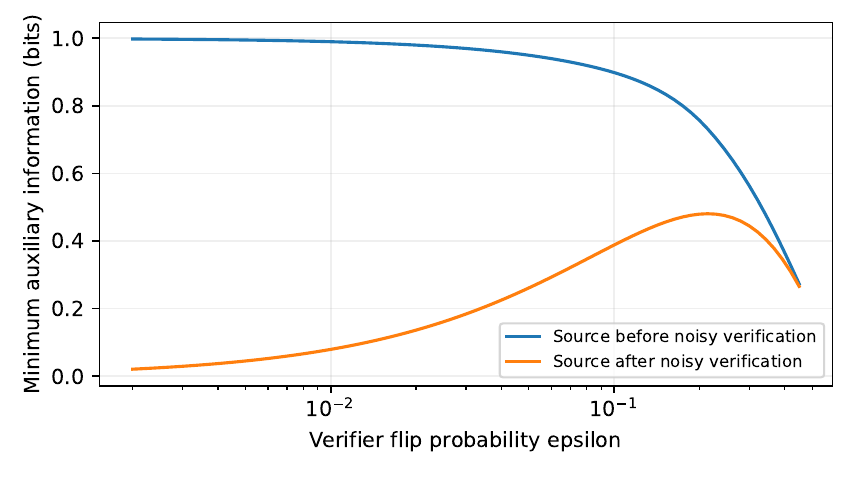}
\caption{\textbf{The upfront reduction fails for noisy feedback.} Uniform binary target, one noisy query, error tolerance $\delta=\epsilon^2$. Both curves are exact optimal information costs. At $\epsilon=0.01$, the costs are $\NoiseUp$ bits upfront and $\NoiseAfter$ bits with interleaving.}\label{fig:noise}
\end{figure}
\section{Source-first suffices for fixed sources}\label{app:fixed}
This section concerns hard call counts rather than the causal information budget. Let the finite source interaction mechanism return an outcome with kernel $K_j(z\mid\theta,h_{\rm src})$, where $h_{\rm src}$ records source selections and replies. It may restrict availability based on that source history (for example, one-use sources). It cannot change its response or availability merely because a target query was made. A latent variable shared by several sources can be included in $\Theta$. Each source transition can be generated using an independent uniform random variable and its specified kernel. The learner's random seed is also independent of the world. Condition on this complete collection of random variables when comparing paths; integrate afterward.

First modify any candidate protocol to enforce its hard caps at \emph{all} histories, not just histories of positive true probability: whenever its next action would exceed a cap, stop and choose an arbitrary fallback. This changes nothing almost surely under the true law. There is no advantage to continuing after a positive target reply, but the argument below does not require the original learner to stop there.

Construct an upfront policy by running a virtual copy of the modified learner. At a source request, request that same source from the actual source mechanism and pass the result to the virtual learner. At a virtual target request, record the candidate response and pass a zero to the virtual learner without making a real target query. At a virtual stop, record its output as the fallback. The virtual source requests obey the source mechanism's availability restrictions because its source-only history agrees with the one actually realized. The virtual execution respects the globally enforced caps. Its transcript $Z$ determines a list of $t$ virtual target candidates and a fallback, where $t\le n$. After source acquisition, query these $t$ candidates in order, stopping at the first positive reply and otherwise outputting the fallback. Repeated candidates may be omitted.

Couple an original and a transformed execution in the same world using the same independent seeds at their matching source histories. Before the original's first positive target reply, all actual target replies equal the virtual zeros, so actions, source choices, and source replies coincide. If a first positive occurs, its candidate belongs to the virtual list, including when the continuation beyond that point is off the original support. Therefore the transformed protocol succeeds. If the original sees no positive reply but has a correct final fallback, its entire history agrees with the virtual execution and that fallback is included. Thus every original successful path is transformed into a successful path under this coupling, up to null events. The source cap and target cap hold pathwise for the transformed policy.

For a hard total cap $Q+N\le C$, enforce that single cap on the virtual learner. If its virtual path used $m$ source requests and $t$ target requests, then $m+t\le C$. The transformed protocol performs the same $m$ sources and at most $t$ actual target queries. Its total is at most $C$. The same coupling proves success domination. This statement would be false for a constraint only on expected call counts: the transformation can acquire sources on worlds where the original stops after an early target success.

Condition on the upfront transcript, including the acquisition policy's independent seed. With at most $n$ target queries, Proposition~\ref{prop:coverage} gives the optimum $V_n(q(\cdot\mid Z))$. Any source-only policy can be followed by an optimal posterior list. This gives both inequalities in~\eqref{eq:fixedvalue} and proves Theorem~\ref{thm:fixed}. The source policy may select later sources based on earlier source outcomes; a source-first protocol need not use a fixed subset.

\paragraph{A source-adaptivity boundary.}
Let the four equally likely worlds be $0,1,2,3$, with the world itself the target answer. Three binary source masks are $\{2,3\}$, $\{1\}$, and $\{3\}$. Query the first mask and, according to its answer, query the second or third. Two source calls identify the world. Any fixed two-source subset has only three nonempty outcome cells, so its best final accuracy without target queries is $3/4$. This example is not a linear uniform source bank in the sense of Theorem~\ref{thm:rank}, and shows why that theorem needs its assumptions.

\paragraph{An expected-cost and information-cost boundary.}
Let the target have prior $(1-2p,p,p)$ with $0<p<1/2$, and let the only source be $Z=\ind\{A=3\}$. Query target answer one first. On success stop; on failure query the source and distinguish answers two and three. Zero error uses one target query, at most one source query, and only $2p$ source queries in expectation. An upfront protocol achieving zero error with one target query must always acquire the source: otherwise its single query and fallback leave one positive-prior answer uncovered. Its expected source count is one. The original causal source information is $2p$ bits, whereas the upfront source information is $h_2(p)$. Thus Theorem~\ref{thm:fixed} preserves hard caps and success, not either of these other currencies. This also distinguishes it from the freely designed information-preserving optimum in Theorem~\ref{thm:exact}.

\section{Exact transfer rank and total-query optimization}\label{app:rank}
All ranks and affine spaces in this section are over $\mathbb F_2$. Let $T$ have rank $k$, $W$ be uniform on $\mathbb F_2^d$, and $B$ stack a chosen set of source rows. Each attainable source value $z=BW$ gives a fiber $w_0+\ker B$, and the posterior of $W$ is uniform on that fiber. Its target image is $Tw_0+T(\ker B)$. Every element of this image has the same number of preimages. Moreover,
\begin{align*}
 \dim T(\ker B)
 &=\dim\ker B-\dim(\ker B\cap\ker T)\\
 &=(d-\operatorname{rank}B)-\left(d-\operatorname{rank}\begin{bmatrix}T\\B\end{bmatrix}\right)
 =k-\rho(S).
\end{align*}
This proves the claimed uniform target support of size $2^{k-\rho(S)}$, and also the familiar identity $I(A;BW)=\rho(S)$ for a fixed source set. Counting uniform linear fibers is elementary; we do not claim the entropy--rank correspondence as a new result.

For an adaptive source-only policy, fix its independent seed. On any leaf let $S$ be the queried rows and $z$ their replies in order. The event of reaching that leaf is exactly the affine system $B_SW=z$. In one direction the event imposes those equations. Conversely, any world satisfying the equations supplies the same sequential replies and therefore follows the same chosen sources and stopping decisions. There is no additional hidden restriction caused by the policy's adaptivity. Repeated rows can be deleted without changing the fiber.

For rectangular caps, Theorem~\ref{thm:fixed} reduces any interleaved policy to such leaves followed by at most $n$ target queries. At a leaf the local value is
\[
 \min\{1,(n+1)2^{-k+\rho(S)}\}
 \le \min\{1,(n+1)2^{-k+\rho_b}\}.
\]
Average over leaves and seeds to obtain the upper bound in~\eqref{eq:rankrect}. Choose a maximizing fixed subset $S$, solve its linear system, and enumerate any $n+1$ distinct points of the target affine image (or all points if fewer remain). The first $n$ serve as verification candidates and the last as fallback. This achieves the upper bound. Gaussian elimination implements decoding, although finding the maximizing subset can be combinatorial.

For a total cap, apply the total-cap transformation in Appendix~\ref{app:fixed}. If a leaf spent $m$ source calls, at most $C-m$ target queries remain, giving an upper bound
\[
 \min\{1,(C-m+1)2^{-k+\rho(S)}\}
 \le \max_{0\le b\le\min(C,J)}\min\{1,(C-b+1)2^{-k+\rho_b}\}.
\]
If more than $J$ source requests occurred, delete repeats first: a repeated deterministic source is already known and cannot help. For attainment choose a maximizing budget $b$ and a fixed subset with at most $b$ rows attaining $\rho_b$, then use at most $C-b$ target queries. Unused source capacity need not be spent. This proves~\eqref{eq:ranktotal}, including policies whose actual source/target allocation depends on the observed values.
\section{Calls and information}\label{app:calls}
An \emph{auxiliary call} is one auxiliary observation, and $Q$ counts them. Sources are \emph{binary} when every auxiliary alphabet has two letters, as in the linear banks of Section~\ref{sec:planner}. This appendix proves Theorem~\ref{thm:calls}, Propositions~\ref{prop:general} and~\ref{prop:interleave}, Corollary~\ref{cor:waste} and Proposition~\ref{prop:noisy}; \texttt{theory\_checks/check\_bridge.py} and \texttt{theory\_checks/exact\_calls.py} check the finite claims numerically.

\subsection{Proof of Theorem~\ref{thm:calls}}
\emph{(i) Converse.} Let verification be exact; \cref{prop:noisy} extends the bound to verifiers that garble exact replies. At an auxiliary node with history $\mathfrak h$ and request $c$, the term charged in~\eqref{eq:budget} is the conditional mutual information $I(\Theta;Z_t\mid H_{t-1}=\mathfrak h,C_t=c)\le H(Z_t\mid H_{t-1}=\mathfrak h,C_t=c)\le\log_2|\mathcal Z_{t}|$. Weighting by the probability of reaching the node and summing gives $\aux\le\E\sum_{t:\,C_t\text{ auxiliary}}\log_2|\mathcal Z_t|$, which is $\E Q$ for binary sources. Theorem~\ref{thm:exact} then gives $\E Q\ge\aux\ge\R(q,s)$ under $N\le n$ almost surely and $\Pp(\acc=1)\ge s$, and its stopped form gives $\E Q\ge\R(q,\Pp(\acc=1,N\le n))$ for an arbitrary stopping rule. For a noisy verifier, \cref{prop:noisy}(i) bounds $\aux$ below by $\sup_r\{\mathrm{kl}_+(s\|V^G_n(r))-\KL(q\|r)\}$, and $V^G_n\le V_n$ because a learner facing an exact verifier can simulate the noisy one by garbling its replies, so this supremum is at least $\R(q,s)$. Nothing else in the argument uses exact verification, fixed kernels, or source-first order.

\emph{(ii) Zero error with unique answers.} If $\ell=n+1\ge M$, no call is needed. Otherwise, by \cref{cor:zero}, $\R(q,1)=\sum_\theta q_\theta\log_2(1/a_\theta)$ with $a_\theta=\min\{1,q_\theta/\tau\}$ and $\sum_\theta a_\theta=\ell$. Set $L_\theta=\lceil\log_2(1/a_\theta)\rceil+1\ge1$, so that $2^{-L_\theta}\le a_\theta/2$ and
\[
 \textstyle\sum_\theta2^{-L_\theta}\le\ell/2 .
\]
Order the worlds by increasing $L_\theta$ and cut the order into consecutive blocks $B_1,B_2,\dots$ of $\ell$ worlds each, except possibly the last. Give block $B_i$ the depth $D_i=\min_{\theta\in B_i}L_\theta$. For $i\ge2$, every world of the full block $B_{i-1}$ has $L_\theta\le D_i$, so $2^{-D_i}\le\ell^{-1}\sum_{\theta\in B_{i-1}}2^{-L_\theta}$. Hence
\[
 \textstyle\sum_i2^{-D_i}\le2^{-D_1}+\ell^{-1}\sum_\theta2^{-L_\theta}\le\tfrac12+\tfrac12=1,
\]
and Kraft's inequality gives a binary prefix-free code with codeword length $D_i$ for block $B_i$. The $j$th source returns the $j$th bit of the codeword of $\theta$'s block, a deterministic function of $\theta$ and $j$. The learner reads bits until the codeword is complete, verifies $\ell-1$ worlds of the block, and outputs the last if none is accepted, so $N\le n$ and the error is zero. The number of calls is at most $\max_iD_i$, and
\[
 \textstyle\E Q=\sum_\theta q_\theta D_{i(\theta)}\le\sum_\theta q_\theta L_\theta<\sum_\theta q_\theta\big(\log_2(1/a_\theta)+2\big)=\R(q,1)+2 .
\]
The construction's constant is approached (uniform $q$ with $M$ just above $\ell2^j$), but the optimum's may be smaller. Exact dynamic programming over candidate sets on 200 random priors with $M\le8$ (\texttt{theory\_checks/exact\_calls.py}) never finds a gap above one, and a gap of one is approached with $n=0$ and $q$ near a point mass, where one call is unavoidable while $\R(q,1)\to0$.

\emph{(ii) Every success level with unique answers.} If $s\le V_n(q)$, verifying the $n$ most probable answers and outputting the next succeeds with probability $V_n(q)\ge s$ without calls. Let $s>V_n(q)$. The function $\R(q,\cdot)$ is convex and nondecreasing, and it vanishes exactly on $[0,V_n(q)]$, because zero information means a list independent of $\Theta$. It is therefore strictly increasing on $[V_n(q),1]$, so an optimal channel $W$ for $\R(q,s)$ covers with probability exactly $s$. Let $C=\ind\{\Theta\in L\}$, $c_\theta=\Pp(C=1\given\Theta=\theta)$ and $q_{1,\theta}=q_\theta c_\theta/s$, the law of $\Theta$ given $C=1$, restricted to its support. Since $C$ is a function of $(\Theta,L)$, $H(\Theta\given L)=H(C\given L)+H(\Theta\given L,C)\le h_2(s)+H(\Theta\given L,C)$, so
\[
 I(\Theta;L)\ge H(\Theta\given C)-h_2(s)-H(\Theta\given L,C)=I(\Theta;L\given C)-h_2(s)\ge s\,I(\Theta;L\given C=1)-h_2(s).
\]
Given $C=1$, $\Theta\sim q_1$ and $\Theta\in L$ almost surely, so $I(\Theta;L\given C=1)\ge\mathcal R_n(q_1,1)$. This lower bound holds for every channel with coverage $s$, not only optimal ones. For achievability, the first source reports $C'\sim\mathrm{Bernoulli}(c_\Theta)$, a stochastic kernel of the world alone. Given $C'=1$ the posterior is exactly $q_1$, and the learner runs the zero-error protocol above for $q_1$, which needs no calls if $q_1$ has at most $n+1$ atoms. Given $C'=0$ it outputs any response. Success is at least $s$, $N\le n$, and
\[
 \E Q<1+s\big(\mathcal R_n(q_1,1)+2\big)\le1+\R(q,s)+h_2(s)+2s\le\R(q,s)+1+\log_2 5,
\]
since $h_2(s)+2s$ is maximized at $s=4/5$ with value $\log_25$. No randomness is shared: the only randomization is inside the first source. On 750 near-optimal Blahut--Arimoto channels the construction is within $3.1$ calls of $I(W)$.

\subsection{Proof of Proposition~\ref{prop:general}}
Let $G$ be a graph on $[M]$, let the responses be the independent sets of $G$ with $f_\theta(I)=\ind\{\theta\in I\}$, and let $n=0$. A zero-error list is then a single independent set containing $\Theta$, so $\mathcal R_0(q,1)$ is Körner's graph entropy $H_{\mathrm K}(G,q)$ \citep{csiszar1990,alon1996source}. Represent every stochastic kernel and the learner's seed by independent uniform variables $U$. For each fixed $U$ the protocol is a binary decision tree whose leaves output independent sets, and the worlds reaching a leaf lie in its output. The leaves therefore colour $G$, and the expected depth is at least the entropy of that colouring, hence at least the chromatic entropy $H_\chi(G,q)$. Averaging over $U$ gives $\E Q\ge H_\chi(G,q)$. \citet[Proposition~2]{cardinal2008minimum} show that $H_\chi-H_{\mathrm K}$ is unbounded even for chordal graphs with uniform $q$. The gap persists for every $n\ge1$. Let the world be $(v,w)$ with $v\sim q$ on the vertices of $G$ and an independent index $w$ uniform on $[n+1]$, let the responses be pairs $(I,j)$ of an independent set and an index, and let $f_{(v,w)}(I,j)=\ind\{v\in I,\ j=w\}$. The list $\{(I,1),\dots,(I,n+1)\}$ with $I$ drawn from an optimal channel for $H_{\mathrm K}(G,q)$ covers every world, so $\R\le H_{\mathrm K}(G,q)$. Conversely, fix $U$ and $w$. The transcript of binary source replies and verification replies is then a prefix-free function of $v$, of expected length $\E[Q+N\mid w,U]$, and it determines the output $\widehat y$. At zero error $\widehat y=(I_v,w)$ with $v\in I_v$, so $v\mapsto I_v$ colours $G$ and $\E[Q+N\mid w,U]\ge H_\chi(G,q)$. Since $N\le n$, averaging gives $\E Q\ge H_\chi(G,q)-n$, and $\E Q-\R\ge H_\chi(G,q)-H_{\mathrm K}(G,q)-n$, which is unbounded for every fixed $n$. For the upper bound, take general acceptance sets and any $n$, and let $L$ be distributed by an optimal channel for $\R(q,s)$, which covers with probability at least $s$. The strong functional representation lemma \citep[Theorem~1]{li2018strong} gives a random variable $V$ independent of $\Theta$, taking at most $M(|\Ls|-1)+2$ values, and a function $g$ with $L=g(\Theta,V)$ and $I(\Theta;V\mid L)\le\log_2(I(\Theta;L)+1)+4$. Since $V$ is independent of $\Theta$ and $L$ is a function of $(\Theta,V)$,
\[
 H(L\mid V)=I(\Theta;L\mid V)=I(\Theta;L)+I(\Theta;V\mid L)\le I(\Theta;L)+\log_2\big(I(\Theta;L)+1\big)+4 .
\]
For each value $v$, the map $g_v=g(\cdot,v)$ is a deterministic list map with entropy $H_v=H(g_v(\Theta))$ and coverage $\sigma_v=\Pp(\Theta\text{ is covered by }g_v(\Theta))$, where $\Theta\sim q$. Averaging over $V$ gives $\sum_vP_V(v)H_v=H(L\mid V)$ and $\sum_vP_V(v)\sigma_v=\E\,c_{\Theta L}\ge s$. Minimizing $\sum_vp_vH_v$ over distributions $p$ subject to $\sum_vp_v\sigma_v\ge s$ is a linear program with one constraint besides normalization, so an optimal $p$ is supported on at most two values $v_1,v_2$, and its objective is at most $H(L\mid V)$. At $s=1$ every $g_v$ with $P_V(v)>0$ covers every world, so a single value suffices. The learner flips one coin with the optimal bias to choose $v\in\{v_1,v_2\}$ and fixes a binary Huffman code for the law of $g_v(\Theta)$, whose expected length is at most $H_v+1$ and whose image has at most $M$ lists, so no codeword is longer than $M-1$ bits. A request is a pair $(v,j)$, and the source returns the $j$th bit of the codeword of $g_v(\theta)$, a deterministic function of the world and the request that never depends on target history. Carrying $v$ in the request is the only randomness shared with the sources. The learner requests bits until the codeword is complete, decodes the list, verifies all but its last entry in turn, and outputs the last entry if none is accepted. Then $N\le n$, $\Pp(\acc=1)\ge s$, at most $M-1$ calls are made on every path, and
\[
 \E Q\le\textstyle\sum_vp_v(H_v+1)\le H(L\mid V)+1\le\R(q,s)+\log_2(\R(q,s)+1)+5 .
\]
This is one-shot channel simulation \citep{harsha2010communication,li2018strong}, derandomized because only the coverage constraint couples the values of $V$. At $n=0$ and $s=1$ in the graph setting, each $g_v$ assigns every world an independent set containing it, so its classes colour $G$, and for the best value $v^\star$
\[
 H_\chi(G,q)\le H\big(g_{v^\star}(\Theta)\big)\le H(L\mid V)\le H_{\mathrm K}(G,q)+\log_2\big(H_{\mathrm K}(G,q)+1\big)+4
\]
for every graph and prior. The gap in the lower bound is therefore at most logarithmic in $H_{\mathrm K}$. On the chordal graphs of \citet[Lemmas~5 and~6]{cardinal2008minimum} it is at least $\log_2k-o(1)$, where $k$ is the depth of the construction and $o(1)$ vanishes as the clique sizes grow with $k$ fixed, while $H_{\mathrm K}$ grows as well. The logarithmic order is attained. In the tightness proof of \citet{cardinal2010efficient}, the comparability graph of the interval order $I_k$ has $M=k2^{k-1}$ vertices, and under a uniform prior its chromatic entropy exceeds its graph entropy by at least $\log_2k-\log_2e-1$. This is $\Omega(\log\log M)$, and hence $\Omega(\log H_{\mathrm K})$ because $H_{\mathrm K}\le\log_2M$ \citep[see also][]{cardinal2012minimum}.

\subsection{Proof of Proposition~\ref{prop:interleave}}
Take unique answers, $q=(1-2p,p,p)$ with $0<p\le1/4$, $n=1$ and $s=1$. By \cref{cor:zero} (list water-filling with inclusion probabilities $(1,\tfrac12,\tfrac12)$), $\R(q,1)=2p$. The interleaved protocol verifies answer~1 and stops if it is accepted. Otherwise it makes one binary call that distinguishes answers~2 and~3 and outputs the result without a second verification. It has zero error, $N\le1$ and $\E Q=\Pp(\Theta\ne1)=2p$, so it meets the converse of Theorem~\ref{thm:calls}(i) with equality. In a source-first protocol, the decision to make a first call depends only on the private seed, so the event that no call precedes the first verification is independent of $\Theta$. On that event the learner's candidate list is a function of the seed, has at most two elements, and misses a world of positive prior. Zero error therefore forces a call on almost every seed, so $\E Q\ge1$, and the ratio $1/(2p)$ is unbounded. The saving over the best \emph{designed} source-first protocol is bounded by Theorem~\ref{thm:calls}. Against a \emph{given} bank it is not. At zero error, take $q=(1-\epsilon,\epsilon/K,\dots,\epsilon/K)$ on $K+1$ worlds with $K\ge4$ a power of two, $n=1$, and sources that return the $\log_2K$ bits of an index in which world~1 shares the index of world~2. Verifying world~1 first and then reading the index costs $\epsilon\log_2K$ expected calls. A source-first protocol must separate world~1 from every world except world~2 before its single verification, which takes all $\log_2K$ bits whenever $\Theta=1$, so it costs at least $(1-\epsilon)\log_2K$.

\subsection{Proof of Corollary~\ref{cor:waste}}
\emph{Hard caps.} By Theorem~\ref{thm:rank}, $\mathsf A_{b,n}=1$ exactly when $(n+1)2^{-k+\rho_b}\ge1$, that is, when $\rho_b\ge k-\log_2(n+1)$; since $\rho_b$ is an integer this is $\rho_b\ge\kappa$. By Theorem~\ref{thm:qsp}, $\rho_b\ge\kappa$ holds exactly when some subspace $U$ generated by the $\bar v_j$ has $t(U)\ge\kappa$ and $b-\dim U\ge\kappa$; if no $U$ has $t(U)\ge\kappa$ then $\rho_J<\kappa$ and no cap suffices. The least such $b$ is $\kappa+u(\kappa)$. For the uniform prior on $2^k$ unique answers, \eqref{eq:uniform} gives $\R(q,1)=k-\log_2(n+1)$, so $b_n\ge\lceil\R(q,1)\rceil$ with equality exactly when $u(\kappa)=0$. Because $\rho_{b+1}\le\rho_b+1$, the least budget with $\rho_b\ge\kappa$ has $\rho_{b_n}=\kappa$, so the excess $b_n-\kappa=u(\kappa)$ equals $b_n-\rho_{b_n}$, the calls of an optimal subset that land on nuisance directions.

\emph{Source-first expected calls.} In a source-first protocol, the posterior of the target on every source-only leaf is uniform on $2^{k-\rho(S)}$ vectors, where $S$ is the set of sources called on that path (\cref{app:rank}). Zero error with $n$ verifications requires $2^{k-\rho(S)}\le n+1$, so $\rho(S)\ge\kappa$ and hence $|S|\ge b_n$ on every path; a fixed optimal subset attains $b_n$ on every path.

\emph{Interleaving can use fewer expected calls.} Take $k=2$, $h=0$, two sources that return $a_0$ and $a_1$, and $n=2$, so $\kappa=1$ and $b_n=1$. Verify $00$, then $01$; if both are rejected, the remaining candidates $10$ and $11$ differ only in $a_1$, so call the source for $a_1$ and output the one it indicates. The error is zero and $\E Q=\Pp(\text{both rejected})=1/2<b_n$, while $\R(q,1)=2-\log_23=0.415\le\E Q$, as \cref{thm:calls}(i) requires.

On 300 random banks with $k\le5$, $h\le4$ and $J\le12$, the identity $b_n=\kappa+u(\kappa)$ agrees with exhaustive search over all source subsets for all 940 pairs (bank, $\kappa$) with $1\le\kappa\le k$ (and trivially for $\kappa=0$).

\subsection{Proof of Proposition~\ref{prop:noisy}}
\emph{(i)} Repeat the construction of Appendix~\ref{app:selective} with the target kernel $G$ in place of exact replies: $\widetilde P$ has prior $r$, the same learner, the same target kernel, and auxiliary kernels $\overline K_t$. Given $\theta$, a target transition has the same probability $G(\cdot\mid f_\theta(y))$ under both laws, so it cancels in~\eqref{eq:ratioapp}, and~\eqref{eq:pathKL} holds unchanged. Under $\widetilde P$ every auxiliary observation can be simulated from the history and independent randomness, so the protocol is a target-only protocol under prior $r$ and $\widetilde P(\acc=1,N\le n)\le V^G_n(r)$. Data processing to the event $\{\acc=1,N\le n\}$ completes the proof as in Lemma~\ref{lem:selective}.

\emph{(i), second form.} Include the private seed $\xi$ in the initial history; $I(\Theta;\xi)=0$. By the chain rule, $I(\Theta;\xi,H_T)$ is the sum over steps of the conditional information of each observation given the past. Actions are functions of the past and add nothing. Auxiliary observations contribute exactly $\aux$. A verification reply depends on $\Theta$ only through the binary input $f_\Theta(y_t)$ to $G$, so it contributes at most $\operatorname{cap}(G)$, and there are at most $n$ of them. The output $\widehat y$ is a function of $(\xi,H_T)$ and is acceptable with probability at least $s$, so $I(\Theta;\widehat y)\ge\mathcal R_0(q,s)$ by the definition of the rate--distortion function with single-response lists. Hence $\mathcal R_0(q,s)\le\aux+n\,\operatorname{cap}(G)$. On the instance of \cref{thm:noise}, $\mathcal R_0(q,1-\delta)=1-h_2(\delta)$ for the uniform binary answer and $\operatorname{cap}(G)=1-h_2(\epsilon)$ for the binary symmetric verifier, so with $n=1$ the bound is $h_2(\epsilon)-h_2(\delta)$, the exact interleaved optimum \eqref{eq:noisyafter}.

\emph{(ii)} The information statement is Theorem~\ref{thm:noise}. For counted calls, let $A$ be uniform on $\{0,1\}$ and let one noisy equality query return $\ind\{A=y\}\oplus E$ with $E\sim\mathrm{Bernoulli}(\epsilon)$. Two fixed sources are available; source $c_r$ reports $\ind\{A\ne r\}$ through a Z-channel that misses a true $1$ with probability $\eta$ and never reports a false $1$. Both kernels depend on $A$ only.

\emph{Interleaved.} Query $y=0$ and let $g$ be the answer the reply favours, so $\Pp(g\ne A)=\epsilon$. Call $c_g$ and output $1-g$ if it reports $1$, else $g$. An error needs $g\ne A$ and a miss, so it has probability $\epsilon\eta$.

\emph{Source first.} A source-first protocol calls $c_0$, calls $c_1$, or calls nothing, possibly at random, and a random choice has the average error of the three. Calling nothing leaves error $\epsilon$. By symmetry consider $c_0$. A report of $1$ reveals $A=1$. After a report of $0$, which has probability $(1+\eta)/2$, the posterior of the less likely answer is $\eta/(1+\eta)$, and the Bayes error after one noisy reply from posterior minority mass $v$ is $\min\{v,\epsilon\}$ (Appendix~\ref{app:noise}). The error is therefore $\tfrac{1+\eta}2\min\{\tfrac\eta{1+\eta},\epsilon\}=\min\{\eta/2,\ \epsilon(1+\eta)/2\}$, which is also below the error $\epsilon$ of using no source, so it is the least error of any source-first protocol. Exhaustive search over all policies with at most one call and one query reproduces both values. At $\epsilon=\eta$ the source-first error is $\eta/2$, so the ratio of the two errors is $1/(2\eta)$, unbounded as $\eta\to0$. The only thing the interleaved protocol uses is the choice of source after seeing the target reply, which is exactly what source-first order forbids. The separation is in error at equal caps, not in the number of calls: calling both sources first and then verifying also reaches error $\epsilon\eta$, at the cost of one extra call.
\section{Noisy verification makes information order matter}\label{sec:noise}
Theorem~\ref{thm:exact} needs exact verification. Take $A\sim\mathrm{Bernoulli}(1/2)$ and at most one equality query, whose reply is independently flipped with probability $\epsilon\in(0,1/2)$. The output is still required to equal the true $A$. Let $h_2$ be the binary entropy function. Compare arbitrary source/target interleaving with the restriction that every auxiliary observation precedes the noisy query; both charge conditional information exactly as in~\eqref{eq:budget}.

\begin{theorem}[Exact noisy ordering separation]\label{thm:noise}
For error probability at most $\delta\in[0,\epsilon]$, the minimum interleaved auxiliary information is
\begin{equation}\label{eq:noisyafter}
 B_{\mathrm{inter}}(\epsilon,\delta)=h_2(\epsilon)-h_2(\delta).
\end{equation}
Let $v_\epsilon\in(0,\epsilon)$ be the unique solution of
$h_2(v)+(\epsilon-v)h_2'(v)=1$. If all auxiliary observations must precede the query, the minimum is
\begin{equation}\label{eq:noisybefore}
 B_{\mathrm{upfront}}(\epsilon,\delta)=
 \begin{cases}
 1-h_2(\delta),&0\le\delta\le v_\epsilon,\\
 (\epsilon-\delta)h_2'(v_\epsilon),&v_\epsilon\le\delta\le\epsilon.
 \end{cases}
\end{equation}
In particular, zero error costs $h_2(\epsilon)$ versus one bit. With the strictly positive error tolerance $\delta=\epsilon^2$, the interleaved cost tends to zero and the upfront cost tends to one as $\epsilon\to0$.
\end{theorem}
\begin{proof}[Proof sketch]
A single noisy reply contributes at most $1-h_2(\epsilon)$ bits, whereas final error $\delta$ requires at least $1-h_2(\delta)$ total bits. Query first, then optimally describe the remaining Bernoulli-$\epsilon$ error: the elementary binary rate--distortion test channel attains the difference. For an upfront source $Z$, set $v=\min\{P(A=1\mid Z),P(A=0\mid Z)\}$. Optimal error after the query is $\min(v,\epsilon)$. By symmetry, any distribution over $v\in[0,1/2]$ is realizable with mean prior $1/2$. Maximizing $\E h_2(v)$ subject to $\E\min(v,\epsilon)\le\delta$ gives the upper concave envelope of $h_2(v)$ and the point $(\epsilon,1)$, hence~\eqref{eq:noisybefore}. Appendix~\ref{app:noise} proves the constructions, the tangent condition, and the limiting separation.
\end{proof}
The formulas use classical binary rate distortion and posterior convexification; the point is the exact change in the verification protocol's operational behavior. The after-query source sees the query history. Restricting it not to see that history defines a different problem. These costs are mutual information, not raw message lengths or Wyner--Ziv communication rates.

\section{Noisy ordering theorem: full proof}\label{app:noise}
Here only, target observations are noisy and the exact-verifier reduction does not apply. Query $y\in\{0,1\}$ and observe $\ind\{A=y\}\oplus E$, where $E\sim\mathrm{Bernoulli}(\epsilon)$ is independent of all prior history. Knowing $y$, this is equivalent to a binary symmetric observation $R=A\oplus E$.

\subsection{Optimal arbitrary interleaving}
By the information chain rule, total transcript information is the sum of the causal source budget and the conditional information of target replies. Actions are functions of the history and private seed and contribute no additional information. With at most one noisy query, the latter contribution is at most $1-h_2(\epsilon)$, including protocols that sometimes make no query. This follows from $H(R\mid H)\le1$ and $H(R\mid A,H)=h_2(\epsilon)$ at every reached query node.

If the final guess has error at most $\delta\le\epsilon<1/2$, binary Fano gives $I(A;H_T)\ge1-h_2(\delta)$. Consequently $\aux\ge h_2(\epsilon)-h_2(\delta)$.

For achievability, query first. Since $A$ is uniform, $R$ is uniform and the residual $X=A\oplus R$ is independent of $R$, with law $\mathrm{Bernoulli}(\epsilon)$. Choose a joint law for $(X,\widehat X)$ by drawing
\[
 \widehat X\sim\mathrm{Bernoulli}\left(\frac{\epsilon-\delta}{1-2\delta}\right),
 \quad D\sim\mathrm{Bernoulli}(\delta)\text{ independently},
 \quad X=\widehat X\oplus D.
\]
The marginal of $X$ is $\mathrm{Bernoulli}(\epsilon)$. Use the resulting conditional law of $\widehat X$ given $X$ as the source kernel, with source access to $(A,R)$, and output $\widehat A=R\oplus\widehat X$. The error is $\delta$ and
\[
 I(A;\widehat X\mid R)=I(X;\widehat X)=h_2(\epsilon)-h_2(\delta).
\]
This proves~\eqref{eq:noisyafter}. At $\delta=0$ the source reveals the remaining error exactly; at $\delta=\epsilon$ no source information is needed.

\subsection{Optimal upfront information}
Aggregate all upfront observations into $Z$, conditioning also on the learner's independent seed. For posterior $p=P(A=1\mid Z)$, set $v=\min(p,1-p)$. The Bayes error after one binary symmetric observation is
\begin{align*}
 e(p)&=\min\{p(1-\epsilon),(1-p)\epsilon\}
      +\min\{p\epsilon,(1-p)(1-\epsilon)\}\\
     &=\min\{p,1-p,\epsilon\}=\min\{v,\epsilon\}.
\end{align*}
The source cost is $1-\E h_2(v)$. Every distribution of $v\in[0,1/2]$ can be realized: give posterior $v$ and $1-v$ equal conditional weights, thereby preserving the uniform prior. Thus the optimization is exactly
\[
 1-B_{\mathrm{upfront}}(\epsilon,\delta)
 =\sup\{\E h_2(v):\E\min(v,\epsilon)\le\delta\}.
\]
For $v\ge\epsilon$, the error coordinate is fixed at $\epsilon$ and entropy is maximized by $v=1/2$, attaining the point $(\epsilon,1)$. For $v\le\epsilon$, the achievable upper boundary is $(v,h_2(v))$. Randomization takes their upper concave envelope.

Define $\Gamma(v)=h_2(v)+(\epsilon-v)h_2'(v)$. On $(0,\epsilon)$,
$\Gamma'(v)=(\epsilon-v)h_2''(v)<0$, $\Gamma(v)\to\infty$ as $v\downarrow0$, and $\Gamma(\epsilon)=h_2(\epsilon)<1$. Hence a unique $v_\epsilon$ has $\Gamma(v_\epsilon)=1$. The envelope follows $h_2(v)$ until $v_\epsilon$ and its tangent line from there to $(\epsilon,1)$. Concavity of $h_2$ makes the line a global upper support; matching slopes at the join gives a concave envelope.

For $\delta\le v_\epsilon$, choose posterior pairs $\delta,1-\delta$. For $\delta\ge v_\epsilon$, mix posterior error $v_\epsilon$ with probability $(\epsilon-\delta)/(\epsilon-v_\epsilon)$ and posterior error $1/2$ otherwise. These constructions attain the envelope and prove~\eqref{eq:noisybefore}. The tangent equation can also be written
\[
 -\epsilon\log_2 v_\epsilon-(1-\epsilon)\log_2(1-v_\epsilon)=1.
\]
Our implementation solves it in log-odds coordinates to avoid underflow for small $\epsilon$.

\subsection{Nonzero-error asymptotic separation}
For every $v\le1/2$, $\min(v,\epsilon)\ge2\epsilon v$. Therefore error at most $\delta$ implies $\E v\le\delta/(2\epsilon)$, and Jensen gives
\[
 B_{\mathrm{upfront}}(\epsilon,\delta)
 \ge1-h_2\left(\frac{\delta}{2\epsilon}\right),\qquad 0\le\delta\le\epsilon.
\]
At $\delta=\epsilon^2$, this lower bound tends to one. The cost cannot exceed one, while $h_2(\epsilon)-h_2(\epsilon^2)$ tends to zero. The advantage in information cost is therefore unbounded. This separation is not based solely on the zero-error discontinuity of a full-support noisy channel.

\section{Information quantity is not source usefulness}\label{sec:channels}
Theorem~\ref{thm:exact} concerns an optimized source; another source carrying the same $B$ bits can fall far short of the frontier. For a \emph{fixed} upfront source $K$, the exact local value is
\begin{equation}\label{eq:sourcevalue}
 A_n(K)=\E_Z V_n(q(\cdot\given Z)).
\end{equation}
This follows by conditioning Proposition~\ref{prop:coverage} and is a task-specific Bayes value of information. It leads to different source choices than maximizing $I(\Theta;Z)$.

Consider a uniform unique answer on $M=256$ symbols, with $n=15$. A balanced partition into sixteen groups reveals four bits and leaves exactly sixteen candidates: accuracy one is possible. A source that reveals the exact answer with probability $1/2$ and otherwise erases also carries four bits, but its best accuracy is $\tfrac12+\tfrac12(16/256)=0.53125$. Increasing the reveal probability to $0.55$ raises information to $4.4$ bits but reaches only $0.578125$. More generally, with $M$ uniform answers and list size $\ell$, a balanced partition and reveal--erasure source of equal information $B=\log_2(M/\ell)$ have accuracies $1$ and $\alpha+(1-\alpha)\ell/M$, where $\alpha=B/\log_2M$. Along $M=2^{k^2}$ and $\ell=M/2^k$, the latter tends to zero while the former remains one. Both accuracies are exact optima for their channels, so the gap comes from the shape of the channel.

\paragraph{A learned decoder check.}
We train a \NeuralParams-parameter message-to-answer network from scratch on independent draws from the partition and reveal--erasure channels, then let its highest-scoring sixteen answers define the verification list. The network has a 64-dimensional message embedding and one 128-unit GELU hidden layer. Training uses $400$ Adam steps, batches of $512$, five seeds, and three learning rates, with no early selection of the reported checkpoint. Evaluation sums exactly over the finite target--message distribution, so variation across runs comes from optimization alone.

\begin{table}[t]
\centering
\caption{\textbf{Source quantity versus value.} Uniform $M=256$ targets, $15$ exact queries. Information and oracle accuracy are exact. Learned values span $15$ runs (five seeds, three learning rates) at step $400$. A single number means all runs agree at the displayed precision.}
\label{tab:neural}
\begin{tabular}{lrrr}
\toprule
Source & Information (bits) & Oracle accuracy & Learned accuracy \\
\midrule
Balanced partition & 4.0 & 1.000000 & 1.000000 \\
Reveal--erasure, $\alpha=0.50$ & 4.0 & 0.531250 & 0.531250 \\
Reveal--erasure, $\alpha=0.55$ & 4.4 & 0.578125 & 0.578125 \\
Uninformative & 0.0 & 0.062500 & 0.062500 \\
\bottomrule
\end{tabular}
\end{table}
The test separates decoding a source correctly from the source's usefulness: the gap between sources remains under perfect posterior decoding. The alphabets are small and explicit, so the experiment does not show that neural representations in general reach the optimized-channel frontier. Full learning curves and the learning-rate sweep are released.

\section{Quotient-subspace optimization and certificates}\label{app:planner}
All spaces and ranks in this section are over $\mathbb F_2$. Work in $\mathcal V=\operatorname{span}(T,B)$, removing unobservable latent directions, so $\dim\mathcal V=k+h$. Gaussian elimination selects coordinates for the quotient by $\mathcal T=\operatorname{row}(T)$. The projected source rows $\bar v_j$ span the full quotient, although a particular lower-dimensional quotient subspace need not be generated by available rows.

\subsection{Exact profile and attaining subsets}
Let $\mathcal U=\{\operatorname{span}(\bar v_j:j\in S):S\subseteq[J]\}$. For $U\in\mathcal U$, the images of the eligible set $\mathcal E(U)$ span exactly $U$: inclusion gives one direction, and a generating set of $U$ is contained in $\mathcal E(U)$. Rank-nullity therefore gives
\[
 \dim\big(\operatorname{span}(v_j:j\in\mathcal E(U))\cap\mathcal T\big)
 =\operatorname{rank}(v_j:j\in\mathcal E(U))-\dim U=t(U).
\]
For any feasible set $S$, write $u=\dim\operatorname{span}(\bar v_j:j\in S)$. Its original-space rank is at most $|S|$, so $\rho(S)\le |S|-u\le b-u$. Since its source span is a subspace of the eligible source span, $\rho(S)\le t(U)$. This proves
$\rho_b\le\max_{U\in\mathcal U}\min(t(U),(b-\dim U)_+)$. Only states with $u\le b$ are needed for this direction; states with $u>b$ contribute zero and do not affect the inequality.

For the converse, choose $u$ eligible rows with independent quotient images. Any dependence of those rows in $\mathcal V$ would project to a dependence in the quotient, so the rows are independent. Their span has zero intersection with $\mathcal T$. Extend these rows to an ordered basis of the eligible source span. The extension has length $R=u+t(U)$; every added row leaves the quotient span unchanged but increases original rank by one. The first $u+s$ rows therefore have target gain $s$, for $0\le s\le t(U)$. Given $b\ge u$, the prefix of length $u+\min(t(U),b-u)$ attains the expression in~\eqref{eq:qsp}. When the expression is zero, the empty source set also attains it. This proves both directions, including zero sources, zero nuisance dimension, dependent or repeated source rows, and budgets exceeding the number of independent sources.

\subsection{Closure enumeration and complexity}
Represent each quotient subspace by its unique reduced row basis. Begin with the empty basis and maintain a queue and hash set of encountered bases. For each encountered $U$ and each available direction $\bar v_j$, compute $U+\operatorname{span}(\bar v_j)$ and insert its canonical basis if new. Every inserted state is generated by available directions. Conversely, induction over a generating source subset shows that every $U\in\mathcal U$ is eventually inserted. No inclusion in a nongenerated subspace is required.

There are at most $\mathcal N_h=\sum_{u=0}^h{h\brack u}_2$ states. The Gaussian coefficient counts reduced subspaces, and
\[
 {h\brack u}_2=\prod_{i=0}^{u-1}\frac{2^h-2^i}{2^u-2^i}
 \le 4\cdot2^{u(h-u)}\le 2^{h^2/4+2}.
\]
Thus $\mathcal N_h\le(h+1)2^{h^2/4+2}=2^{O(h^2)}$. Per state, at most $J$ Gaussian reductions generate its successors; testing eligibility, constructing original/quotient bases, and updating $J+1$ budget entries require $\operatorname{poly}(J,k+h)$ field operations. Preprocessing the input costs polynomial time in its original dimension; after removal of redundant coordinates the stated bound applies. The count includes polynomial-length source witnesses, not just optimum values.

A state limit in the implementation raises an explicit exception; it never calls a partially explored search globally optimal. The limit is 500,000 states (5,000 for the stress-test fits). Counts of reached states can be far below $\mathcal N_h$, but there is no polynomial-in-$h$ claim. All finite-field operations are exact bit operations; runtime measurements are properties of this implementation and machine, not a comparison to every coding solver.

\subsection{Checkable proof objects}
The export contains $T$, $B$, quotient images, every generated reduced quotient basis, the eligible source IDs, original and quotient ranks, and an ordered achieving source basis. A verifier checks: (i) the quotient map is linear on the joint row space, vanishes on $\mathcal T$, and has rank $h$; (ii) the state set contains zero and is closed under adjoining every available direction; (iii) every state is generated by its reported eligible directions; (iv) eligible sets and all ranks are exact; (v) each basis prefix has its advertised cost and target rank; and (vi) the reported profile is the maximum in~\eqref{eq:qsp} and each chosen subset attains it. Closure and the initial state rule out omission of a superior generated subspace.

The independent verifier uses dense binary matrices and a separate Gaussian-elimination implementation rather than importing the planner's bit-packed rank routines. These checks certify the mathematical optimization for the exported model. They do not verify that a fitted model describes a real data source.

\subsection{Online target generation}
After the selected source labels $z=B_SW$ are observed, solve that system for one representative $w_0$ and a basis of its nullspace. The compatible targets form $Tw_0+T(\ker B_S)$. Reduce the projected nullspace to a basis, and enumerate at most $n+1$ distinct linear combinations. A positive exact target query supplies the output; after $n$ failures use the last candidate. This avoids enumerating all latent worlds or all target candidates when the requested list is short. A learned model that receives inconsistent source labels reports that failure instead of inventing a posterior.

\section{Relation to relative generalized Hamming weights}\label{app:coding}
Let $B:\mathbb F_2^J\to\mathcal V$ send coordinate vector $e_j$ to source row $v_j$, and $N=pB$ with $p:\mathcal V\to\mathcal V/\mathcal T$ the quotient map. Define $K_0=\ker B\subseteq K_1=\ker N$. Let $\mathbb F_2^S$ denote the coordinate subspace supported on $S$. Restricted rank-nullity gives
\[
 \dim(K_1\cap\mathbb F_2^S)-\dim(K_0\cap\mathbb F_2^S)
 =(|S|-\operatorname{rank}N_S)-(|S|-\operatorname{rank}B_S)
 =\rho(S).
\]
In the standard relative-weight definition, the $t$th weight is the smallest support of a $t$-dimensional subcode $D\subseteq K_1$ with $D\cap K_0=\{0\}$ \citep{luo2005,kurihara2012,sanjose2025}. Such a subcode supported in $S$ exists if and only if the displayed dimension difference is at least $t$: one direction follows by injecting $D$ into the quotient; the other chooses a complement to $K_0\cap\mathbb F_2^S$ in $K_1\cap\mathbb F_2^S$. Hence the relative weights are precisely the inverse thresholds of $\rho_b$.

Because $K_1^\perp$ has dimension $h$, each $U$ corresponds to a subcode of $K_1^\perp$, its annihilator, and $\mathcal E(U)$ is the complement of that subcode's support. QSP is therefore enumeration on the small side of Forney's duality \citep{forney1994dimension}, which is the right side when the nuisance dimension is small.

The algorithms of \citet{sanjose2025} address generalized and relative weights using generalized Brouwer--Zimmermann search and a Sage implementation. Our search is over the dimension-$h$ image quotient and is intended for bounded nuisance dimension, with a learning and verification-policy interface. Appendix~\ref{sec:fullexternal} runs that package, unmodified and pinned, on the same inputs: the profiles agree wherever both finish, and neither implementation is uniformly faster. The released source-aware baselines are specified independently in Appendix~\ref{app:vfourexp}; they are not proxies for all exact coding algorithms.

\paragraph{Hardness without a nuisance bound.}
Given a binary syndrome-decoding instance with matrix $H$, nonzero syndrome $t$, and weight cap $b$, take source rows to be the columns of $H$ and the single target row to be $t$. Then $\rho_b\ge1$ holds if and only if $t$ lies in the span of at most $b$ columns, equivalently if $Hc=t$ has a binary solution of Hamming weight at most $b$. Thus deciding this profile contains classical syndrome decoding \citep{berlekamp1978}. This reduction locates the algorithm's parameter dependence; it is not claimed as a new coding-hardness result. The assumption $k=1$ does not bound $h$.

\section{Learning the full uniform-linear joint model}\label{app:learnlinear}
\begin{proposition}[Learning the bank]\label{thm:learnlinear}
Suppose training contexts supply the target answer and all $J$ source labels, drawn independently and uniformly from an unknown rank-$\ell$ subspace of $\F^{k+J}$ with target rank $k$, and let $\ell_0\ge\ell$ be a declared bound. If
\begin{equation}\label{eq:sample}
 m\ge\ell_0+\lceil\log_2(1/\delta)\rceil ,
\end{equation}
then with probability at least $1-\delta$ the span of $m$ contexts is the whole subspace, and QSP applied to it returns optimal policies for every budget, with nuisance dimension $h=\ell-k$.
\end{proposition}
We prove \cref{thm:learnlinear} using the closure argument of \citet{helmbold1992learning}, specialized to uniform sampling. Let the true joint support be a rank-$\ell$ linear space $\mathcal X\subseteq\mathbb F_2^{k+J}$, with target marginal rank $k$. Independent training vectors are uniform on $\mathcal X$. Their span is always a subspace of $\mathcal X$. If it is proper, some nonzero linear functional on $\mathcal X$ annihilates every sample. There are $2^\ell-1$ such functionals, and each vanishes on one uniform draw with probability $1/2$. Consequently
\[
 P(\widehat{\mathcal X}\ne\mathcal X)\le(2^\ell-1)2^{-m}<2^{\ell-m}.
\]
For $m\ge \ell$, the exact success probability is
\begin{equation}\label{eq:spanexact}
 P(\widehat{\mathcal X}=\mathcal X)=\prod_{i=0}^{\ell-1}(1-2^{i-m});
\end{equation}
for $m<\ell$ it is zero. To derive the product, choose coordinates on $\mathcal X$ and view the samples as the rows of an $m\times \ell$ matrix with independent uniform entries. Its columns are independent uniform vectors of $\mathbb F_2^m$. The first column avoids zero with probability $1-2^{-m}$, and the next avoids a previously spanned dimension-$i$ space with probability $1-2^{i-m}$. The formula is classical random-matrix counting \citep{fulman2015}.

Given the observed basis $\Gamma$ (rows of length $k+J$), every point of $\widehat{\mathcal X}$ is represented uniquely as $\alpha\Gamma$. On exact recovery, uniform $X$ corresponds exactly to uniform $\alpha\in\mathbb F_2^\ell$. The first $k$ columns of $\Gamma$, transposed into coefficient rows, form the target matrix; the remaining columns form the source bank. A change of row basis merely changes latent coordinates, not any observable distribution, quotient dimension, or acquisition value. Thus on the single recovery event, all the hypotheses of the linear-rank theorem hold simultaneously for all budgets. No union bound over budgets or future test contexts is required. The learned nuisance dimension is then exactly $\ell-k$.

Training computation is polynomial-time binary row reduction. Planning costs $\mathcal N_{\ell-k}\operatorname{poly}(J,\ell)$ field operations. For declared $\ell_0\ge \ell$, the sample size in~\eqref{eq:sample} makes the common recovery event have probability at least $1-\delta$. Conditional on that event, zero-error policies have a worst-case online guarantee over the support. For other budgets their accuracy is the optimal uniform-prior accuracy, not a guarantee of success on every world.

The observation cost is $m$ target demonstrations and $Jm$ source labels. A binary source label can be acquired by one exact verification of zero; a target demonstration conveys the whole $k$-bit target. The online budget does not include either pretraining resource. Knowledge of the uniform-linear model class, a rank bound for selecting sample size, and exact labels are assumptions. When the sample does not span, the empirical model may falsely identify constraints and produce wrong outputs. Our estimator abstains if even the target rank is deficient, but this alone cannot certify full recovery. These failures are retained in the released learning curves.
\section{Planner and learning experiments}\label{app:vfourexp}
These experiments use finite binary linear tasks, exact algebra, and noiseless source and target acceptance mechanisms. A target query tests the entire target vector, while a source call reveals one binary source label. The learning method is row-space estimation followed by algebraic planning. The seed plan, every instance, raw training labels, held-out worlds, all budget profiles, and timings are released in \texttt{results\_v4/}.

\subsection{Instance generation and input separation}
The base seed is 202609061. An independent bank has target rows spanning $k$ target coordinates and $J$ source rows with independently sampled target coefficients and nuisance coefficients in $\mathbb F_2^h$. The clustered variant samples two thirds of nuisance coefficients from a lower-dimensional quotient subspace and the remainder from the full quotient. Random invertible elementary coordinate operations and a source permutation hide the presentation while preserving the linear dependence structure; they do not sample uniformly from all invertible matrices.

The main comparison uses 16 independently generated banks of each random family with $k=16,h=4,J=48$. A third diagnostic starts with the parity-check matrix of the binary length-15 Hamming code, whose columns are all nonzero vectors of $\mathbb F_2^4$. The target rows are a basis of its 11-dimensional kernel. Fifteen sources give the underlying 15 latent bits. Two further source pairs give masked linear target coordinates, introducing two independent pads and four sources. Thus $k=11$, $J=19$, the joint rank is 17, and $h=6$. Eight coordinate/source scramblings of this one construction test the algorithms' sensitivity to presentation.

For the acquisition comparisons, each instance produces $\ell+12$ joint labeled contexts. The learner receives only the packed labels, with the first $k$ bits denoting target labels and the remaining $J$ bits source labels. It is not given latent world vectors, source coefficients, target coefficients, or coordinate permutations. All 40 fits recover the support in this sample; a fit that did not would be recorded as a failure, not resampled. The evaluator checks model recovery, and its latent information is never passed to a planner. Every baseline receives the same fitted coefficient representation.

\subsection{Baselines and budget matching}
For a current source set $S$, lookahead-$\ell$ evaluates the \emph{joint} target-rank gain per added source over all unselected blocks of size at most $\ell$, rather than summing one-step gains. It adds a maximizing block, records each prefix, and repeats until the largest plotted source cap. Among tied scores the implementation uses a fixed seeded priority order; if all joint gains are zero, it first favors an independent row with minimum added nuisance rank. One- and two-step methods run on every bank. Four-step lookahead and eight seeded two-step restarts run on the coded diagnostic only.

The randomized comparator generates 32 linearly independent source-basis orderings. For each source budget it chooses the most informative prefix among these schedules, which is stronger than querying all $J$ raw sources or using a single random ordering. One-exchange local search starts from the best lookahead and randomized subsets at each cap. It pads to the permitted source size and repeatedly makes a strictly rank-improving swap of one chosen and one unchosen source. Two-exchange search, only on the coded diagnostic, starts from the one-exchange solutions and permits removal/addition of up to two sources. It therefore corrects dependencies missed by simple lookahead on this example. Its recorded timing includes initialization, restarts, and the randomized and greedy starting solutions.

For each strategy and each total budget $C$, all available source caps $b\le C$ are compared using
$\min(1,(C-b+1)2^{-k+\rho(S_b)})$. The methods may choose different source/target splits but have the same total cap. The displayed operating point $C^\star$ is the smallest budget for which the exact profile reaches one; it is chosen from the exact model. All curves for $C$ from $k-2$ to $C^\star+2$ are released. Population accuracies are computed from the exact posterior support rather than by Monte Carlo. Means over random banks describe variation across instances, not a test-set confidence interval.

A source-only reconstruction comparator can discard linearly redundant observations and query a basis of all source rows. For the three main families it uses 20, 20, and 17 calls, while mean optimal mixed budgets are 18.4375, 17.6875, and 15. These savings over source-aware reconstruction are modest. The exact planner additionally gives every partial-budget optimum and a certificate. Two-exchange local search matches its accuracy on the coded family.

\subsection{General-bank model learning and undersampling}
The learning study generates 80 banks, alternating $k=8$ and $16$, cycling $h=1,2,3,4$, taking $J=4k$, and alternating independent and clustered generations. Each bank has joint rank $\ell=k+h$. The declared generative dimension sets the sample sizes; the fitting function is given only label dimensions and samples. We use training prefixes with $m-\ell\in\{-2,0,2,4,7,10\}$ of a fixed pool of $\ell+10$ independent contexts. Each bank generates 512 further worlds not appearing in its training pool. These test worlds are reused across prefixes, so the comparison is paired: each sample size has 40,960 context evaluations on the same contexts.

The estimator abstains if its target coefficient rows are not full rank. Otherwise it plans in the fitted support, chooses that model's smallest zero-error cap, acquires only the selected true source labels, and tests the generated target list against the actual target. The evaluator records support recovery, target-rank abstention, incompatible observations, query caps, and correctness. False fitted relations are not replaced by the true ones. Below, ``expected'' sums~\eqref{eq:spanexact} over banks; ``test errors'' count abstentions as failures, out of 40,960 at every prefix.
\begin{center}\begin{tabular}{rrrrr}\toprule
$m-\ell$ & Recovered (banks of 80) & Expected (banks) & Abstentions (banks) & Test errors (contexts) \\ \midrule
-2 & 0 & 0.00 & 45 & 36544 \\
0 & 18 & 23.11 & 19 & 20480 \\
2 & 59 & 61.62 & 6 & 7008 \\
4 & 75 & 75.11 & 2 & 1799 \\
7 & 80 & 79.38 & 0 & 0 \\
10 & 80 & 79.92 & 0 & 0 \\
\bottomrule\end{tabular}\end{center}

At $m=\ell+7$, the sufficient failure bound is $2^{-7}<0.01$, with 16--27 target demonstrations and 512--1728 source calls per bank. All 80 fits recover and all held-out policies succeed. At $m=\ell$ and below, fitted supports are often incorrect and test contexts fail. The model is learned as a full uniform-linear joint distribution rather than through supplied pairs; it does not cover nonlinear, noisy, or nonuniform latent relations.

\subsection{Scaling, exact checks, and reproducibility}
The scaling grid crosses $k=16,32,64$, $J=4k$, and $h=2,4,6$, with three generated inputs per cell. The solver receives coefficient matrices in this computation-only experiment; the separate learning study measures labeled-data requirements. Timings measure the complete QSP call, including quotient preprocessing and construction of every budget witness, and exclude file serialization and data generation. The environment and raw timings are in \texttt{summary.json} and \texttt{scaling.json.gz}. They measure our Python implementation; the comparison with an independent solver is in Appendix~\ref{app:externalv5}.
\begin{center}\begin{tabular}{rrrrr}\toprule
$k$ & $J$ & $h$ & Mean seconds & Max. states \\ \midrule
16 & 64 & 2 & 0.0009 & 5 \\
32 & 128 & 2 & 0.0042 & 5 \\
64 & 256 & 2 & 0.0286 & 5 \\
16 & 64 & 4 & 0.0086 & 67 \\
32 & 128 & 4 & 0.0434 & 67 \\
64 & 256 & 4 & 0.2813 & 67 \\
16 & 64 & 6 & 0.2915 & 2543 \\
32 & 128 & 6 & 1.0044 & 2796 \\
64 & 256 & 6 & 5.1482 & 2825 \\
\bottomrule\end{tabular}\end{center}
\begin{figure}[h]
\centering\includegraphics[width=.70\linewidth]{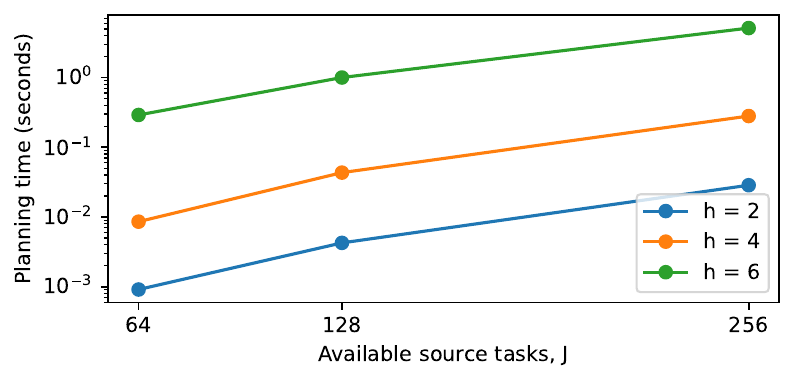}
\caption{\textbf{Scaling with available sources and nuisance dimension.} Each point averages three generated inputs. Target dimension is $k=J/4$. The number of subspace states is at most $\mathcal N_h=2^{O(h^2)}$ ($5$, $67$ and $2825$ for $h=2,4,6$; \cref{app:planner}).}\label{fig:qsptime}
\end{figure}

Small-instance checks use 240 additional random banks with latent dimension at most seven and at most eleven source tasks. Exhaustive enumeration of source subsets agrees on all 1911 budget entries, using exact integer ranks. The independent certificate checker validates 40 closure-complete exports: 4104 subspace states and 1728 profile entries. Tamper tests cover incorrect ranks and incomplete state sets. Additional regression tests compare relative-weight definitions, Gaussian subspace counts, affine decoding, learned coordinate changes, and exact Bellman values.

The planner suite (\texttt{python -m revision.run\_planner\_experiments}), the independent certificate verifier (\texttt{revision.verify\_planner\_certificate}) and the saved-label and outcome audit (\texttt{revision.audit\_learned\_plans}) produce and check these results, and \texttt{revision.make\_v4\_assets} regenerates their macros, runtime table and plots. No pretrained model is used. The audits check code and proof witnesses for finite models. The numbered results are checked separately in Lean~4 (start of the appendix).
\section{External comparison, local search and stress tests: results}\label{sec:fullexternal}
This appendix reports in full the three experiments summarized in Section~\ref{sec:planner}: a comparison with a published solver, the local-search comparison behind the best-baseline numbers, and stress tests of the learning proposition under label noise and nonuniform training. They measure, in turn, the planner's speed in the parameter regime it targets, the quality of its acquisitions against source-aware search, and the sensitivity of the learned model to its assumptions. The external comparison receives known matrices; the local-search and stress studies fit models from labels. Protocols, software environments, seeds, acquisition costs, uncertainty and the replay procedure are in Appendix~\ref{app:vfive}.

\paragraph{Published coding solver.}
We run the unmodified GHWs package of \citet{sanjose2025} at a pinned commit. Both solvers compute the complete rank profile. The kernels $K_0=\ker B\subseteq K_1=\ker N$ of~\eqref{eq:relative}, constructed in \cref{app:coding}, are nested codes, where $N$ contains nuisance-quotient columns and $B$ contains full source columns; GHWs computes their relative weight hierarchy. We use six independently generated banks in each of four $(k,h,J)$ regimes and three standard Reed--Muller constructions. Both methods run under the same SageMath 10.6 process environment on one CPU, with fresh workers, alternating order, two repetitions, and a 15-second wall cap per invocation. Timings exclude imports and worker startup but include input conversion; raw wall times and memory are also saved.
\ifdefined\MergedPaper
The table of times and completions is Table~\ref{tab:externalv5} in Section~\ref{sec:planner}.
\else
\begin{table}[t]
\centering\footnotesize
\setlength{\tabcolsep}{4pt}\begin{tabular}{lcrcrrr}\toprule
Bank family & $(k,h,J)$ & Banks & Agree & QSP (ms) & GHWs (ms) & Done\\\midrule
Low nuisance & $(4,2,12)$ & 6 & 6/6 & 0.37 & 30.91 & 6/6 \\
Wider low nuisance & $(8,3,24)$ & 6 & -- & 1.49 & \textit{timeout} & 0/6 \\
Intermediate nuisance & $(4,5,12)$ & 6 & 6/6 & 6.87 & 50.72 & 6/6 \\
High nuisance, short bank & $(3,7,12)$ & 6 & 6/6 & 35.46 & 18.25 & 6/6 \\
RM(1,3) & $(2,2,8)$ & 1 & 1/1 & 0.37 & 13.13 & 1/1 \\
RM(1,4) & $(2,3,16)$ & 1 & 1/1 & 0.78 & 16.44 & 1/1 \\
RM(2,4) & $(4,7,16)$ & 1 & 1/1 & 106.79 & 18.95 & 1/1 \\
\bottomrule\end{tabular}
\caption{\textbf{An independent implementation computes the same profiles.} Where both solvers finish, their complete rank profiles agree on every bank. Both compute the complete rank profile from identical inputs on one CPU, with two fresh-worker repetitions and a common 15-second wall cap. Times are medians over per-bank medians, excluding imports/worker startup but including input conversion. QSP completes all 27 banks; timeout entries are censored. RM rows are standard Reed--Muller constructions.}\label{tab:externalv5}
\end{table}
\fi
QSP completes all 27 banks and GHWs 21, and every profile computed by both agrees exactly. The twelve censored GHWs invocations are the two repetitions on each of the six wider low-nuisance banks. GHWs is faster on the high-nuisance short banks and on RM(2,4), so the speed advantage depends on the parameter regime. The scaling grid of Appendix~\ref{app:vfourexp}, which does not involve GHWs, reaches $k=64,J=256,h=6$ in about $\QSPLargeTime$ seconds.

\paragraph{Two-exchange search on every acquisition family.}
We extend the largest-neighborhood local-search comparator to all 40 label-fitted banks: 16 independent and 16 nuisance-clustered banks with $(k,h,J)=(16,4,48)$, and eight scramblings of one coded geometry. Both methods receive the same fitted support, and each optimizes its own source/target split. The baseline starts from one-/two-step lookahead and the best of 32 random independent-source bases, then exhaustively considers one- and two-element swaps with steepest strict rank improvement. A residual-rank implementation evaluates the same neighborhood faster and is checked against the original exhaustive implementation. Timings include initialization, which uses neither the exact optimum nor the true coefficients.
\begin{table}[t]
\centering\small
\begin{tabular}{lrrrrrr}\toprule
 & & \multicolumn{2}{c}{Two-exchange run} & \multicolumn{2}{c}{Best released baseline} & QSP\\
Family & Banks & Acc. at $C^\star$ & Gap (pp) & Acc. at $C^\star$ & Gap (pp) & Acc. at $C^\star$\\\midrule
Independent linear & 16 & 0.734 & 2.83 & 0.750 & 2.28 & 1.000 \\
Clustered linear & 16 & 0.812 & 2.36 & 0.875 & 1.10 & 1.000 \\
Coded scramblings & 8 & 1.000 & 0.01 & 1.000 & 0.00 & 1.000 \\
\bottomrule\end{tabular}
\caption{\textbf{Local search on every family, against the best released baseline.} Every method receives the same label-fitted model and optimizes its own source/target split. The two-exchange run uses a larger neighborhood than the lookahead and one-exchange heuristics of Appendix~\ref{app:vfourexp} but different random starts, and loses to one-step lookahead at $C^\star$ on three of the forty banks; the right-hand pair therefore reports the per-bank best over \emph{all} released source-aware baselines. Gap averages QSP-minus-baseline population accuracy over all integer budgets from zero through source rank plus one. QSP has accuracy one at $C^\star$ by definition of $C^\star$. The coded rows are scramblings of one geometry.}\label{tab:exchangev5}
\end{table}
At the exact zero-error cap $C^\star$, two-exchange accuracy is $\VFiveIIDExchange$ on the independent banks and $\VFiveClusterExchange$ on the clustered banks, and it matches QSP on all coded variants. The larger neighborhood does not always help, because this run's lookahead breaks ties with different seeds than in Appendix~\ref{app:vfourexp} (its random bases also differ): on $\VFiveIIDRegressions$ independent and $\VFiveClusterRegressions$ clustered banks, the one-step lookahead there reaches accuracy one at $C^\star$ and two-exchange does not. We therefore also report the per-bank best over all released source-aware baselines: $\VFiveIIDPortfolio$ and $\VFiveClusterPortfolio$ at $C^\star$. Evaluation at $C^\star$ magnifies small differences in query counts. Averaged over the complete integer-budget grid, the gaps to QSP are $\VFiveIIDRegret$ and $\VFiveClusterRegret$ percentage points for the two-exchange run and $\VFiveIIDPortfolioRegret$ and $\VFiveClusterPortfolioRegret$ for the best baseline; intervals and runtimes are in Table~\ref{tab:exchangeciv5}, and the lookahead and one-exchange results in Appendix~\ref{app:vfourexp}. Against source-aware baselines, QSP's savings are modest, not exponential.

\paragraph{Source-label noise.}
On 24 further banks with $(k,h,J)=(8,3,24)$, we fit from 18 joint labeled contexts and decode 512 independent test contexts per condition. Only source replies are flipped, independently, at rates $0,0.001,0.01,0.03,0.1$; target demonstrations and target verification remain exact. The span learner is used unchanged, with a fixed computational guard that abstains when the fitted nuisance dimension exceeds six. We report \emph{success yield} over every attempted context, counting abstention as unsuccessful, and accuracy conditional on producing a plan.

At 1\% noise, single measurements yield $\VFiveNoiseSingle$ overall: 13 of 24 banks abstain, and accuracy on the remaining eleven is $\VFiveNoiseConditional$. Three-measurement majority voting at both training and test raises yield to $\VFiveNoiseTriple$ (bank-bootstrap 95\% interval $[0.9972,0.9986]$), but increases training source calls from 432 to 1296 and the mean online physical cap to $\VFiveTripleCap$, compared with the clean single-measurement cap $\VFiveCleanCap$. All runs additionally use 18 noiseless target demonstrations. At 3\% noise the learner without repetition abstains on every bank. The span learner is thus not robust to source-label noise, and QSP's certificates hold for the \emph{fitted} model, not for the true bank.

\paragraph{Nonuniform training and distribution shift.}
Using the same 24 banks without label noise, we draw independent latent bits with training probabilities $p=0.5,0.2,0.05$ and fit at $m=18,22,44$. With $p=0.5,m=18$, all supports recover and all evaluated policies succeed. With $p=0.2$, recovery rises from 17/24 at $m=18$ to 24/24 at $m=44$. With $p=0.05,m=44$, only 9/24 recover: yield is $\VFiveBiasedMatched$ under the training prior and $\VFiveBiasedUniform$ on uniform test worlds. Span learning therefore needs the training data to cover the latent support, and the finite-sample guarantee, which assumes uniform data, does not apply to these distributions. Every training context costs one target demonstration and 24 source calls.

\paragraph{Validation and uncertainty.}
The test suite passes $\VFiveTestCount$ tests. Saved-label replay checks all 368,640 stored context outcomes across the \VFiveTotalCells{} stress cells and re-decodes the \VFivePlannedCells{} of them that produced a plan; every successful external profile is recomputed with QSP. Intervals use 10,000 bootstrap resamples of banks, not of test contexts; conditions and training prefixes are paired. Complete protocols and results, including all noise and prior-shift failures, are in Appendix~\ref{app:vfive}.
\section{External comparison, local search and stress tests: protocols}\label{app:vfive}
This appendix gives the protocols for the experiments of Appendix~\ref{sec:fullexternal}.

\subsection{Pinned external solver and equal computational task}\label{app:externalv5}
The GHWs package \citep{sanjose2025} is released under the GPL-3.0 license. We obtain it from its public repository at commit
\texttt{eec8adb39deeac01bd0701026747c945367e4501}, invoke it unmodified, and do not redistribute its source. We call
\texttt{rhierarchy(C1,C0,verbose=False)} with default algorithmic options and do not tune it or test its optional low-memory variants. Before timing, its documented six-coordinate example is run and returns the documented weights $(3,5)$.

For full source columns $B$ and quotient columns $N$, both over $\mathbb F_2$, we construct the nested codes
\[
 C_1=K_1=\ker N,\qquad C_0=K_0=\ker B\subseteq C_1.
\]
A relative generalized weight $d_r(C_1,C_0)$ equals the first source budget at which $\rho_b\ge r$. Consequently the complete relative weight hierarchy determines the same all-budget profile as QSP. GHWs timing includes code conversion, kernel computations, and hierarchy-to-profile conversion; QSP timing includes quotient preprocessing, state enumeration, and construction of its complete profile and source witnesses. Neither includes generating inputs or importing Sage and the packages. GHWs is not asked to produce QSP's planner certificates.

\paragraph{Inputs, randomness, and limits.}
The four random regimes are $(k,h,J)\in\{(4,2,12),(8,3,24),(4,5,12),(3,7,12)\}$, with seeds $i=0,\ldots,5$. For each, Python's \texttt{random.Random} receives seed $70001+1000k+100h+i$; source columns are independent draws from the nonzero binary $(k+h)$-vectors, resampled as an entire bank until they span that space. Target rows are the first $k$ coordinate vectors. Three further inputs use generator-matrix columns from Sage's binary Reed--Muller codes RM(1,3), RM(1,4), and RM(2,4), with target ranks $2,2,4$. All 27 banks are saved in \texttt{results\_v5/external/instances.json}.

Both solvers execute in fresh forked workers after imports, in alternating order for two repetitions. A parent process imposes the same 15-second wall deadline on every worker; at the deadline it terminates the worker and records a timeout. The time reported for completed runs is the internal algorithm timer, and wall time including worker overhead is saved separately. Timeouts are reported as censored: they are not converted to finite runtimes, dropped from completion counts, or counted as incorrect profiles. Group summaries take the median of the two timings for each bank, then the median over banks. The experiment has 108 solver invocations: 96 complete and 12 are censored, with no recorded solver errors. All 54 QSP invocations complete, and all 42 completed GHWs invocations agree exactly with QSP. The six GHWs-censored banks are censored in both repetitions.

The external run uses SageMath 10.6, Python 3.12.5, and the official Sage Docker image with its immutable digest recorded in \texttt{image.json}. The container has one CPU, a 5 GB memory limit, and no network access during algorithm execution. CPU information is saved in \texttt{cpu.txt}. Peak resident memory includes the Sage interpreter's footprint. Timings in the local-search and stress experiments below come from a different environment and are not compared with these.

\paragraph{Crossover and preprocessing.}
Table~\ref{tab:externalv5} reports every regime. On the high-nuisance short banks the median runtime is 35.46 ms for QSP and 18.25 ms for GHWs. On RM(2,4) it is 106.79 versus 18.95 ms. On the wider low-nuisance banks QSP takes a median 1.49 ms, while neither GHWs repetition finishes within 15 seconds. Per-invocation conversion time, total time, returned profiles, censoring flags, and memory are released. With six banks per random regime and two timing repetitions, these runs locate the crossover but do not give precise speed ratios.
\begin{figure}[ht]
\centering\includegraphics[width=.70\linewidth]{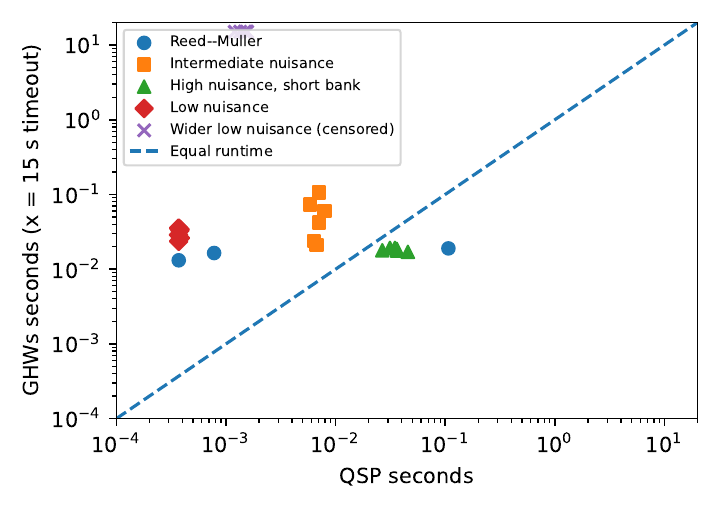}
\caption{\textbf{Every external-comparison instance.} Each point is the median of two repetitions. Crosses at 15 seconds are censored GHWs runs. Points on both sides of the diagonal show the crossover.}\label{fig:externalv5}
\end{figure}

\subsection{Expanded two-exchange local search}\label{app:exchangev5}
The 40 banks and their joint training labels are those of Appendix~\ref{app:vfourexp}: 16 independent, 16 clustered, and eight coded scramblings. We refit the model from each stored label matrix and check that its fitted support equals the saved support. QSP and the baseline receive only this fitted target/source representation; the true coefficients are used only by the evaluator. All source queries are charged, and each strategy chooses its own source/target split.

The baseline initializes a per-budget best subset from one- and two-step joint-information lookahead and 32 randomized independent-source bases. It then examines all one- and two-element replacements, selecting the largest strictly positive transfer-rank increase, and repeats until none exists, so it stops at a zero-improvement plateau. Ties have deterministic iteration order. Its timings include the lookahead and random-basis initialization.

A residual-rank implementation evaluates the same neighborhood more efficiently. After dropping one or two source indices, reduce each candidate full vector and quotient vector modulo the respective remaining row spaces. The rank increment of two binary residuals $u,v$ is
\[
 \ind\{u\ne0\text{ or }v\ne0\}+
 \ind\{u\ne0,v\ne0,u\ne v\}.
\]
Subtracting quotient-rank increment from full-rank increment gives the exact transfer-rank change. Twelve tests check on small banks that complete profiles and selected subsets agree with the brute-force neighborhood evaluator. The seeds and every fitted model, initial/final subset, runtime, and all-budget curve are in \texttt{two\_exchange\_all\_families.json}.

Let $C^\star$ be the exact model's smallest total-call cap yielding zero error; QSP has accuracy one there by construction. Each bank also records its accuracy at every integer $C=0,\ldots,\operatorname{rank}(B)+1$, and the mean gap in Table~\ref{tab:exchangev5} averages the population-accuracy difference over this grid, with equal weight on each budget. In the coded family the heuristic matches the zero-error threshold in all eight scramblings, which does not imply equality at every smaller budget.
\begin{table}[ht]\centering\small
\begin{tabular}{lrrrr}\toprule
Family & Acc. 95\% interval & Gap interval (pp) & QSP (ms) & Two-exchange (ms)\\\midrule
Independent & [0.625,0.844] & [2.13,3.50] & 10.56 & 500.65 \\
Clustered & [0.688,0.938] & [1.34,3.40] & 8.70 & 480.65 \\
Coded & [1.000,1.000] & [0.00,0.03] & 18.33 & 76.81 \\
\bottomrule\end{tabular}
\caption{Bank-bootstrap percentile intervals and mean runtimes for the expanded local-search comparison. Coded intervals reflect only eight scramblings of one geometry. These timings come from a different host than those of Table~\ref{tab:externalv5}.}\label{tab:exchangeciv5}
\end{table}

Confidence intervals use 10,000 percentile bootstrap resamples with bank as the resampling unit and seed 202609079. For the coded family they reflect only variation across scramblings of one geometry, and coded banks are not pooled with the independently generated ones.

\subsection{Source-label noise and repeated measurements}\label{app:noisev5}
The stress generator creates 24 banks with $k=8,h=3,J=24$, alternating the independent and clustered generators with seed $202609071+i$, $i=0,\ldots,23$. Every joint target/source support has rank 11. The logical source/target budget is fixed to each clean bank's exact zero-error cap $C^\star$ (mean 9.25 calls). This cap is chosen from the true bank, so the study does not test whether the learner can choose a safe budget from corrupted data. We fit from $m=18$ joint labeled contexts and draw 512 independent test contexts in each cell.

Noise flips each source measurement independently with $\epsilon\in\{0,.001,.01,.03,.1\}$. The target demonstrations in training remain noiseless, and test-time target equality queries remain exact. For repetition counts $r\in\{1,3,9\}$, the source label is the majority of $r$ independently noisy measurements in both training and test. Every physical measurement is charged: training costs $18\cdot24\cdot r$ source queries and 18 target demonstrations, and a fitted policy that selects $b$ sources and $n$ target queries has online physical cap $rb+n$.

All conditions reuse each bank's clean training and test worlds and uniform noise arrays, with repetitions using prefixes, so cells are paired rather than independent. Worlds are independent draws from a finite space, so training and test worlds can coincide. The learner sees only its training joint labels. The evaluator stores labels for all potential test sources, but decoding reads only the selected indices; unit tests check that changing unrequested test source labels leaves the outcome unchanged.

\paragraph{Model misspecification and abstention.}
We apply the empirical-span learner without a robust-fitting change. It abstains if the target rank cannot be represented or if the inferred nuisance dimension exceeds the fixed computational guard $h=6$. As $\mathcal N_h=2^{O(h^2)}$, the guard caps planning at $\mathcal N_6=2825$ states, as timed in \cref{app:vfourexp}. Failed fits are not repaired with the true coefficients: planned policies keep their fitted supports, which can be incorrect, and contradictory observed source equations also produce unsuccessful outcomes.

We report support recovery, the number of banks producing a plan, \emph{success yield} over all test attempts, and accuracy conditional on a plan. Abstentions count as unsuccessful in yield but not as incorrect answers. The online caps shown are averages over planned banks; training calls are incurred on abstaining banks too.
\begin{table}[ht]
\centering\small
\begin{tabular}{rrrrrrrr}\toprule
Noise (\%) & Reps & Recovered & Planned & Yield & Conditional & Online cap & Train calls\\\midrule
0 & 1 & 24/24 & 24/24 & 1.0000 & 1.0000 & 9.25 & 432 \\
0 & 3 & 24/24 & 24/24 & 1.0000 & 1.0000 & 25.25 & 1296 \\
0 & 9 & 24/24 & 24/24 & 1.0000 & 1.0000 & 73.25 & 3888 \\
0.1 & 1 & 14/24 & 24/24 & 0.9914 & 0.9914 & 9.25 & 432 \\
0.1 & 3 & 24/24 & 24/24 & 1.0000 & 1.0000 & 25.25 & 1296 \\
0.1 & 9 & 24/24 & 24/24 & 1.0000 & 1.0000 & 73.25 & 3888 \\
1 & 1 & 0/24 & 11/24 & 0.4020 & 0.8771 & 9.18 & 432 \\
1 & 3 & 20/24 & 24/24 & 0.9979 & 0.9979 & 25.25 & 1296 \\
1 & 9 & 24/24 & 24/24 & 1.0000 & 1.0000 & 73.25 & 3888 \\
3 & 1 & 0/24 & 0/24 & 0.0000 & -- & -- & 432 \\
3 & 3 & 5/24 & 24/24 & 0.9705 & 0.9705 & 25.17 & 1296 \\
3 & 9 & 24/24 & 24/24 & 1.0000 & 1.0000 & 73.25 & 3888 \\
10 & 1 & 0/24 & 0/24 & 0.0000 & -- & -- & 432 \\
10 & 3 & 0/24 & 0/24 & 0.0000 & -- & -- & 1296 \\
10 & 9 & 18/24 & 24/24 & 0.9931 & 0.9931 & 73.25 & 3888 \\
\bottomrule\end{tabular}
\caption{\textbf{All source-noise conditions.} Yield counts abstaining banks as unsuccessful; conditional accuracy excludes them. ``Recovered'' means the fitted support equals the true support, which is used for evaluation only. Online cap is the mean physical cap over planned banks, charging every repeated source measurement. Every row also costs 18 noiseless target demonstrations per bank.}\label{tab:noisev5}
\end{table}
\begin{figure}[ht]
\centering\includegraphics[width=.76\linewidth]{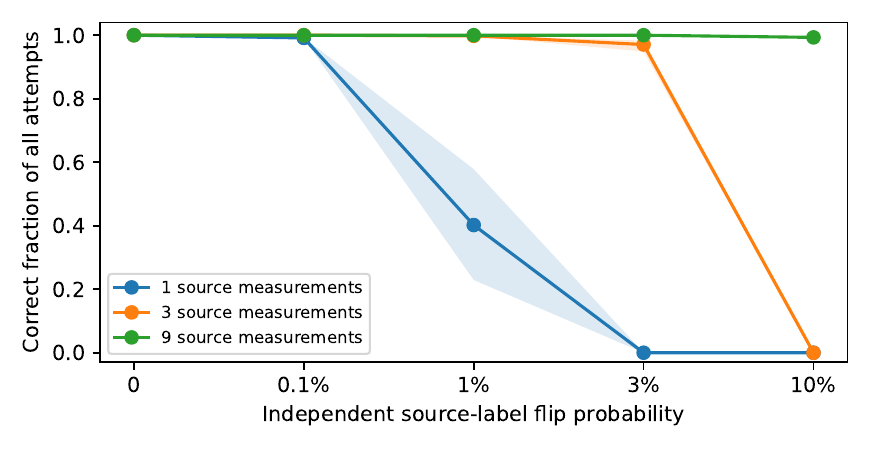}
\caption{\textbf{The span learner is not noise-robust.} Bands are bank-bootstrap intervals for yield, with abstention counted as failure. Repetition improves some conditions but multiplies acquisition cost, so the curves are not at equal physical budgets. Target feedback stays exact.}\label{fig:noisev5}
\end{figure}

The noise study has 360 bank--noise--repetition cells. At 1\% noise without repeats, 13 banks abstain and yield is 0.4020 (95\% interval $[0.2292,0.5781]$). Conditional accuracy on the eleven planned banks is 0.8771. Three repeats produce plans on every bank and yield 0.9979 ($[0.9972,0.9986]$), at a mean physical online cap of 25.25 and 1296 training source calls. At 3\% noise, one measurement produces no plans; three repeats yield 0.9705 ($[0.9484,0.9829]$). Nine repeats incur 3888 training source calls and a mean physical online cap of 73.25 on planned banks. Repetition factors were not optimized.

\subsection{Nonuniform training and prior shift}\label{app:biasv5}
The same 24 banks are evaluated without source-label noise. In the generator's latent coordinates, bits are independent Bernoulli($p$), with $p\in\{.5,.2,.05\}$. Target/source coefficient matrices remain fixed; only the latent distribution changes. Training sizes are $m\in\{18,22,44\}$. For $p<.5$, we evaluate both matched-prior test draws and a shift to independent uniform latent bits; for $p=.5$ these coincide and are counted once. There are therefore 216 fitted models and 360 evaluation cells.

Training sets use prefixes of the same bank-specific uniform arrays, and test conditions use paired independent uniform arrays. Each evaluation uses 512 test draws, with repeat values allowed. The fitting procedure still treats its estimated support as uniform and uses the same clean threshold $C^\star$, so the prior is misspecified. A nonuniform law with full support eventually reveals the correct span, but the finite-sample bound for uniform data does not apply. When the true support is recovered, a zero-error list covers that support whatever its prior weights; below zero error, the planner's optimality guarantee assumes the uniform prior.
\begin{table}[ht]
\centering\small
\begin{tabular}{rrrrrrr}\toprule
Train bit $p$ & Contexts & Recovered & Planned & Same-prior yield & Uniform-test yield & Train calls\\\midrule
0.05 & 18 & 0/24 & 4/24 & 0.1436 & 0.0405 & 432 \\
0.05 & 22 & 0/24 & 8/24 & 0.2834 & 0.0731 & 528 \\
0.05 & 44 & 9/24 & 22/24 & 0.8761 & 0.6416 & 1056 \\
0.2 & 18 & 17/24 & 24/24 & 0.9213 & 0.8516 & 432 \\
0.2 & 22 & 22/24 & 24/24 & 0.9836 & 0.9563 & 528 \\
0.2 & 44 & 24/24 & 24/24 & 1.0000 & 1.0000 & 1056 \\
0.5 & 18 & 24/24 & 24/24 & 1.0000 & 1.0000 & 432 \\
0.5 & 22 & 24/24 & 24/24 & 1.0000 & 1.0000 & 528 \\
0.5 & 44 & 24/24 & 24/24 & 1.0000 & 1.0000 & 1056 \\
\bottomrule\end{tabular}
\caption{\textbf{All nonuniform-training and shift conditions.} Each training context costs one noiseless target demonstration in addition to the source calls shown. Rows share the same 24 banks and test random numbers. The uniform-data theorem does not provide its stated finite-sample recovery guarantee at $p<1/2$.}\label{tab:biasv5}
\end{table}
\begin{figure}[ht]
\centering\includegraphics[width=.76\linewidth]{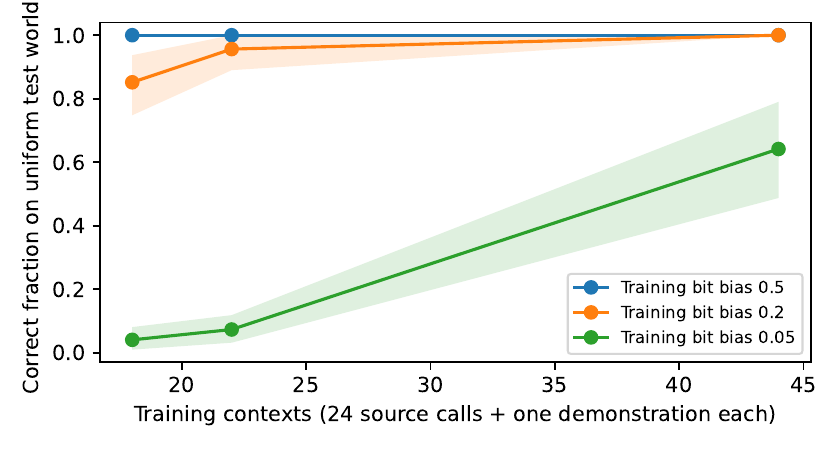}
\caption{\textbf{Coverage of the latent support matters.} Uniform-test yield after fitting under each training prior; bands resample banks. More training contexts overcome moderate bias but not the strongest bias tested. Each training context costs 24 source calls and one noiseless target demonstration.}\label{fig:biasv5}
\end{figure}

At $p=.05,m=44$, only nine supports recover and 22 banks produce a plan; yield is 0.8761 on matched-prior test worlds and 0.6416 on uniform ones. We did not add a robust span estimator or select a best test-time checkpoint.

\subsection{Saved evidence and reproduction}
\texttt{experiments\_v5/audit.py} re-verifies each saved fitted support against its training labels by a rank identity, replays selected-source decoding on the \VFivePlannedCells{} of \VFiveTotalCells{} stress cells that produced a plan (\VFiveDecodedContexts{} decoded contexts), and checks all 368,640 stored per-context outcomes; the remaining cells are recorded abstentions, for which there is no plan to replay. It independently computes the rank of each fitted support from dense binary matrices. It also recomputes the 27 external profiles with QSP and checks every successful returned profile against the released inputs. SHA256 checks identify the external input file and exported byte content. The replay checks arithmetic and data integrity only.

The \VFiveTestCount-test suite covers the planner, the rate--distortion solvers, the measurement, call-count and local-search implementations, and the scratchpad trace format. The 1,911 exhaustive profile comparisons, 5,440 exact Bellman comparisons, and 960 numerical rate--distortion points are released with their result files. The release includes all outcomes, among them timeouts, abstentions, incorrect fitted supports, and competitor wins.

Run \texttt{bash scripts/reproduce\_paper.sh} for the tests, the local experiments, the saved-external replay, the figures except \cref{fig:nonuniform,fig:noise}, and the manuscript build. By default it checks the saved external outputs without running Sage; with \texttt{RUN\_EXTERNAL=1} on a host with Docker and network access it also reruns the pinned external solver (\texttt{scripts/run\_external\_v5.sh}).
\section{Neural transfer experiments: details}\label{app:vfiveexp}
All neural experiments use the decoder-only transformer of the released code: two pre-LayerNorm blocks with four-head attention and a GELU MLP of width $4d_{\rm model}$, $d_{\rm model}=128$, learned positional embeddings, and a two-symbol output head over bits. Pretraining uses AdamW (learning rate $10^{-3}$, cosine decay, weight decay $0.01$), batch $256$, and the step counts listed with each run. Everything runs on CPU with four threads. A default pretraining run takes minutes at $d\le7$; the widest and longest in-context runs and the scratchpad executors take between one and nine hours. Throughout, ``seeds'' index secrets $W^\star$ (and curricula) evaluated on a single pretraining run per bank; variation across banks is reported wherever a fraction is quoted. Seeds, banks, pretrained checkpoints' statistics, per-run result files, and all exact distributions are in \texttt{results\_v5/}, produced by \texttt{neural/exp\_transfer.py} (in-weight learner) and \texttt{neural/exp\_incontext.py} (in-context learner).

\paragraph{Banks.}
Banks are generated by \texttt{revision.planner\_instances.random\_bank} with the seeds recorded in each result file; latent dimension $d=k+h$, independent target and nuisance coefficients, and an invertible coordinate scramble. The exact profile, $C^\star$, and $(b^\star,n^\star)$ come from the quotient-subspace planner. Bit packing follows \texttt{revision/quotient\_planner.py}: coefficient $i$ is bit $i$, and a target vector is indexed by $\sum_ia_i2^i$.

\paragraph{In-weight learner.}
Tokens are $[\mathrm{BOS},c,q]$ with a context token $c$ and a query token $q$ naming a source or the target; target answers are emitted as $k$ binary tokens fed back autoregressively. Pretraining draws a context uniformly among $n_{\rm known}$ known contexts (fixed $W_c$) and $8$ uncertain contexts ($W$ resampled per sequence), and a query that is the target with probability $1/2$; the loss is next-bit cross-entropy on answer positions. A held-out fraction of $1/4$ of the known contexts never shows its target prompt during pretraining, and target accuracy on those contexts is the structure diagnostic. Post-training uses one uncertain context and a secret $W^\star$ per seed. SFT acquisition reveals $v_jW^\star$ with one verifier call (candidate zero passes or fails) and takes Adam steps (learning rate $3\cdot10^{-4}$) on that prompt until $P(\text{bit})\ge0.9$, at most $300$ steps; GRPO acquisition draws groups of $G=8$ rollouts with group-normalized advantages, no KL term and no clipping (the GRPO + KL variant adds the penalty $0.1\,\mathrm{KL}(\pi\,\|\,\pi_{\mathrm{ref}})$ to the loss on the source prompt, with $\pi_{\mathrm{ref}}$ the pretrained policy), until the same threshold or $25$ groups, counting every rollout as a call. The target phase runs GRPO ($G=8$, Adam $3\cdot10^{-4}$) until $P(A^\star)\ge0.9$, capped at $\min(30/P_0,30\cdot2^k)+50$ rollouts where $P_0$ is the exact target probability after acquisition.

\paragraph{In-context learner.}
Sequences are $[\mathrm{BOS}]$, then $(\text{source}_j,\text{label})$ token pairs for a uniformly random ordered subset of at most $\min(J,d+2)$ sources, then a target token and the $k$ target bits. Pretraining draws $W$ uniformly per sequence and trains next-bit prediction on every source label and target bit. At test time the prompt contains exactly the labels the curriculum acquired for $W^\star$ (one call each), and the target phase is as above. The budget profile sweeps the planner's optimal subset at every source budget $b\le\operatorname{rank}(B)$ over $64$ fresh worlds each, reporting the mean bits used, mean coset mass, the fraction of worlds whose argmax lies in the coset, and the mean rank of $A^\star$ (the calls of a most-probable-first eliminating learner), against the Bayes value $(2^{k-\rho}+1)/2$.

\paragraph{Scratchpad learner.}
Banks are generated in canonical coordinates ($T$ is the identity on the first $k$ latent bits; \texttt{make\_canonical\_bank}). This is a different generator from \texttt{random\_bank}: with the same seed it draws a different bank, and at $(k,h,J)=(6,2,16)$ the scratchpad bank has $b^\star=5$ where the in-context bank~0 has $b^\star=4$. Appendix~\ref{app:controls27} runs the plain decoder on the scratchpad's bank. A row is $d+1$ bits in the column order nuisance bits, then $a_{k-1},\dots,a_0$, then the right-hand side. For each acquired source, in acquisition order, the trace writes the row, then for each column $p=0,\dots,d-1$ writes the row again after XOR with the slot-$p$ row if bit $p$ is set and slot $p$ is filled, and finally stores the reduced row in the slot of its pivot column (or writes an empty marker); slots are never modified once written. For each target bit $a_i$ in the order $a_0,\dots,a_{k-1}$ the trace writes the index tag, then, if the slot of $a_i$'s column is filled, the running parity of that row's right-hand side with the row's entries at the already generated bits, and then the answer token; if the slot is empty it writes a free marker and the answer token, which is uniform under the data distribution. The model is the same two-block transformer with an added embedding of the token's column within its row (\texttt{TraceGPT}); it is trained on the whole sequence except the prompt with batch $64$ at $d=5$ and $32$ at $d=8$. At test time the trace is generated greedily up to the target tag; each target bit's chain of thought is generated greedily given the candidate's earlier bits, and the probability of the answer token is read out, so $P(A^\star)$ and the full distribution over $2^k$ candidates use the model's own computation. Trace length is linear in the number of acquired sources and quadratic in the latent dimension: with independent acquired rows it is $3+2|S|+|S|(d+2)^2$ tokens up to the target tag ($207$ at $d=5$, $|S|=4$); a dependent row emits a single empty marker instead of a stored row. Three executors are reported: at $d=5$, width $128$, two layers, $3000$ steps, and width $256$, three layers, $8000$ steps; at $d=8$, width $256$, three layers, $6000$ steps ($2.6$M parameters, a trace of $513$ tokens per context); a first version of the small executor with inconsistent field ids between training and generation is retained in \texttt{results\_v5/} with a \texttt{buggyfields} suffix and is not used.

\paragraph{Curricula.}
QSP uses the certificate's attaining subset at budget $b$; greedy is one-step joint transfer-rank gain with seeded tie breaking (\texttt{planner\_baselines.schedule}); random is the prefix of a random independent source basis; none acquires no sources. Every method has the same $b^\star$ and the same secret, and all reported probabilities are computed exactly. \Cref{tab:curricula27} in \cref{app:controls27} reports every curriculum at $b^\star$ for $d=5$, $8$ and $11$.

\begin{table}[h]\centering\scriptsize\setlength{\tabcolsep}{4pt}\begin{tabular}{llrrrrrrr}\toprule
acquisition & curriculum & runs & calls & coset mass & bits used & up/down & wrecked & rollouts (ref.) \\ \midrule
SFT & none & 9 & 0 & 1.000 & -0.03 & 0/0 & 0 & 65.6 (64.0) \\
 & qsp & 9 & 5 & 0.054 & -2.06 & 0/9 & 3 & 26893 (2.7) \\
 & greedy & 9 & 5 & 0.049 & -0.41 & 2/4 & 2 & $3.1\times10^{11}$ (2.7) \\
 & random & 9 & 5 & 0.076 & -0.08 & 3/5 & 0 & 81.6 (8.0) \\
\midrule
GRPO & none & 9 & 0 & 1.000 & -0.03 & 0/0 & 0 & 65.6 (64.0) \\
 & qsp & 9 & 437 & 0.000 & -3.75 & 0/9 & 7 & $2.9\times10^{18}$ (2.7) \\
 & greedy & 9 & 436 & 0.000 & -3.98 & 0/9 & 7 & $9.6\times10^{17}$ (2.7) \\
 & random & 9 & 532 & 0.000 & -2.42 & 0/9 & 6 & $3.1\times10^{18}$ (8.0) \\
\midrule
GRPO + KL & none & 9 & 0 & 1.000 & -0.03 & 0/0 & 0 & 65.6 (64.0) \\
 & qsp & 9 & 400 & 0.012 & -2.71 & 0/9 & 4 & $5.7\times10^{11}$ (2.7) \\
 & greedy & 9 & 430 & 0.000 & -1.37 & 1/8 & 5 & $2.2\times10^{16}$ (2.7) \\
 & random & 9 & 453 & 0.000 & -2.19 & 0/7 & 6 & $1.0\times10^{17}$ (8.0) \\
\bottomrule\end{tabular}\caption{\textbf{In-weight acquisition at $k=6$} (three banks, three secrets each). Calls are verifier calls spent on sources (mean over runs); coset mass is the median over runs; bits used is $k+\log_2$ of the mean over runs of $P(A^\star)$ after acquisition; up/down counts runs whose own bits used moved above $+0.2$ or below $-0.2$; a run is wrecked when the coset mass falls below $10^{-2}$; rollouts is the expected number of policy samples until the first verified success, $\mathbb{E}\,[1/P(A^\star)]$ over runs, with the exact posterior in parentheses. GRPO runs whose source becomes confidently wrong receive no reward gradient afterwards (all-fail groups) and spend their remaining call budget without acquiring the label.}\label{tab:inweight}\end{table}

\begin{table}[h]\centering\footnotesize\begin{tabular}{lrrrrrr}\toprule
bank & width & layers & steps & lr & final training loss & bits used at $b^\star$ \\ \midrule
0 & 64 & 2 & 6,000 & 0.001 & 0.693 & -0.01 of 4 \\
0 & 128 & 2 & 6,000 & 0.0003 & 0.693 & -0.04 of 4 \\
0 & 128 & 2 & 6,000 & 0.001 & 0.693 & 0.00 of 4 \\
0 & 128 & 2 & 6,000 & 0.003 & 0.693 & 0.00 of 4 \\
0 & 256 & 2 & 6,000 & 0.001 & 0.692 & -0.03 of 4 \\
0 & 256 & 2 & 20,000 & 0.001 & 0.693 & 0.00 of 4 \\
0 & 256 & 4 & 6,000 & 0.001 & 0.693 & 0.01 of 4 \\
0 & 512 & 2 & 6,000 & 0.001 & 0.693 & 0.00 of 4 \\
1 & 64 & 2 & 6,000 & 0.001 & 0.617 & 1.96 of 5 \\
1 & 128 & 2 & 6,000 & 0.001 & 0.634 & 1.00 of 5 \\
1 & 256 & 2 & 6,000 & 0.001 & 0.646 & 1.00 of 5 \\
2 & 64 & 2 & 6,000 & 0.001 & 0.692 & 0.01 of 5 \\
2 & 128 & 2 & 6,000 & 0.001 & 0.685 & 0.77 of 5 \\
2 & 256 & 2 & 6,000 & 0.001 & 0.693 & 0.00 of 5 \\
\bottomrule\end{tabular}\caption{\textbf{In-context decoder capacity sweep} at $k=6$ ($d=8$): 8 configurations in 14 runs, 3 of them on all three banks and the rest on bank~0, one pretraining seed per row. The Bayes-optimal training loss for banks 0, 1, 2 is 0.477, 0.476, 0.472 nats against the uninformed value $\ln 2=0.693$, so every configuration that stays at $\ln2$ has left the whole learnable signal on the table.}\label{tab:capacity}\end{table}

\begin{figure}[h]
\centering\includegraphics[width=.8\linewidth]{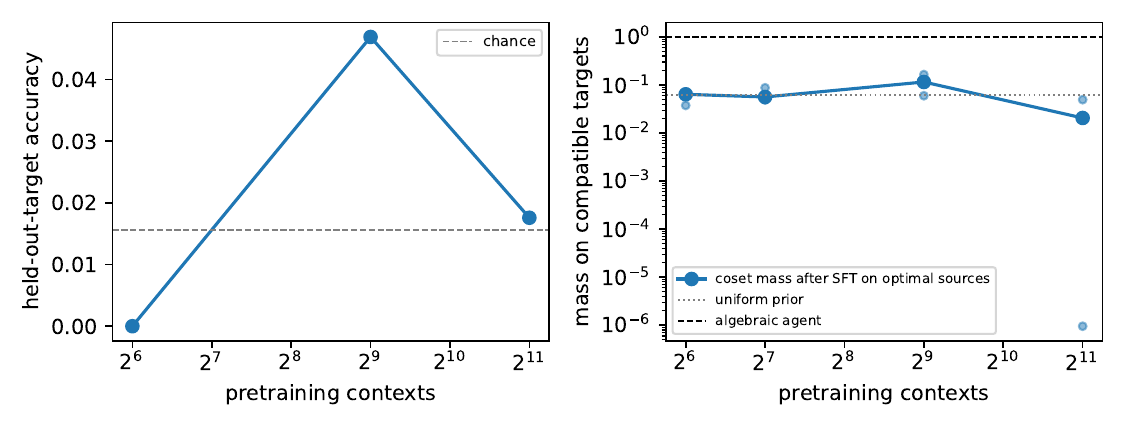}
\caption{\textbf{Diversity does not buy structure for the in-weight learner.} $k=6$, one bank; left, target accuracy on known contexts whose target prompt was held out of pretraining; right, coset mass after SFT on the planner's sources, per run and median, against the prior and the algebraic agent.}\label{fig:diversity}
\end{figure}

\begin{figure}[h]
\centering\includegraphics[width=.42\linewidth]{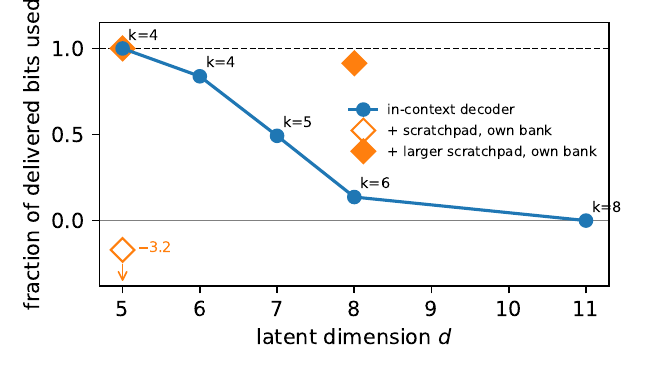}
\caption{\textbf{Fraction of delivered bits used against latent dimension,} at the zero-error split with the planner's sources. In-context decoder: three banks per point at width $128$ (one pretraining run per bank). Diamonds: the scratchpad executors, each on its own canonical bank, which differs from the in-context banks. The smaller $d=5$ executor puts less probability on $A^\star$ than the uniform prior, so its fraction is negative; it is drawn at the bottom edge with its value.}\label{fig:capacity}
\end{figure}

\begin{figure}[h]
\centering\includegraphics[width=.7\linewidth]{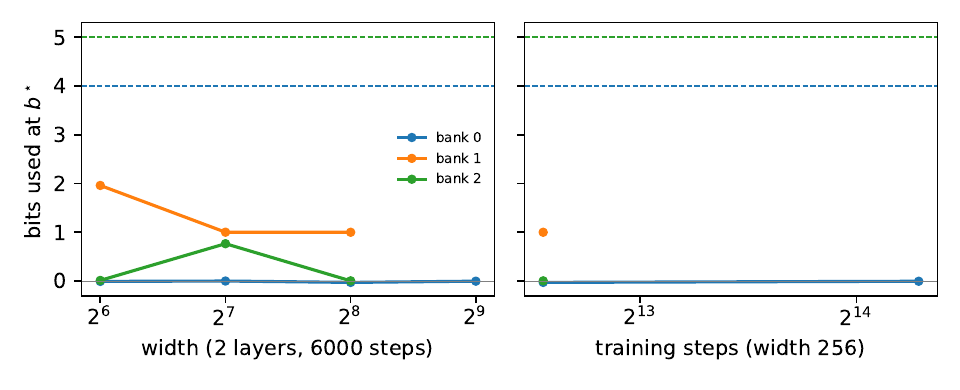}
\caption{\textbf{Capacity sweep at $d=8$.} In-context decoder on each of the three $d=8$ banks, one pretraining seed per configuration: bits used at the zero-error split against width (two layers, $6000$ steps) and against training steps (width $256$). Dashed lines mark each bank's delivered bits (4, 5 and 5). No configuration reaches the frontier, and with one seed per configuration there is no clear trend in width or steps.}\label{fig:capacitysweep}
\end{figure}

\paragraph{Bayes-optimal training loss.}
For the in-context task, the optimal predictor of each supervised bit is the exact posterior given the labels seen so far, computed by enumerating all $2^d$ worlds (\texttt{neural/bayes\_loss.py}, $1000$ sampled training sequences per bank). Its average cross-entropy is \BayesLossDFive{} nats at $d=5$, \BayesLossDSix{} at $d=6$, \BayesLossDSeven{} at $d=7$, \BayesLossDEight{} at $d=8$, and \BayesLossDEleven{} at $d=11$, against $\ln 2=0.693$ for an uninformed predictor; a model whose training loss stays at $\ln2$ has captured none of the available signal.

\paragraph{GRPO lock-in.}
A source whose answer has become confidently wrong through earlier updates receives only all-fail groups, whose group-normalized advantage is zero, and the label is never acquired within the cap of $25$ groups. The implementation still takes an optimizer step on such groups, and one Adam state is shared across sources, so momentum keeps moving the parameters without a policy gradient. We count a source as locked in when its final probability is below $10^{-2}$. Under GRPO, \IWgrpoStuckRuns{} of the \IWgrpoRuns{} runs with a curriculum end with a locked-in source, and \IWgrpoStuckCallsPct{} of the source calls of all \IWgrpoRuns{} runs go to sources that end locked in (\IWklStuckRuns{} runs and \IWklStuckCallsPct{} with the KL term). The damage to the target prior is partly an artifact of this implementation, so these runs say nothing general about policy-gradient methods.

\paragraph{Undiscovered runs.}
A target run that reaches its cap without a reward is reported as undiscovered and counted at its cap in call totals. GRPO acquisition on source prompts often drives the target probability below $10^{-6}$; those runs are retained.

\section{Curricula, degree analysis and controls}\label{app:controls27}
\paragraph{Setup of \cref{sec:neuralexp}.} The \emph{in-context} learner is a two-layer decoder-only transformer ($d_{\rm model}=128$, about $0.4$M parameters) trained from scratch. It is pretrained on sequences that list the acquired (source, label) pairs for a random subset of sources, followed by the $k$ target bits. At test time the planner chooses which labels appear in the prompt, one call each. We use $(k,h,J)\in\{(4,1,8),(4,2,12),(5,2,20),(6,2,16),(8,3,32)\}$, so the latent dimension is $d=k+h\in\{5,6,7,8,11\}$, with three random banks for each and four secrets per bank. Each bank has one pretraining run of $6{,}000$ steps ($8{,}000$ at $d=11$), and bank~0 at $d\in\{6,7,8\}$ is repeated with three training seeds (\cref{app:vfiveexp}). We evaluate at the \emph{zero-error split}, with the $b^\star$ sources of an optimal split of $C^\star$ (\cref{sec:planner}).

\paragraph{Spread over banks and seeds.} At the zero-error split the fraction used ranges over banks from \ICfracDSixMin{} to \ICfracDSixMax{} at $d=6$, \ICfracDSevenMin{} to \ICfracDSevenMax{} at $d=7$ and \ICfracDEightMin{} to \ICfracDEightMax{} at $d=8$. On bank~0, three training seeds span \SeedMinSix--\SeedMaxSix{} at $d=6$, \SeedMinSeven--\SeedMaxSeven{} at $d=7$ and \SeedMinEight--\SeedMaxEight{} at $d=8$. Bayes-optimal pretraining losses are in \cref{app:vfiveexp}. Post hoc per-bit scores (\cref{tab:perbit}) show that at $6{,}000$ steps each determined bit is used almost fully or almost not at all.

\paragraph{What choosing sources is worth.} \Cref{tab:curricula27} compares curricula at the zero-error split. At $d=5$, where the decoder decodes, the learner reproduces the ordering that the exact references fix. The planner's sources need \CurFiveBlindQsp{} expected samples against the posterior's \CurFiveBlindRefQsp; a random basis needs \CurFiveBlindRandom{} against \CurFiveBlindRefRandom; no sources need \CurFiveBlindNone. At $d=8$ the exact references still differ by a factor of \CurEightSpreadRef{} between the planner's sources and none, but the learner's costs differ by a factor of \CurEightSpreadLearner. Counted in distinct verifications instead, as in \cref{thm:rank}, the planner's sources leave the $d=8$ learner \CurEightElimQsp{} guesses against the posterior's \CurEightBayesQsp{}, and no sources \CurEightElimNone.

\paragraph{Supervised traces.} A variant pretrained to write out the elimination, reducing each acquired row against a table of pivots, uses \SPbigSixDEightfracQsp{} of the delivered bits with the planner's sources on a separately drawn $d=8$ bank. The variant is larger ($2.6$M parameters) and supervised at every step, so this separates neither size nor serial computation from process supervision \citep{wies2023,abbe2024far,kim2025transformers}.

\paragraph{Curricula.} \Cref{tab:curricula27} gives the costs quoted in \cref{sec:neuralexp}. Guesses count distinct verifications, as in \cref{thm:rank}; samples count draws with replacement, as a policy makes them.
\begin{table}[h]\centering\small\setlength{\tabcolsep}{4pt}
\caption{\textbf{Curricula at the zero-error split} (in-context decoder, three banks, four secrets each; $|S|$ ranges over banks, whose $b^\star$ differ). Bits are $k+\log_2\E$ over contexts of $2^{\rho-k}$ (delivered) and of $P(A^\star)$ (used). On frontier is the fraction of contexts whose own bits used are within $0.2$ of their own $\rho(S)$; coset mass is the mean probability the learner puts on the targets compatible with the acquired labels. Samples is the expected number of draws, with replacement, from the unchanged policy until a verification accepts, $\E[1/P(A^\star)]$; guesses is the number of distinct verifications made by a learner that verifies its most probable answers first, as in Theorem~\ref{thm:rank}, averaged over contexts. The exact posterior is in parentheses: $\E[2^{k-\rho}]$ samples and $\E[(2^{k-\rho}+1)/2]$ guesses. Means are over twelve contexts and are therefore noisy: without sources every learner has the same expected number of guesses as the exact posterior, $(2^k+1)/2$, so the difference in those rows is sampling noise.}\label{tab:curricula27}
\begin{tabular}{llrrrrrrr}\toprule
$d$ & sources & $|S|$ & delivered & used & on frontier & coset mass & samples (exact) & guesses (exact) \\ \midrule
5 & none & 0 & 0.00 & 0.00 & 1.00 & 1.00 & 16.0 (16.0) & 8.0 (8.5) \\
 & random basis & 2--3 & 1.74 & 1.40 & 0.75 & 0.86 & 7.8 (5.3) & 4.2 (3.2) \\
 & greedy & 2--3 & 2.74 & 2.65 & 0.92 & 0.96 & 2.8 (2.7) & 1.8 (1.8) \\
 & planner & 2--3 & 2.74 & 2.74 & 1.00 & 1.00 & 2.7 (2.7) & 1.6 (1.8) \\
\midrule
8 & none & 0 & 0.00 & 0.00 & 1.00 & 1.00 & 64.0 (64.0) & 34.8 (32.5) \\
 & random basis & 4--6 & 3.58 & 0.44 & 0.00 & 0.14 & 51.8 (8.0) & 24.0 (4.5) \\
 & greedy & 4--6 & 4.74 & 0.63 & 0.00 & 0.06 & 45.6 (2.7) & 23.6 (1.8) \\
 & planner & 4--6 & 4.74 & 0.65 & 0.00 & 0.06 & 45.2 (2.7) & 20.9 (1.8) \\
\midrule
11 & none & 0 & 0.00 & 0.00 & 1.00 & 1.00 & 256 (256) & 105.2 (128.5) \\
 & random basis & 8 & 5.00 & 0.00 & 0.00 & 0.03 & 256 (8.0) & 96.6 (4.5) \\
 & greedy & 8 & 6.58 & 0.00 & 0.00 & 0.01 & 256 (3.0) & 103.2 (2.0) \\
 & planner & 8 & 7.00 & 0.00 & 0.00 & 0.01 & 256 (2.0) & 104.1 (1.5) \\
\bottomrule\end{tabular}\end{table}
\paragraph{Degree.}
Target bits are emitted in the order $a_0,\dots,a_{k-1}$ and teacher-forced, so when the decoder predicts $a_i$ it has seen the acquired labels and $a_0,\dots,a_{i-1}$. Bit $a_i$ is determined when $T_i$ lies in the span of the acquired rows $B_S$ and $T_0,\dots,T_{i-1}$, and its degree is the fewest of those rows whose XOR is $T_i$. The degree is computed from the bank and the planner's subset alone (\texttt{neural/select\_degree\_banks.py}). Table~\ref{tab:degrees} lists every bank. On the released banks, a bank's largest degree separates the five banks on the frontier (at least \DegLowMin{} of the delivered bits used) from the ten that are not (at most \DegHighMax); no released bank has largest degree three. The per-bit account is weaker: at $d\ge7$ some banks use none of their degree-two bits.

\paragraph{Degree-controlled test.}
Because degree grows with $d$ in the released banks, we fixed $(k,h,J)=(4,2,12)$, drew banks from a seed range disjoint from all other runs, and, among banks with at least three determined bits, kept the first three whose determined bits all have degree at most two and the first three with a bit of degree four or more (\texttt{results\_v6/degree\_banks.json}). Before training we recorded two predictions: (P1) at least $95\%$ of the delivered bits are used on every low bank and less than $95\%$ on every high bank; (P2) on every high bank the bits used are at most the number of determined bits of degree at most two, plus $0.5$. Each bank was trained with the default configuration and two training seeds.
\CtrlResultsParagraph{} \DegreeConclusion

\paragraph{Which bit is lost (post hoc).}
To see which delivered bit the decoder loses, we reran the first training seed of the controlled banks with the same seeds, which reproduces every recorded probability exactly, stored the full distribution over targets, and scored the budget curve on all $2^d=64$ secrets (\cref{tab:perbit}; \texttt{results\_v6/run\_perbit2.sh}). On all \PerBitN{} banks the decoder uses each determined bit either almost fully (at least \PerBitUsedMin{} bit) or almost not at all (at most \PerBitLostMax{} bit). At $b^\star$ the lost bits are \PerBitLostList. With every source given, so that every target bit is determined, it uses \PerBitFullByPos: at this budget use depends on the position of a bit in the answer, and the first bit is never used. We have no explanation for this pattern. Degree does not predict it, nor does how often a bit is determined during pretraining. The fewest acquired labels a bit needs, counting earlier target bits as free, separates used from lost bits on 17 of the 18 determined bits here, but it fails on the released banks at $d\in\{5,6\}$, where bits that need two labels are used (\texttt{neural/perbit\_diagnostics.py}). Scoring all 64 secrets instead of four changes the bits used by at most \PerBitSecretsGap{} bit.

\paragraph{Four times the training budget (post hoc).}
After bank 202609503 reached the frontier when trained for $24{,}000$ steps instead of $6{,}000$, we recorded two predictions for the remaining five banks before any of their results were seen (\texttt{results\_v6/long\_predictions.json}): (L1) at least $95\%$ of the delivered bits used at $b^\star$ on every low-degree bank and less on every high-degree bank; (L2) on every high-degree bank, bits used at most the number of determined bits of degree at most two, plus $0.5$. \LongSentence{} \LongTestSentence{} Longer training lowers the loss on every bank, but whether a bit is recovered is not predicted by its degree: the recovered bits have degrees two and three, one of degree four is recovered partly (0.72 bit), and a degree-two bit ($a_1$) stays lost on bank 202609512.
\begin{table}[h]\centering\footnotesize\setlength{\tabcolsep}{3pt}
\caption{\textbf{Which bit is lost (post hoc).} Controlled banks, first training seed, $6{,}000$ pretraining steps unless stated, rerun with the same seeds (reproducing the recorded runs exactly) and scored on all $2^d=64$ secrets. Per-bit use of $a_i$ is $1+\E\log_2P(a_i\mid a_{<i},\text{labels})$; the four terms add up to $k+\E\log_2P(A^\star)$, which is at most the bits used $k+\log_2\E P(A^\star)$ by Jensen's inequality. Bits used over all 64 secrets are compared with the four-secret estimate of \cref{tab:degrees}. The column for all sources scores the largest budget on the planner's profile, where every target bit is determined; the last column repeats $b^\star$ after $24{,}000$ steps.}\label{tab:perbit}
\begin{tabular}{llllrrll}\toprule
set & bank & degrees at $b^\star$ & per-bit use at $b^\star$ & used (64) & used (4) & per-bit use, all sources & at $b^\star$, $24{,}000$ steps \\ \midrule
low & 202609500 & $(2,2,\cdot,1)$ & $0.13,1.00,0.00,1.00$ & 2.17 & 1.97 & $0.02,1.00,0.03,1.00$ & $1.00,1.00,0.00,1.00$ \\
low & 202609503 & $(2,2,\cdot,2)$ & $0.00,1.00,0.00,1.00$ & 2.00 & 2.00 & $0.00,1.00,1.00,1.00$ & $1.00,1.00,0.00,1.00$ \\
low & 202609512 & $(\cdot,2,1,2)$ & $0.00,0.00,1.00,1.00$ & 2.00 & 2.00 & $0.00,0.00,1.00,1.00$ & $0.00,0.00,1.00,1.00$ \\
high & 202609506 & $(\cdot,3,4,2)$ & $0.00,0.02,0.11,1.00$ & 1.16 & 1.27 & $0.01,0.02,0.03,1.00$ & $0.00,1.00,0.72,1.00$ \\
high & 202609518 & $(\cdot,4,2,1)$ & $0.00,0.00,1.00,1.00$ & 2.00 & 2.00 & $0.00,0.00,1.00,1.00$ & $0.00,0.01,1.00,1.00$ \\
high & 202609534 & $(\cdot,4,2,2)$ & $0.00,0.00,1.00,1.00$ & 2.00 & 2.00 & $0.02,0.02,1.00,1.00$ & $0.00,0.00,1.00,1.00$ \\
\bottomrule\end{tabular}\end{table}

\paragraph{Same-bank control for the scratchpad.}
The $d=8$ scratchpad executor was trained and evaluated on its own bank, drawn by \texttt{make\_canonical\_bank}. \CanonicalSentence{} The comparison with the scratchpad therefore concerns one bank, twelve contexts, and learners that differ in supervision, architecture and size, as well as in serial computation.

\begin{table}[h]\centering\small\setlength{\tabcolsep}{5pt}
\caption{\textbf{XOR degrees and bits used at the zero-error split.} Degree of target bit $a_i$: the fewest rows among the planner's acquired source rows and the earlier target rows whose XOR is $T_i$ ($\cdot$ = not determined). Bits are $k+\log_2\E$ of $2^{\rho-k}$ (delivered) and of $P(A^\star)$ (used) over four secrets with the planner's sources. Released banks list one training seed, except bank 0 at $d\in\{6,7,8\}$, which lists three; controlled banks list both training seeds.}\label{tab:degrees}
\begin{tabular}{llllrr}\toprule
set & $d$ & bank & degrees $(a_0,\dots,a_{k-1})$ & delivered & used \\ \midrule
released & 5 & 0 & $(\cdot,1,\cdot,2)$ & 2.00 & 2.01 \\
released & 5 & 1 & $(\cdot,2,2,2)$ & 3.00 & 2.99 \\
released & 5 & 2 & $(\cdot,2,2,2)$ & 3.00 & 3.00 \\
released & 6 & 0 & $(\cdot,1,4,3)$ & 3.00 & 1.46, 1.98, 1.90 \\
released & 6 & 1 & $(\cdot,\cdot,2,2)$ & 2.00 & 1.99 \\
released & 6 & 2 & $(\cdot,2,2,2)$ & 3.00 & 3.00 \\
released & 7 & 0 & $(2,\cdot,4,2,3)$ & 4.00 & 1.03, 0.03, 1.01 \\
released & 7 & 1 & $(\cdot,5,2,2,2)$ & 4.00 & 3.12 \\
released & 7 & 2 & $(\cdot,4,3,2,2)$ & 4.00 & 0.02 \\
released & 8 & 0 & $(\cdot,\cdot,3,3,2,4)$ & 4.00 & 0.00, -0.02, 0.48 \\
released & 8 & 1 & $(1,4,\cdot,2,3,3)$ & 5.00 & 1.00 \\
released & 8 & 2 & $(3,\cdot,3,3,4,3)$ & 5.00 & 0.77 \\
released & 11 & 0 & $(\cdot,6,3,5,3,3,2,3)$ & 7.00 & 0.00 \\
released & 11 & 1 & $(\cdot,5,2,3,3,4,4,3)$ & 7.00 & 0.00 \\
released & 11 & 2 & $(6,2,\cdot,3,3,3,5,4)$ & 7.00 & 0.00 \\
\midrule
controlled (low) & 6 & 202609500 & $(2,2,\cdot,1)$ & 3.00 & 1.97, 2.06 \\
controlled (low) & 6 & 202609503 & $(2,2,\cdot,2)$ & 3.00 & 2.00, 2.08 \\
controlled (low) & 6 & 202609512 & $(\cdot,2,1,2)$ & 3.00 & 2.00, 2.01 \\
controlled (high) & 6 & 202609506 & $(\cdot,3,4,2)$ & 3.00 & 1.27, 1.25 \\
controlled (high) & 6 & 202609518 & $(\cdot,4,2,1)$ & 3.00 & 2.00, 1.99 \\
controlled (high) & 6 & 202609534 & $(\cdot,4,2,2)$ & 3.00 & 2.00, 2.01 \\
\bottomrule\end{tabular}\end{table}
\section{Numerical methods, certificates, and reproducibility}\label{app:compute}
\subsection{Two independently formulated optimizations}
For fixed $\beta$, enumerate coverage columns and maximize
$F_\beta(\mu)=\sum_\theta q_\theta\ln\sum_L\mu_Le^{\beta c_{\theta L}}$ over the simplex, using analytic gradients and a sequential quadratic programming optimizer. A multiplicative Blahut--Arimoto polishing step is available. The returned Gibbs channel is evaluated directly, using its \emph{actual} marginal $\mu'=qW$ rather than assuming that the optimization variable is already its marginal.

There is a globally valid certificate even before convergence. For the current $\mu$, define $Z$, $g_L=\sum_\theta q_\theta e^{\beta c_{\theta L}}/Z_\theta$, and $g_{\max}=\max_Lg_L$. At the channel's achieved success $s$, a lower bound on its optimal information cost is
\[
 \frac{\beta s-F_\beta(\mu)-\ln g_{\max}}{\ln2}.
\]
The actual channel information equals
$[\beta s-F_\beta(\mu)-D_{\mathrm{nat}}(\mu'\|\mu)]/\ln2$.
Their difference is a primal--dual gap; failed certificates raise an error rather than being reported as optima. An explicit reference prior $r_\theta\propto q_\theta/Z_\theta$ can further strengthen the lower bound via~\eqref{eq:dual}.

The second optimization does not parameterize a list marginal or construct a Gibbs channel. It maximizes
\[
 \frac{\sum_\theta q_\theta t_\theta-\sum_\theta q_\theta\ln q_\theta+\beta s}{\ln2}
 \quad\text{subject to}\quad
 \operatorname{logsumexp}_\theta(t_\theta+\beta c_{\theta L})\le0\quad\forall L,
 \quad\beta\ge0.
\]
This is a classical rate--distortion geometric-programming dual in logarithmic coordinates \citep{chiang2004}. After optimization, subtract the largest constraint violation, plus $10^{-13}$, from every $t_\theta$ to restore feasibility before reporting the bound. Both formulations use floating-point arithmetic, and the release reports their numerical tolerances.

\subsection{Experimental configurations}
The certificate generator uses NumPy seed $202609051$. There are 240 prior--geometry settings with $M\in\{3,\ldots,9\}$ worlds, local budgets in $\{0,1,2\}$, and either unique-answer identity geometry or overlapping binary acceptance matrices with three to seven responses. Each world is made coverable. Priors are drawn from a Dirichlet distribution with concentration $0.7$. Four natural-unit multipliers $\beta\in\{0.15,0.7,2,6\}$ produce 960 points. Success targets are those attained by the primal channel; the independent dual is solved at exactly those targets. A separate program computes the zero-error value on every setting.

The source-protocol audits recursively enumerate every finite transcript of fixed-seed adaptive policies. They vary priors, reference priors, acceptance geometry, queries, Bernoulli source kernels, and stopping decisions. The 2,000-tree suite is run, and its first 240 trees are also checked using independently optimized rate--distortion duals for their joint success--budget probabilities. The coverage check uses both exhaustive subsets and an independent Bellman recursion. All branches of each generated tree are summed; no trajectories are sampled.

For the noisy theorem, 81 values of $\epsilon$ are logarithmically spaced between $.002$ and $.45$, with $\delta=\epsilon^2$. The explicit conditional test-channel joint distributions are independently evaluated for error and information. Unit tests compare the upfront formula to a posterior-mixture linear program on a 1,001-point grid augmented with the analytic tangent and target-error points, at four noise levels and four error budgets each.

\paragraph{Computational limits.}
Response-list enumeration is exponential, and the information ledger requires the true probabilistic model. Released channel matrices and independent dual variables can be checked without optimization using \texttt{python -m revision.verify\_saved\_certificates}.

\paragraph{Commands.}
Run \texttt{bash scripts/reproduce\_paper.sh} for the test suite, the theory checks (including \texttt{python -m theory\_checks.check\_bridge}), the stress and local-search experiments, the replay of the saved external-solver outputs, the number macros and all figures except \cref{fig:nonuniform,fig:noise} (redrawn by \texttt{python -m revision.run\_frontier\_experiments}), and the manuscript. With \texttt{RUN\_NEURAL=1} it also reruns every transformer run, listed in \texttt{scripts/reproduce\_v5\_neural.sh}; in sequence they take one to two days of CPU time.
\section{Extended related work}\label{app:relatedfull}
\paragraph{Information frontier.}
Static list-Fano functions, rate--distortion duals and convex-corner entropies underlie the frontier \citep{csiszar1974extremum,sakai2020,blahut1972,chiang2004,csiszar1990}. Guessing with side information has been priced for a helper of fixed rate, a compressed description and a distortion criterion \citep{arikan1998guessing,bunte2014encoding,weinberger2020guessing,saito2018nonasymptotic,graczyk2021}; \citet{saito2018nonasymptotic} already allow queries of the form ``is this answer acceptable'', which corresponds to overlapping acceptance sets here. The causal charge~\eqref{eq:budget} is directed information \citep{massey1990causality}. The converse uses change of measure \citep{kaufmann2016,simchowitz2017} and the fundamental inequality of \citet{garivier2019explore} in the interactive-Fano form \citep{chen2024assouad}; related Bayes-risk lower bounds via small-ball probabilities are \citet{chen2016bayes,xu2017information,gerchinovitz2020fano}, where the small-ball probability plays the role of $V_n$. The achievability part of \cref{prop:general} is channel simulation with common randomness through the strong functional representation lemma \citep{li2018strong}.

\paragraph{Queries and value of information.}
Verification is an equivalence query and source calls are membership-style queries \citep{angluin1988queries}; twenty questions with equality and comparison queries \citep{dagan2017twenty} counts every query, whereas here verifications are capped but free and only auxiliary calls are counted. For $n=0$ and unique answers, $\mathcal R_0(q,s)$ is Erokhin's $\epsilon$-entropy, and the expected-call results relate to variable-length compression allowing errors \citep{kostina2015variable}. Deterministic source-first protocols at zero error are minimum-entropy covers, whose gap to hypergraph entropy is studied by \citet{alon1996source,cardinal2008minimum}. The terminal utility of \cref{thm:fixed} is classical value of information \citep{krause2009}, and XOR complementarity is the textbook failure of one-step information gain \citep{krause2005,golovin2011adaptive}. Source-versus-target data trade-offs in transfer learning \citep{hanneke2019value} concern statistical estimation rather than a verifier budget.

\paragraph{Linear banks.}
The rank profile is a relative dimension/length profile \citep{wei1991generalized,forney1994dimension,luo2005,kurihara2012,zhuang2014relative}, with existing solvers \citep{sanjose2025}, one of which we run directly. Support recovery is closure learning of subspaces \citep{helmbold1992learning} with random-matrix counting \citep{fulman2015}.

\paragraph{Reinforcement learning from verifiable rewards.}
Recent analyses study base-model coverage \citep{foster2025,chen2026coverage}, sharpening \citep{huang2025self}, compositional learnability \citep{barzilai2026}, generation budgets \citep{wachi2026}, autocurricula \citep{rajaraman2026,rajaraman2026learning}, implicit curricula \citep{huang2026implicit}, outcome-based learning of reasoning \citep{ranmilo2026outcome}, the limits of outcome feedback \citep{chen2025outcome}, query complexity with a verifier \citep{botta2025query}, whether such training extends or sharpens a base model \citep{yue2025does}, how it compares with supervised fine-tuning \citep{chu2025sft}, and why it forgets less than supervised fine-tuning \citep{shenfeld2025razor}. We price the information and calls that source tasks deliver about one instance; we do not model generalization across prompts.

\paragraph{Parities and learned decoders.}
The in-context decoder is in-context learning of a function class \citep{garg2022,akyurek2023}, here linear systems over $\F$. Gradient descent learns sparse parities slowly \citep{barak2022}, transformers fail at in-context parities \citep{bhattamishra2024understanding}, and curricula or intermediate supervision make parities learnable \citep{abbe2023provable,cornacchia2023mathematical,wang2025learning,wies2023,abbe2024far,kim2025transformers,malach2024autoregressive,joshi2025theory,wen2025sparse}. Expressivity results for chains of thought \citep{feng2023,li2024cot,merrill2024} concern representation, not learnability. Usable versus Shannon information is formalized by $\mathcal V$-information \citep{xu2020theory}. The in-weight learner's failure to compose stored facts is the known difficulty of knowledge manipulation \citep{allenzhu2023,berglund2023}, and its failure to gain structure from more contexts echoes the diversity threshold for in-context learning \citep{raventos2023}. The policy-gradient learner is GRPO \citep{shao2024}.
\end{document}